\documentclass{article}

\usepackage{microtype}
\usepackage{graphicx}
\usepackage{booktabs} % for professional tables
\usepackage{longtable}

\usepackage[preprint]{neurips_2026}

\usepackage[hypertexnames=false]{hyperref}
\usepackage{amsmath}
\usepackage{amssymb}
\usepackage{mathtools}
\usepackage{amsthm}

\usepackage{microtype}
\usepackage{graphicx}
\usepackage{booktabs} % for professional tables

\usepackage{algorithm}
\usepackage{algpseudocode}
\usepackage{graphicx} 
\usepackage{todonotes}
\usepackage{xcolor}
\usepackage{dsfont}
\usepackage{nicefrac}
\usepackage{tabularx}
\usepackage{pifont}
\usepackage{enumitem}
\usepackage{multirow}
\usepackage{caption}
\usepackage{wrapfig}
\usepackage[export]{adjustbox}
\usepackage[list=true]{subcaption}
\usepackage{colortbl}
\definecolor{lightgray}{RGB}{242, 242, 242}
\usepackage{float}
\usepackage[capitalize,noabbrev]{cleveref}
\usepackage{siunitx}
\usepackage{xspace}
\DeclareRobustCommand{\ie}{i.e.,\@\xspace}
  
\DeclareRobustCommand{\eg}{e.g.,\@\xspace}

\usepackage{tikz}
\usepackage{tikz-cd}
\usetikzlibrary{shapes,backgrounds}
\usepackage{array}
\usepackage{enumitem}

\usepackage{amsthm}
\usepackage{thmtools, thm-restate}
\declaretheorem[numberwithin=section]{thm}
\declaretheorem[sibling=thm]{theorem}
\declaretheorem[sibling=thm]{lemma}
\declaretheorem[sibling=thm]{corollary}
\declaretheorem[numberwithin=section]{assumption}
\declaretheorem[]{challenged assumption}
\declaretheorem[]{definition}

\newcounter{relctr} % Counter for relations
\everydisplay\expandafter{\the\everydisplay\setcounter{relctr}{0}} % Reset every equation
\AtBeginDocument{} % Store original definition

\Crefname{figure}{Fig.}{Fig.}
\usepackage{amsfonts}[mathscr]
\usepackage{amssymb}
\usepackage{amsmath}

\DeclareMathOperator*{\argmax}{arg\,max}
\DeclareMathOperator*{\argmin}{arg\,min}
\newcommand{\X}{\mathcal{X}}
\newcommand{\R}{\mathbb{R}}

\newcommand{\mypar}[1]{\textbf{#1.}}

\usepackage[most]{tcolorbox}

\newcommand{\Z}{\mathcal{Z}}

\usepackage{shorthands}

\definecolor{insightcolor}{HTML}{2B2F36}
\newtcolorbox{takeawaybox}{
  enhanced,
  breakable,
  colback=white,                 % white inside
  colframe=insightcolor,       % violet border
  boxrule=0.9pt,
  arc=3mm,                       % corner radius
  left=10pt,right=10pt,top=8pt,bottom=8pt,
  title=Key Insight,
  fonttitle=\bfseries,
  coltitle=insightcolor,
  attach boxed title to top left={xshift=10pt,yshift=-2mm},
  boxed title style={
    colback=white,               % white title background
    colframe=insightcolor,
    boxrule=0.9pt,
    arc=2mm,
    left=6pt,right=6pt,top=2pt,bottom=2pt
  }
}

\usepackage{dsfont}
\usepackage{nicefrac}

\definecolor{myviolet}{rgb}{0.6, 0.4, 0.8}

\definecolor{pastelblueold}{RGB}{56,146,236}
\definecolor{pastelblue}{RGB}{43,115,187}
\definecolor{pastelgreen}{RGB}{63,159,95}

\definecolor{oxfordblue}{HTML}{163A70}
\definecolor{softbluebg}{HTML}{EAF3FF}
\definecolor{deepviolet}{HTML}{4B2E83}
\definecolor{coolgraybg}{HTML}{F3F4F7}

\hypersetup{
  colorlinks=true,
  citecolor=pastelgreen,
  linkcolor=pastelblue,    % internal references (equations, figures, tables)
  urlcolor=pastelblue      % URLs
}

\newcommand{\AlgNameLong}{Active Flow Expansion \xspace}
\newcommand{\AlgNameShort}{\textsc{\small{ActFlow}}\xspace}
\newcommand{\AlgNameDef}{Active Flow Expansion (\textsc{\small{ActFlow}})\xspace}

\newcommand{\AlgRecF}{\textsc{\small{Rec-F}}\xspace}
\newcommand{\AlgRecNF}{\textsc{\small{Rec-NF}}\xspace}

\title{Active Flow Expansion for Out-of-Distribution \\   Discovery: from Theory to Molecules}

\author{%
  Riccardo De Santi \\
  ETH Zurich \\
  ETH AI Center\\
  \texttt{rdesanti@ethz.ch} \\
  \And
   Bruce Lee \\
   ETH Zurich \\
   ETH AI Center\\
   \texttt{bruce.lee@ai.ethz.ch} \\
  \And
   Cristian Perez Jensen\\
   ETH Zurich \\
   \texttt{cjense@ethz.ch} \\
   \And
   Kimon Protopapas \\
   ETH Zurich \\
   \texttt{kprotopapas@ethz.ch} \\
   \And
   Sophia Tang \\
   University of Pennsylvania \\
   \texttt{sophtang@seas.upenn.edu} \\
   \And
   Cheng-Hao Liu \\
   Caltech \\
   FutureHouse \\
   \texttt{chl@caltech.edu} \\
   \And
   Pranam Chatterjee \\
   University of Pennsylvania \\
   \texttt{pranam@seas.upenn.edu} \\
   \And
   Yisong Yue \\
   Caltech \\
   \texttt{yyue@caltech.edu} \\
   \And
   Andreas Krause \\
   ETH Zurich \\
   ETH AI Center\\
   \texttt{krausea@ethz.ch} \\
}

\begin{document}
\maketitle

\begin{abstract}
\looseness -1 Standard flow and diffusion pre-training matches the distribution of available data (e.g., molecules), which often covers only a small fraction of the valid design space. In generative discovery, however, one aims to sample valid new-to-nature designs, assigned negligible probability under, and thus inaccessible to, standard models fitted to the observed data. To overcome this limitation, we depart from data distribution matching and view a generative model through its \emph{generable set}: the region it covers with non-negligible probability. This allows to introduce a new learning principle for \emph{out-of-distribution flow modeling}: enlarging a model’s generable set to increase coverage of the valid design space.  We propose {\em \AlgNameDef}, a continued pre-training method that employs verifier feedback to expand a pre-trained model over new valid regions by iteratively adapting to synthetic data generated through active exploration in the learned flow representation. Theoretically, we establish to our knowledge first-of-their-kind statistical learning guarantees for out-of-distribution flow modeling, analyzing generable set expansion as a local-to-global reachability process over a learned representation. Empirically, we assess \AlgNameShort with suitable out-of-distribution generative modeling metrics across small organic molecules, mid-sized drug-like molecules, therapeutic peptides, and protein sequence design tasks. Results show that \AlgNameShort expands valid coverage far beyond the region modeled by the initial pre-trained model, significantly outperforming widely adopted synthetic flow pre-training methods.
\end{abstract}

\begin{table*}[h!]
\centering
\vspace{-1mm}
\small
\setlength{\tabcolsep}{4.0pt}
\renewcommand{\arraystretch}{1.12}
\begin{tabular}{
l
S[table-format=3.2\uncert{00.00}]
S[table-format=3.2\uncert{0.00}]
S[table-format=3.2\uncert{00.00}]
S[table-format=1.2\uncert{0.00}]
S[table-format=3.2\uncert{00.00}]
S[table-format=2.2\uncert{00.00}]
}
\toprule
& \multicolumn{2}{c}{\textbf{GEOM-Drugs Molecules}} 
& \multicolumn{2}{c}{\textbf{Therapeutic Peptides}} 
& \multicolumn{2}{c}{\textbf{Protein Sequences}} \\
\cmidrule(lr){2-3} \cmidrule(lr){4-5} \cmidrule(lr){6-7}
\textbf{Method} 
& {\textbf{Coverage $\uparrow$}} & {\textbf{Diversity $\uparrow$}}
& {\textbf{Coverage $\uparrow$}} & {\textbf{Diversity $\uparrow$}}
& {\textbf{Coverage $\uparrow$}} & {\textbf{Diversity $\uparrow$}} \\
\midrule
Pre-trained 
& 35.89\uncert{2.05} & 255.03\uncert{1.09}
& 44.33\uncert{3.46} & 13.45\uncert{0.14}
& 66.50\uncert{5.63} & 12.87\uncert{0.63}
\\
\AlgRecNF 
& 44.67\uncert{2.36} & 267.10\uncert{2.88}
& 0.00\uncert{0.00} & 0.00\uncert{0.00} 
& 63.75\uncert{11.45} & 11.85\uncert{0.34}
\\
\AlgRecF 
& 89.33\uncert{11.56} & 284.30\uncert{2.62}
& 59.67\uncert{17.97} & 13.62\uncert{4.34}
& 49.50\uncert{9.60} & 11.67\uncert{0.40}
\\
\rowcolor{softbluebg}
\AlgNameShort 
& \bfseries 144.30\uncert{19.28} & \bfseries 303.10\uncert{5.71}
& \bfseries 358.33\uncert{95.45} & \bfseries 58.87\uncert{25.98}
& \bfseries 102.75\uncert{18.36} & \bfseries 42.14\uncert{10.85}
\\
\bottomrule
\end{tabular}
\caption{\AlgNameShort expands model coverage (\ie number of valid clusters) and diversity (\ie Vendi score~\citep{friedman2022vendi} estimated via valid samples) across domains from de novo 3D molecular design to protein sequence design, significantly outperforming widely adopted recursive self-generation baselines.}
\label{tab:front_summary}
\vspace{-2mm}
\end{table*}

\addtocontents{toc}{\protect\setcounter{tocdepth}{-1}}
\section{Introduction}\vspace{-0mm}

\begin{figure*}[t!]
    \centering
    \includegraphics[
        width=\textwidth,
        keepaspectratio,
        trim=0 190mm 0 00mm,
        clip
    ]{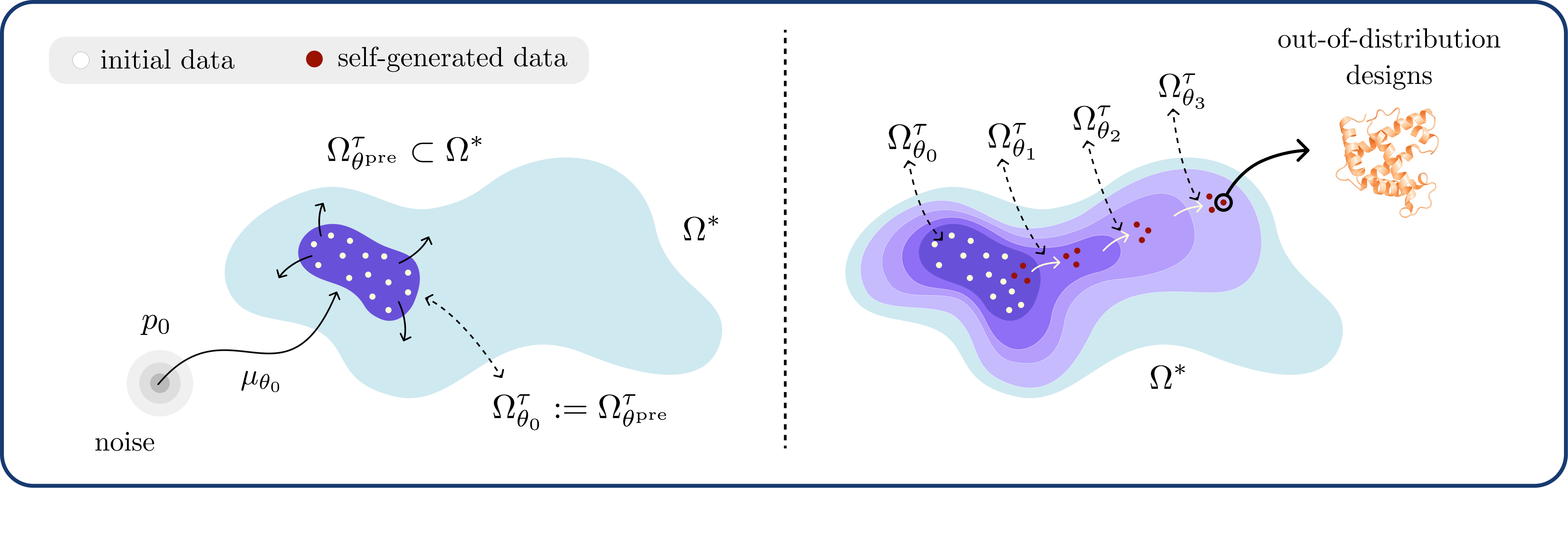}
    \vspace{-4mm}
    \caption{\looseness -1 (left) Pre-trained flow model $\mu_{\theta_0}$ generates with sufficient probability a set $\Omega_{\theta_0}^\tau$ poorly covering the valid design space $\Omega^*$, i.e., $\Omega_{\theta_0}^\tau \subset \Omega^*$. (right) \AlgNameDef increases model coverage over new valid regions of $\Omega^*$, enabling out-of-distribution flow modeling.}
    \label{fig:top_figure} \vspace{-4mm}
\end{figure*}

\looseness -1 Large-scale generative modeling has advanced rapidly in recent years, with flow~\citep{lipman2022flow, lipman2024flow} and diffusion models~\citep{sohl2015deep, song2019generative, ho2020denoising} emerging as powerful approaches for generating high-fidelity samples across domains including chemistry~\citep{hoogeboom2022equivariant}, biology~\citep{corso2022diffdock}, and robotics~\citep{chi2025diffusion}. Alongside this progress, over the last years, a growing literature has introduced test-time reward-tilting techniques~\citep[\eg][]{uehara2025inference, jensen2026value, uehara2024understanding, santi2025flow}. These advances have expanded the use of generative models in scientific discovery, enabling applications such as molecular property optimization~\citep[][]{gutjahr2025constrained}, or functional enzyme design~\cite{rector2026general}.

\vspace{-0.5mm}
\mypar{A Core Limitation of Standard Flow Pre-training for Discovery}
\looseness -1 Despite recent progress, generative discovery remains largely constrained by standard pre-training. Flow and diffusion models are learned by matching the distribution of available data~\citep[][]{song2020score, lipman2024flow}, which typically covers only a limited portion of the valid design space. This objective is natural when the goal is to reproduce observed examples, \eg in robotics for manufacturing, where behavioral cloning is often sufficient~\citep[][]{pearce2023imitating}, but it is poorly suited to out-of-distribution discovery, where one typically seeks valid new-to-nature designs with negligible probability under the data distribution, and arbitrarily far from available data acquired via nature's evolution and prior discoveries. In other words, the goal of out-of-distribution discovery is fundamentally opposed to standard generative modeling as distribution matching. Nonetheless, pre-trained generative models remain, to date, arguably the most promising way to access the complex design spaces in which scientific discovery takes place. This raises the central question of this work:\vspace{-2mm}
\begin{center}
\emph{How can we adapt a pre-trained flow or diffusion model in a task-agnostic way to enable out-of-distribution discovery?}
\end{center}
\vspace{-2mm}
\looseness -1 Answering this question would advance the algorithmic-theoretical foundations of generative discovery, with direct implications for high-impact scientific discovery applications. We take a step toward this goal by moving beyond data distribution matching and \emph{rethinking flow model learning} for out-of-distribution generative modeling. Concretely, we make the following contributions.\vspace{-0mm}

\mypar{Our contributions}\vspace{-1mm}
\begin{itemize}[noitemsep,topsep=0pt,parsep=5pt,leftmargin=16pt]
    \item \looseness -1 We introduce the notion of \emph{generable set}: the region of design space a model can access with non-negligible probability. This yields a rigorous mathematical framework for \emph{out-of-distribution flow modeling}, and a new learning principle for continued pre-training: expanding a model generable set to cover a larger portion of the valid design space, which we name \emph{generable set expansion}.\vspace{-0mm}
    \item We propose \AlgNameDef, a continued pre-training method that expands a pre-trained flow model through recurrent adaptation on self-generated synthetic data --  as illustrated in Fig. \ref{fig:top_figure} (right). \AlgNameShort actively explores in the \emph{learned diffusion representation} at intermediate noise levels, where the geometry is more favorable to global exploration.\vspace{-0mm}
    \item \looseness -1 We establish, to our knowledge, first-of-their-kind statistical learning guarantees for out-of-distribution flow modeling. These rely on an energy-based modeling abstraction, and analyze active generable set expansion as a local-to-global reachability process in a learned representation.\vspace{-0mm}
    \item \looseness -1 We employ domain-agnostic metrics for out-of-distribution generative modeling and assess \AlgNameShort across small organic molecules, mid-sized drug-like molecules, therapeutic peptides, and protein design. \AlgNameShort significantly increases valid coverage beyond the initial pre-trained model, vastly outperforming widely used recursive data generation baselines. We further show that the same algorithmic principles and empirical gains extend to discrete (diffusion) models. \vspace{-0mm}
\end{itemize}

\vspace{-0mm}
\section{From Data Distribution Matching to Out-of-Distribution Flow Modeling}\vspace{-0mm}
\label{sec:problem}
\looseness -1 We consider a design space $\X \subseteq \R^d$, where each $x \in \X$ is a candidate design, such as a molecule or protein. Let $p_{data}$ denote the distribution of the available data, assumed to consist of i.i.d.\ samples that satisfy the domain-specific notion of validity, e.g., valid molecules. A continuous-time flow model is defined by a time-dependent velocity field
\(u_\theta : \X \times [0,1] \to \X\) and the ordinary differential equation\vspace{-1mm}
\begin{equation}
\frac{d x_t}{dt} = u_\theta(x_t,t),
\qquad x_0 \sim p_0, 
\label{eq:flow_ode}
\end{equation}
where $p_0$ is a simple base distribution, typically Gaussian~\citep[][]{lipman2024flow}. Let \smash{$p_{t}^\theta$} denote the marginal density induced at time $t$ by the flow, and let \smash{$p_1^\theta$} denote its terminal density. In standard flow matching pre-training, one specifies conditional interpolation paths $p_t(\cdot \mid x_0,x_1)$ between $x_0 \sim p_0$ and $x_1 \sim p_{data}$, with associated target velocity field $u_t^\star(\cdot \mid x_0,x_1)$, and learns $u_\theta$ by minimizing the loss:
\begin{equation}
\mathbb{E}\!\left[\|u_\theta(x_t,t)-u_t^\star(x_t \mid x_0,x_1)\|_2^2\right],
\qquad
t \sim U[0,1],\;
x_t \sim p_t(\cdot \mid x_0,x_1).
\label{eq:fm_obj}
\end{equation}
\looseness -1 This vector-field regression objective allows to match the terminal distribution \smash{$p_1^\theta$} to the available-data distribution $p_{data}$, as by Eq.~\ref{eq:data_distribution_matching}. Beyond flow matching, the same distribution-matching learning principle underlies also standard pre-training of continuous~\citep{song2020score} and discrete~\citep{lou2023discrete} diffusion models.
\vspace{-0mm}
\begin{tcolorbox}[colback=softbluebg, colframe=oxfordblue, top=2pt,left=2pt,right=2pt,bottom=2pt]
\begin{center}
\textbf{Distribution Matching: Standard Flow Model Pre-Training}
\begin{align}
    \text{learn } \theta \text{ such that } p_1^\theta \approx p_{data}.
    \label{eq:data_distribution_matching}
\end{align}
\end{center}
\end{tcolorbox}
\vspace{-0mm}

\vspace{-0mm}
\subsection{Generable Set Expansion: Flow Modeling Beyond Data Distribution Matching} \vspace{-0mm}
In this work, we depart from viewing generative models as data distribution approximators and instead cast them as \emph{valid design space approximators}. Let $v:\X \to \{0,1\}$ denote whether a design is valid according to the  domain one wishes to model, e.g., $v$ might express molecular physiochemical validity, protein sequence foldability, etc. We indicate via $\Omega^*$ the entire valid design space:\vspace{-0mm}
\begin{align}
    \Omega^* \coloneqq \{x \in \X : v(x)=1\},
\end{align}
\looseness -1 which we wish our generative model to cover well. To capture the data regions sufficiently covered by a pre-trained flow model $\mu_\theta$, we introduce the following notion of \emph{generable set}.
\vspace{-0mm}
\begin{tcolorbox}[colback=softbluebg, colframe=oxfordblue, top=2pt,left=2pt,right=2pt,bottom=2pt]
\begin{restatable}[Generable Set]{definition}{defGenerableSet}
\label{definition:generable_set}
Let $\mu_\theta$ be a flow model inducing terminal density $p^\theta_1$ on $\X$. For $\tau > 0$, we define its $\tau$-level \emph{generable set} as
\begin{align}
    \Omega^\tau_{\theta} \coloneqq \{x \in \X \mid p^\theta_1(x) \geq \tau\}.
\end{align}
\end{restatable}
\end{tcolorbox}
\vspace{-0mm}
\looseness -1 For $\tau \to 0$ the generable set approximates the model support, \ie $\Omega^0_{\theta} \coloneqq \mathrm{Supp}(p_1^\theta)$. However, the support can include regions assigned only infinitesimal, yet positive probability, and therefore may substantially overstate the model's effective coverage. Instead, for sufficiently small $\tau>0$, the generable set $\Omega^\tau_{\pi}$ captures the region likely to be sampled under a finite budget. 
With this notion of generable set, we can state the following limitation of standard pre-trained flows for out-of-distribution discovery \footnote{For clarity, here we consider a model generating only valid designs. In general $\Omega^\tau_{\theta^{\mathrm{pre}}} \cap  \left(\X \backslash \Omega^*\right)$ is non-empty (Apx. \ref{sec:imperfect_model_discussion_apx})}: 
\vspace{-0mm}
\begin{tcolorbox}[colframe=white!, top=2pt,left=2pt,right=2pt,bottom=2pt]
\begin{center}
\textbf{Central Limitation of Standard Flow Pre-training}
\begin{align}
    \Omega^\tau_{\theta^{\mathrm{pre}}} \subset \Omega^*
\qquad 
\text{with } \qquad \mathrm{Vol}(\Omega^\tau_{\theta^{\mathrm{pre}}}) \ll \mathrm{Vol}(\Omega^*).
    \label{eq:representational_limit}
\end{align}
\end{center}
\end{tcolorbox}
\vspace{-0mm}

Under exact distribution matching, Eq.~\eqref{eq:representational_limit} allows to define the valid out-of-distribution (OOD) region:\vspace{-0mm}
\begin{equation}
\text{Valid OOD region:} \quad \bar{\Omega}^\tau_{\theta^{\mathrm{pre}}} \coloneqq \Omega^* \setminus \Omega^\tau_{\theta^{\mathrm{pre}}}
\end{equation}
\looseness -1 This is the set of valid designs not reliably covered by the pre-trained model $\mu_{\theta_{\mathrm{pre}}}$ and typically not represented in $p_{\mathrm{data}}$, such as new-to-nature designs. To overcome this limitation, we introduce \emph{generable set expansion} (see Eq. \ref{eq:valid_design_space_approximation}), a new learning principle for continued pre-training: given a standard pre-trained flow model, expand its generable set to cover a larger portion of the valid design space \(\Omega^*\).

\vspace{0mm}
\begin{tcolorbox}[colback=softbluebg, colframe=oxfordblue, top=2pt,left=2pt,right=2pt,bottom=2pt]
\begin{center}
\textbf{Generable Set Expansion: Out-of-Distribution Flow Modeling}
\begin{align}
    \text{learn } \theta \text{ such that } \Omega^\tau_{\theta} \approx \Omega^*.
    \label{eq:valid_design_space_approximation}
\end{align}
\end{center}
\end{tcolorbox}
\vspace{-0mm}

\vspace{-0mm}
\subsection{Generable Set Expansion via Iterative Continued Pre-training}
\vspace{-0mm}

\looseness -1 Given a pre-trained flow model $\mu_{\theta^{\mathrm{pre}}}$, we approach
generable set expansion through continued pre-training on self-generated data.
Beyond the data distribution, validity must be obtained through black-box verifier
queries, e.g., molecular validity~\citep{de2026verifier, guo2024takes} or protein
foldability~\citep{jumper2021highly, watson2023novo}. Since such queries can only
be made on generable designs, expansion is intrinsically recursive, as shown in Fig. \ref{fig:top_figure}:
\[
    \theta^{\mathrm{pre}}=\theta_0\to\theta_1\to\cdots\to\theta_T,
    \qquad
    \Omega^\tau_{\theta_0}\subseteq\Omega^\tau_{\theta_1}
    \subseteq\cdots\subseteq\Omega^\tau_{\theta_T}\approx\Omega^\star .
\]
\looseness - 1 The central algorithmic question is therefore how to self-generate data that expands the model’s generable set, rather than merely reinforcing regions already assigned high density. Standard sampling is poorly suited to this goal: it preferentially draws from dominant modes, and may therefore neglect valid frontier regions that enable valid expansion. We report an illustrative Gaussian analysis in Apx.~\ref{sec:warm_up_gaussian_analysis} to isolate this failure mode in a minimal setting. The analysis shows that, even when expansion-enabling samples remain within the current generable set, passive self-generation can be exponentially unlikely to sample them when valid directions of expansion are sparse. This highlights the core algorithmic requirement: self-generation must steer sampling toward generable and valid \emph{frontier} regions, rather than merely reinforce dominant modes. In the next section, we introduce \AlgNameShort, which implements this principle and yields provable generable set expansion guarantees.

\vspace{-0mm}
\section{Algorithm: \AlgNameLong}\vspace{-0mm}
\label{sec:algorithm}
\looseness -1 We now introduce \AlgNameDef\ (Alg.~\ref{alg:main_algorithm}), a continued pre-training method for \emph{generable set expansion}: it expands a pre-trained flow model's generable set to cover a larger portion of the valid design space. Concretely, \AlgNameShort fine-tunes the model on self-generated data acquired via inference-time active exploration in a learned flow representation at intermediate noising levels.

\vspace{-0mm}
\paragraph{High-level Algorithm Summary.}
\looseness -1 At round $t$, \AlgNameShort fits a verifier uncertainty estimate $\sigma_t(\cdot)$ from a buffer $\mathcal D_t=\{(x_i,y_i)\}_{i=1}^t$, self-generates a new candidate, or a batch, $x_{t+1}$ via Eq.~\eqref{eq:actflow_kl_control}, queries the verifier to obtain $y_{t+1}=\tilde v(x_{t+1})$, appends $(x_{t+1},y_{t+1})$ to $\mathcal D_t$, and updates the flow into $\mu_{\theta_{t+1}}$.

\vspace{-0mm}
\paragraph{Active Exploration over a Noised Flow Representation.}

\looseness -1 \AlgNameShort performs active self-generation in the learned representation $\Z_s$, given by hidden features of the velocity network $\mu_\theta$ at noising level $s\in(0,1)$. At each iteration, Algorithm  \ref{alg:main_algorithm} steers the current model at inference-time toward high verifier-uncertainty regions in $\Z_s$, while remaining close to the current model density:\vspace{0mm}
\begin{tcolorbox}[colframe=white!, top=2pt,left=2pt,right=2pt,bottom=2pt]
\centering
\textbf{(Inference-Time) Active Exploration over Noised Flow Representation $\Z_s$}
\begin{equation}
\begin{aligned}
x_{t+1} \sim \tilde p_t
&\in
\argmax_q 
\mathbb{E}_{x \sim q} \big[\sigma_t \bigl(\phi^t_s(x)\bigr)\big]
-
\beta \,\mathrm{KL}\!\left(q \,\|\, p^{\theta_t}_1\right).
\end{aligned}
\label{eq:actflow_kl_control}
\end{equation}
\end{tcolorbox}
\vspace{-0mm}

\looseness -1 Here, $\phi_s^t:\X\to\Z_s$ denotes the representation extracted from the velocity network at noising level $s \in (0,1)$. The first term of Eq. \ref{eq:actflow_kl_control} favors informative queries about the verifier \(v\), and thus about the valid design space \(\Omega^*\). The KL term regularizes search toward the current generative prior, biasing exploration toward likely valid regions. The parameter $\beta$ controls the exploration--prior trade-off: as $\beta\to\infty$, Eq.~\eqref{eq:actflow_kl_control} recovers standard sampling, \smash{$x_{t+1}\sim p_1^{\theta_t}$}; as $\beta\to0$, it becomes pure uncertainty maximization, which may target remote regions where the model prior might no longer provide a useful validity bias.

To instantiate Eq.~\eqref{eq:actflow_kl_control}, we model verifier uncertainty directly in $\Z_s$. 
We then view the verifier labels as noisy observations $y_t=\tilde v(x_t)$  of an
unknown validity function over the representation space \(\Z_s\). To represent the uncertainty about this function, we use a linear kernel over the learned representation space: $k_{\phi_s^t}(x,x') := \langle \phi_s^t(x),\phi_s^t(x')\rangle$. 
Then, for a set of queried designs $x_1,\dots,x_t$, let $X_t:=(x_1,\dots,x_t)$, let $K_{t, \phi_{s,t}}\in\mathbb{R}^{t\times t}$ be the kernel matrix with entries $(K_t)_{ij}=k_{\phi_s^t}(x_i,x_j)$, and write $k_{\phi_s^t}(x,X_t):=\bigl(k_{\phi_s^t}(x,x_1),\dots,k_{\phi_s^t}(x,x_t)\bigr)$. We then express our uncertainty about the unknown verifier via the following closed-form expression:\vspace{-0mm}
\begin{equation}
\sigma_t^2(x)
=
k_{\phi_s^t}(x,x)-k_{\phi_s^t}(x,X_t)\bigl(K_{t,\phi_s^t}+\lambda I\bigr)^{-1}k_{\phi_s^t}(X_t,x),
\qquad
\sigma_t(x):=\sqrt{\sigma_t^2(x)}.
\end{equation}
\looseness -1 This corresponds to the posterior variance in Bayesian linear
regression under Gaussian observations. 

\newtcbox{\algokeybox}{
  on line,
  colback=softbluebg,
  colframe=softbluebg,
  boxrule=0.4pt,
  arc=2pt,
  left=5pt,right=5pt,top=2.5pt,bottom=2.5pt,
  boxsep=1pt
}

\begin{algorithm}[t]
\caption{\AlgNameDef}
\label{alg:main_algorithm}
\begin{algorithmic}[1]
\setlength{\baselineskip}{1.1\baselineskip}
\Require Flow model $\mu_{\theta_0}$, black-box verifier $
\tilde v$, representation time-step $s \in (0,1)$, iterations $T$
\State $\mathcal D_0 \gets \emptyset$
\For{$t=0,1,\dots,T-1$}
    \State Update surrogate uncertainty $\sigma_t$ from $\mathcal D_t$ 
    \State Self-generate:
    \Statex \mbox{}\hfill \algokeybox{$x_{t+1} \sim \tilde p_t
        \in
        \argmax_q 
        \mathbb{E}_{x \sim q}\big[\sigma_t \big(\phi^t_s(x)\big)\big]
        -
        \beta \,\mathrm{KL}\!\left(q \,\|\, p^{\theta_t}_1\right)$} \hfill\mbox{} \vspace{0.7mm}
    \State Query verifier: $y_{t+1}\gets \tilde v(x_{t+1})$
    \State $\mathcal D_{t+1} \gets \mathcal D_t\cup \{(x_{t+1}, y_{t+1})\}$
    \State $\theta_{t+1}\gets \textsc{UpdateFlow}(\theta_{t},\mathcal D_{t+1})$
\EndFor
\State \Return $\mu_{\theta_T}$
\end{algorithmic}
\end{algorithm}

\vspace{-0mm}
\paragraph{Model update.}
The uncertainty estimator is fit on the pairs in $\mathcal D_t$. The flow
model is then updated by replay-based continued pre-training. Let
\smash{$\mathcal D_t^+=\{x:(x,1)\in\mathcal D_t\}$} and
\smash{$\mathcal D_t^-=\{x:(x,0)\in\mathcal D_t\}$} denote accepted and rejected samples,
and let \smash{$U_t^\pm\subseteq\mathcal D_t^\pm$} be minibatches and \smash{$\widehat L_t^\pm(\theta)$} their respective standard flow-matching loss terms.
We use the signed update\smash{
$
g_t=\nabla\widehat L_t^+(\theta_t)-\alpha_t\nabla\widehat L_t^-(\theta_t), 
$}
where rejected samples are seen as an unlearning signal~\citep{alberti2025data};
see Apx.~\ref{sec:loss_appendix}. In practice, our results often hold with
$\alpha_t=0$, i.e., standard flow matching~\citep{lipman2024flow} on
verifier-accepted samples.

\looseness -1 Across iterations, \AlgNameShort reallocates model mass from dominant pre-trained modes toward newly discovered valid regions. It remains to show that this process truly expands the model’s generable set, rather than only redistributing density within it. The next section answers this question affirmatively by establishing statistical guarantees for out-of-distribution generative modeling via \AlgNameShort.

\vspace{-0mm}
\section{Statistical Guarantees for Out-of-Distribution Flow Modeling}\vspace{-0mm}
\label{sec:theory}
We now theoretically analyze \AlgNameDef, thereby providing a first guarantee for out-of-distribution
generative modeling in terms of generable set expansion rather than distribution matching of $p_{\text{data}}$. Formal theorem and detailed derivations are reported in Appendix.~\ref{app:ade analysis}. 

As a first step, since \AlgNameShort effectively performs active exploration in the learned representation $\Z_s$, we introduce the following notion of \emph{generable representation set} $\Omega^\tau_{\pi,\phi}$ for a fixed map $\phi \coloneqq \phi_s$.

\vspace{-0mm}
\begin{tcolorbox}[colback=softbluebg, colframe=oxfordblue, top=2pt,left=2pt,right=2pt,bottom=2pt]
\begin{restatable}[Generable Representation Set]{definition}{defGenerableRepresentationSet}
\label{definition:generable_representation_set}
Let $\phi:\X\to\Z$ be fixed, and let $\pi$ be a diffusion model inducing design-space density \smash{$p^\pi_1$} on $\X$.
Let \smash{$p^{\pi,\phi}_1$} denote the induced density on $\Z$ obtained by pushing forward $p^\pi_1$ through $\phi$.
For $\tau > 0$, we define its $\tau$-level \emph{generable representation set} as
\begin{align}
    \Omega^\tau_{\pi,\phi} \coloneqq \{ z \in \Z \mid p^{\pi,\phi}_1(z) \geq \tau \}.
\end{align}
\end{restatable}
\end{tcolorbox}
\vspace{-0mm}

\vspace{-0mm}
\paragraph{Formal feedback model: logistic bandit feedback}
\label{sec:theory:logistic_bandit}

We now formalize the verifier as a probabilistic model over $Z=\phi(X)$.
At each round $t$, \AlgNameShort generates a sample $z_t \coloneqq \phi(x_t)$ and observes a binary label $y_t\in\{0,1\}$,
generated according to a logistic model:
\begin{equation}
\Pr[y_t=1 \mid z_t] = s(g(z_t))
\label{eq:logistic_feedback}
\end{equation}
where $g:Z\to\mathbb{R}$ is an unknown latent \emph{validity score}, and $s: \R \to [0,1]$ is the sigmoid function.
We assume there exists a threshold $h\in[0,1]$, such that we can formally define the valid design space as:
\begin{equation}
\Omega_* := \{z\in Z : s(g(z)) \ge h\}.
\label{eq:prob_valid_set}
\end{equation}
\looseness -1 For the sake of analysis, we assume $g$ lies in an RKHS $\mathcal{H}_k$ with $\|g\|_k\le B$, $\int_{\calZ} \exp(g(z)) \leq \bar Z$,
and is $L_g$-Lipschitz in a metric $d$ over $\mathcal{Z}$. This model allows us to construct \emph{anytime confidence sequences} for logistic prediction~\citep{pasztor2024bandits}.  Next, we model \AlgNameShort's uncertainty-tilted sampling procedure.

\vspace{-0mm}
\paragraph{Uncertainty-tilted sampling oracle modeling}
We model the inference-time uncertainty sampling in Eq. \eqref{eq:actflow_kl_control} as approximately maximizing verifier
uncertainty within the current generable set $\Omega_t^\tau$ for some $\tau > 0$. Concretely, at round $t$ it returns $x_t$ such that $z_t = \phi(x_t)\in \Omega_t^\tau$
satisfies:\vspace{-0mm}
\begin{equation}
\sigma_t(z_t)\ \ge\ \frac{1}{\alpha}\max_{z\in \Omega^\tau_t} \sigma_t(z),
\qquad \alpha \ge 1.
\label{eq:uncertainty_oracle}
\end{equation}
We depart from standard bandit analyses, which assume global search over $\Omega^\star$ by instead restricting the (generative) sampler to approximate uncertainty maximization over the generable set. We suppose that the $\tau$-level generable set of the pre-trained model contains a non-empty set $S_0$ of valid designs. Further regions can be reached only after intermediate expansion steps. We formalize a local expansion step via the following \emph{one-step reachability operator over the learned representation}:
\begin{equation} 
    R_\epsilon(S)
    :=
    \left\{
    z \in \Z :
    \exists z' \in S
    \ \text{s.t.}\
    s(g(z')) - L_s L_g d(z,z') - \epsilon \ge h
    \right\}. \label{eq:one_step_reachability_main}
\end{equation}
Here, \(R_\epsilon(S)\) contains the representations whose validity is certifiable
from \(S\) up to accuracy \(\epsilon\). Let \(R_\epsilon^H(S_0)\) denote its
\(H\)-fold recursive application starting from \(S_0\). This enables viewing of
\AlgNameShort as executing a local-to-global expansion process toward the reachable valid
set \(R_\epsilon^H(S_0)\). 

Given these objects, we can now present the main theorem, stated formally in Theorem \ref{thm:formal_main_theorem}, which allows to state set-theoretic guarantees for out-of-distribution generative modeling.
\vspace{-0mm}
\paragraph{(Informal) Assumptions:}
\begin{itemize}[leftmargin=*, itemsep=1mm, topsep=1mm]\vspace{-0mm}
    \item \textbf{(Well-specified verifier).}
    Verifier labels follow a calibrated logistic model in the learned representation space,
    $\mathbb P(y=1\mid z)=s(g(z))$, where the latent validity score $g$ has bounded
    RKHS norm and is Lipschitz continuous in the learned representation.

    \item \textbf{(Approximate and local uncertainty sampling).}
    At each round, the sampler approximately maximizes verifier uncertainty over the
    current generable set $\Omega_t^\tau$, as by Eq. \eqref{eq:uncertainty_oracle}.

    \item \textbf{(EBM generative update).}
    \looseness -1 We employ an energy-based model (EBM) abstraction, detailed in Assump. \ref{asmp: EBM}, to abstract flow model endpoint density updates as learning an implicit energy functional over the flow-learned representation $Z$. This captures that flow models preserve high density on regions that our verifier regards as valid, and maintain low density on points far from this set.
\end{itemize}
\vspace{-0mm}
\begin{tcolorbox}[colback=softbluebg, colframe=oxfordblue, top=2pt,left=2pt,right=2pt,bottom=2pt]
\begin{theorem}[(Informal) Generable representation set covers reachable set]
\label{thm:main_reachability}
Fix $\epsilon>0$ and an integer $H\ge 1$.
Let $\tau$ and $T^*$ satisfy\vspace{-0mm}
\begin{align}
    \label{eq: sample complexity}
    \tau
    \leq
    \frac{h}
    {
       (1-h) \bar Z
    } \quad \mbox{  and  } \quad T^\star \ \gtrsim\
\left(\frac{\alpha \gamma^{\Omega_\star}_{H T^*}}
{\epsilon}\right)^2, 
\end{align}
with up to problem-dependent constants in $\gtrsim$. Define the maximum information gain from $t$ samples by  \vspace{-0mm}
\begin{align*}
\gamma_t^A
\triangleq
\max_{z_1,\ldots,z_t \in A}\;
\frac{1}{2}\log\det \Bigl(I_t + (\lambda\kappa)^{-1}K_t\Bigr)
\end{align*}
By running \AlgNameShort, it holds with probability at least $1-\delta$ that after $T^\star$ verified samples,  \vspace{-0mm}
\begin{equation}
R_\epsilon^H(S_0) \subseteq\ \Omega^\tau_{T^\star}.
\label{eq:main_inclusion}
\end{equation}
\end{theorem}
\end{tcolorbox}
\vspace{-0mm}

Theorem~\ref{thm:main_reachability} is stated in the learned representation space $Z$. Under a  map $\phi$ with measure-preserving regularity conditions stated in Assumption \ref{ass:fixed_phi_jacobian}, the same reachability guarantee transfers to the design-space generable set. For $\tau_X>0$, define the design-space generable set at round $t$ by
\begin{equation}
\Omega^{X,\tau_X}_t := \{x\in\X:\; p^{\pi_t}_1(x)\ge \tau_X\}.
\label{eq:design_space_generable_set_cor}
\end{equation}
We also define the induced one-step reachability operator on $\X$ by
\begin{equation}
R^{X,\phi}_\epsilon(A)
:=
\Bigl\{
x\in \X:\exists x'\in A
\ \text{s.t.}\
s \bigl(g(\phi(x'))\bigr)-L_sL_g\,d \bigl(\phi(x),\phi(x')\bigr)-\epsilon \ge h
\Bigr\}, \label{eq:design_space_reachability_operator}
\end{equation}
for any $A\subseteq \X$, and let $\bigl(R^{X,\phi}_\epsilon\bigr)^H(A)$ denote its $H$-fold iterate.

\vspace{-0mm}
\begin{tcolorbox}[colback=softbluebg, colframe=oxfordblue, top=2pt,left=2pt,right=2pt,bottom=2pt]
\begin{restatable}[Design-space coverage of the induced reachable valid set]{corollary}{corDesignSpaceReachability}
\label{cor:design_space_reachability}
Assume the conditions of Theorem~\ref{thm:main_reachability} and Assumption~\ref{ass:fixed_phi_jacobian}, with $j_{\min}:=\inf_{x\in \X}|\det J_\phi(x)|>0$.
Let $S_0^X:=\phi^{-1}(S_0)$ and $\tau_X:=j_{\min}\tau$. Then, after the same number $T^\star$ of verified samples as in Theorem~\ref{thm:main_reachability}, with probability at least $1-\delta$,
\begin{equation}
\bigl(R^{X,\phi}_\epsilon\bigr)^H(S_0^X)\subseteq \Omega^{X,\tau_X}_{T^\star}. \label{eq:design_space_reachability_inclusion}
\end{equation}
\end{restatable}
\end{tcolorbox}

% --------
\vspace{-3mm}
\paragraph{From reachable-set to full coverage.}
The guarantee is reachability-based: if $\Omega^\star$ contains components
that are disconnected from $S_0^X$ in the learned geometry, local expansion provides no
guarantee of covering those components. Conversely, if every $x\in\Omega^\star$ can be
connected to $S_0^X$ by at most $H$ local valid expansions in representation space, then
Corollary~\ref{corollary:full_coverage_chain} yields full valid-space coverage,\vspace{-1mm}
\[
\Omega^\star \subseteq \Omega_{T^\star}^{X,\tau_X}.
\]
\looseness -1 Theorem~\ref{thm:main_reachability} shows that \AlgNameShort expands the
model's $\tau$-generable set for any sufficiently small $\tau$, effectively describing finite-sample coverage. Corollary~\ref{corollary:validity_lower_bound}
controls model validity, \ie how tight the expanded generable set is w.r.t.~$\Omega^\star$ -- the opposite set-inequality for \emph{generable set expansion}.

\vspace{-0mm}
\section{Experimental Evaluation of ActFlow on Molecules, Peptides, and Proteins} \vspace{-0mm}
\label{sec:experiments}
\begingroup
  \captionsetup[subfigure]{aboveskip=1.7pt, belowskip=0pt}
\newlength{\imgw}
\setlength{\imgw}{0.21\textwidth}
\begin{figure*}[t]
    \centering
    % row 1
    \begin{subfigure}{\imgw}
      \centering
      \includegraphics[width=\textwidth]{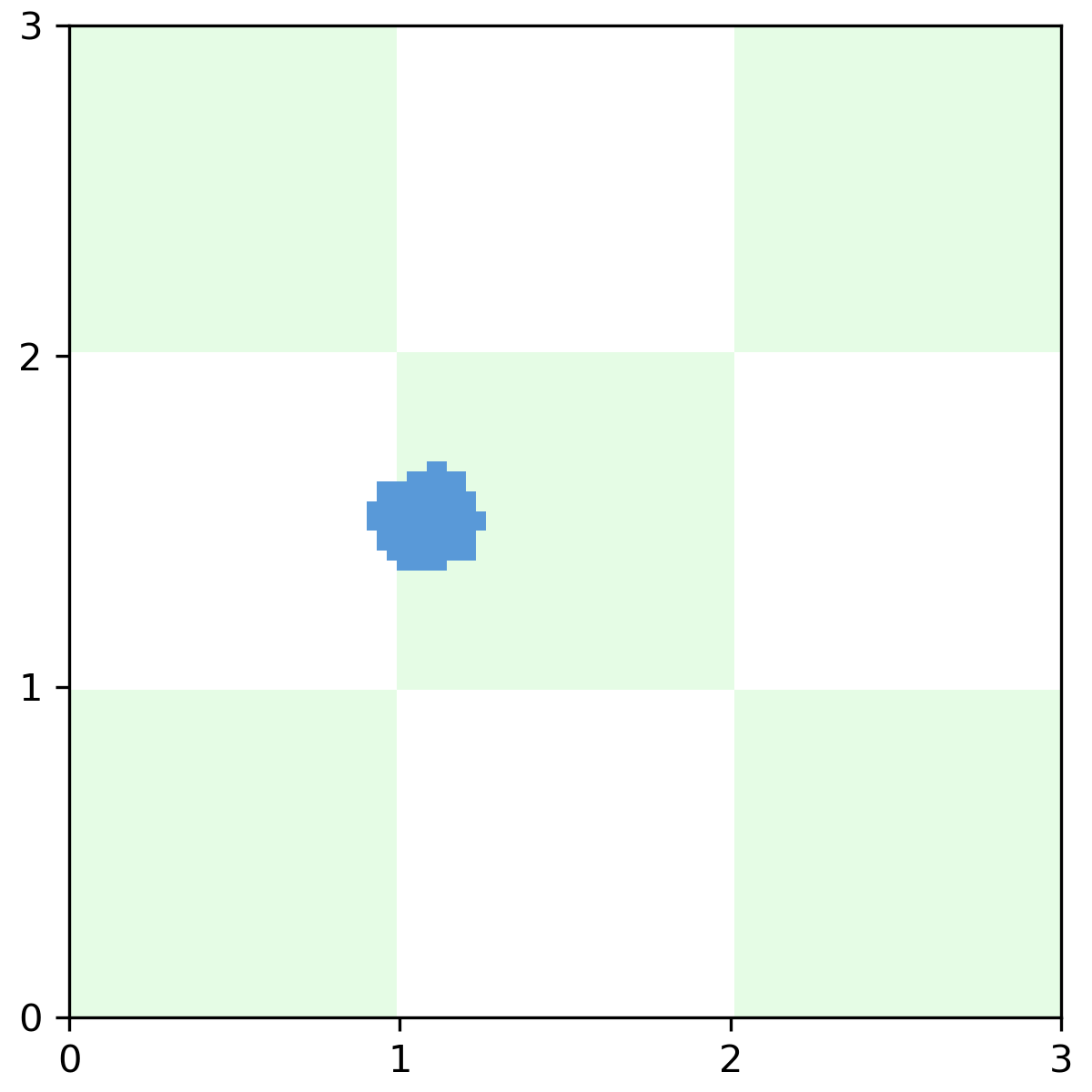}
      \caption{Pre-trained samples}
      \label{fig:toy_top_a}
    \end{subfigure}\hfill
    \begin{subfigure}{\imgw}
      \centering
      \includegraphics[width=\textwidth]{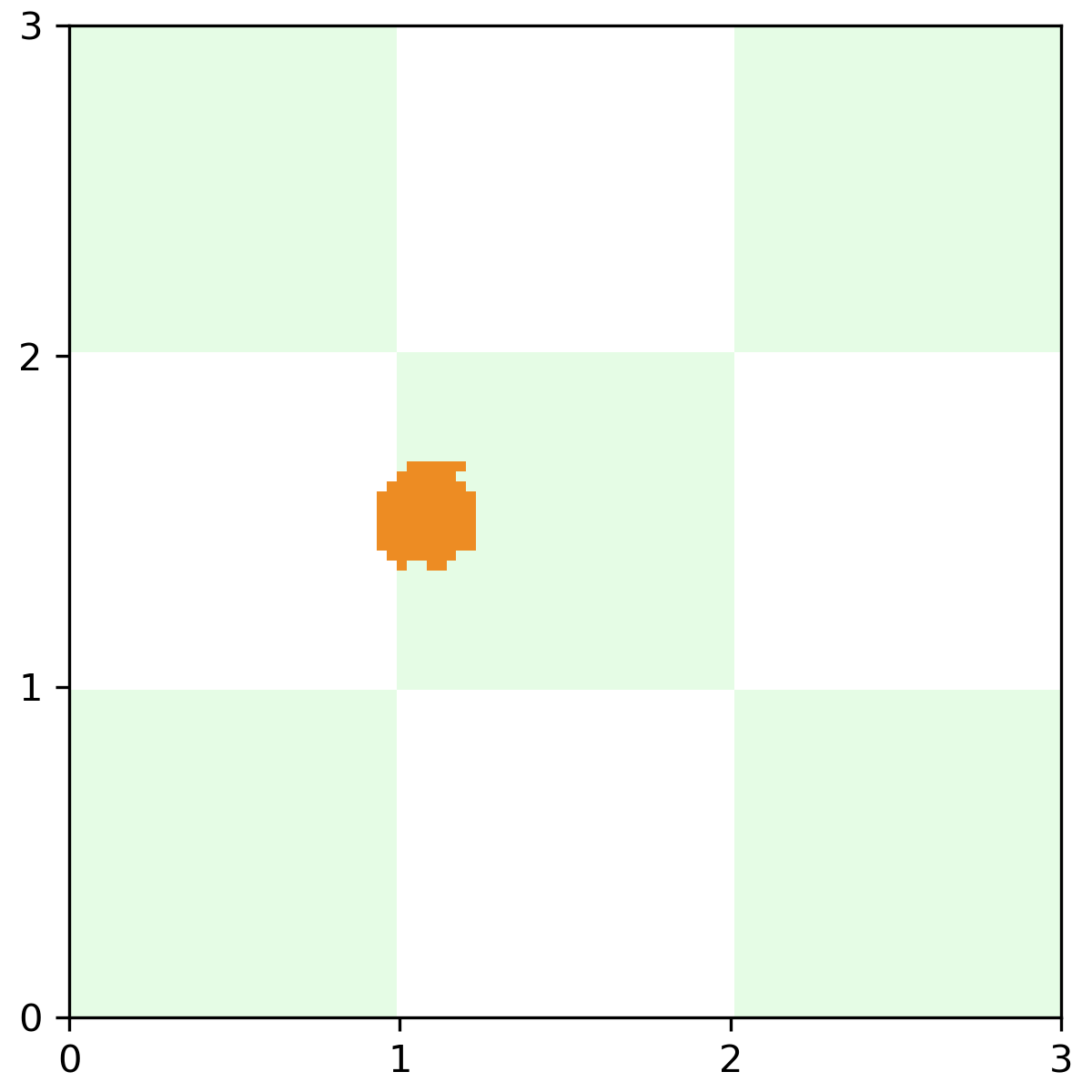}
      \caption{\AlgRecNF samples}
      \label{fig:toy_top_b}
    \end{subfigure}\hfill
    \begin{subfigure}{\imgw}
      \centering
      \includegraphics[width=\textwidth]{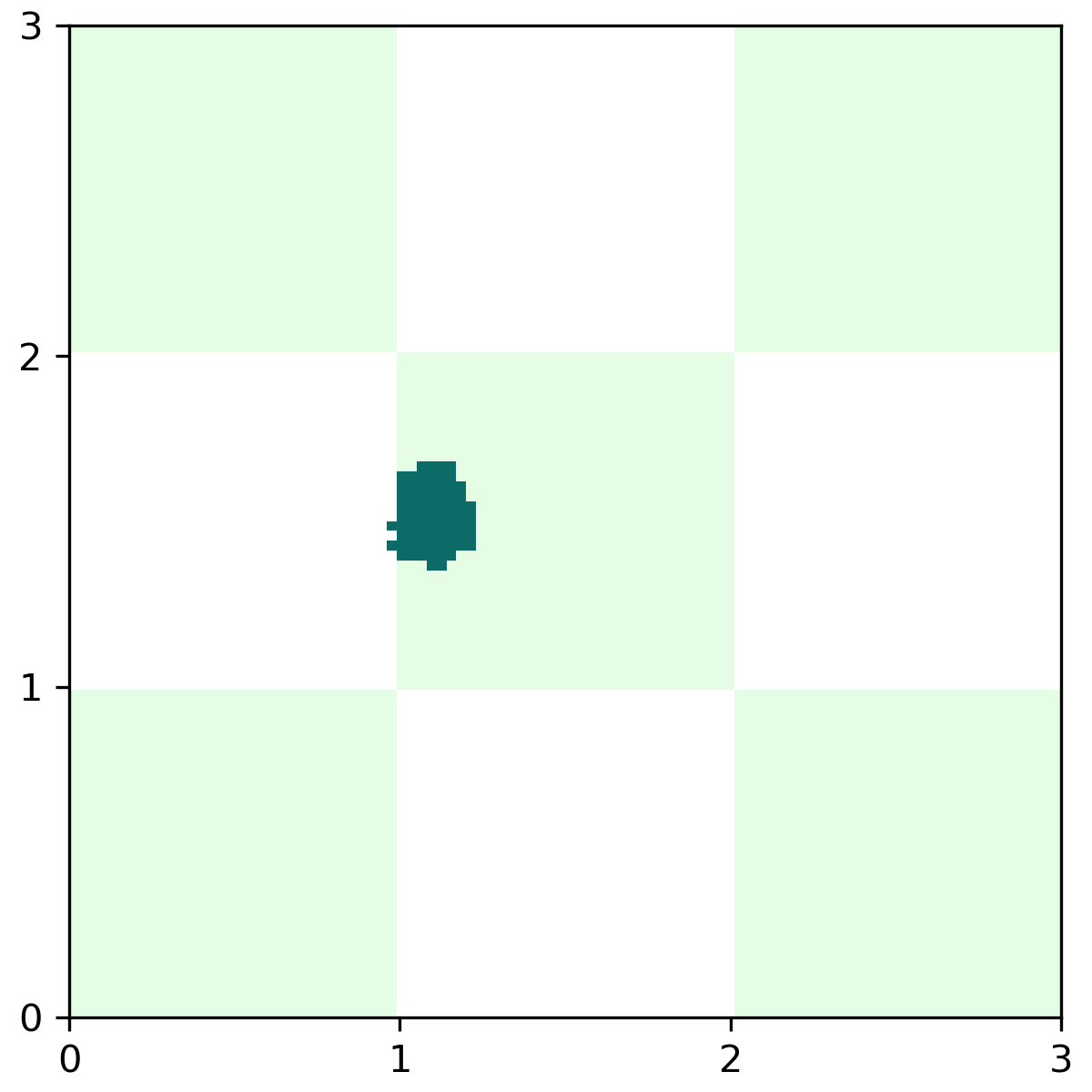}
      \caption{\AlgRecF samples}
      \label{fig:toy_top_c}
    \end{subfigure}\hfill
    \begin{subfigure}{\imgw}
      \centering
      \includegraphics[width=\textwidth]{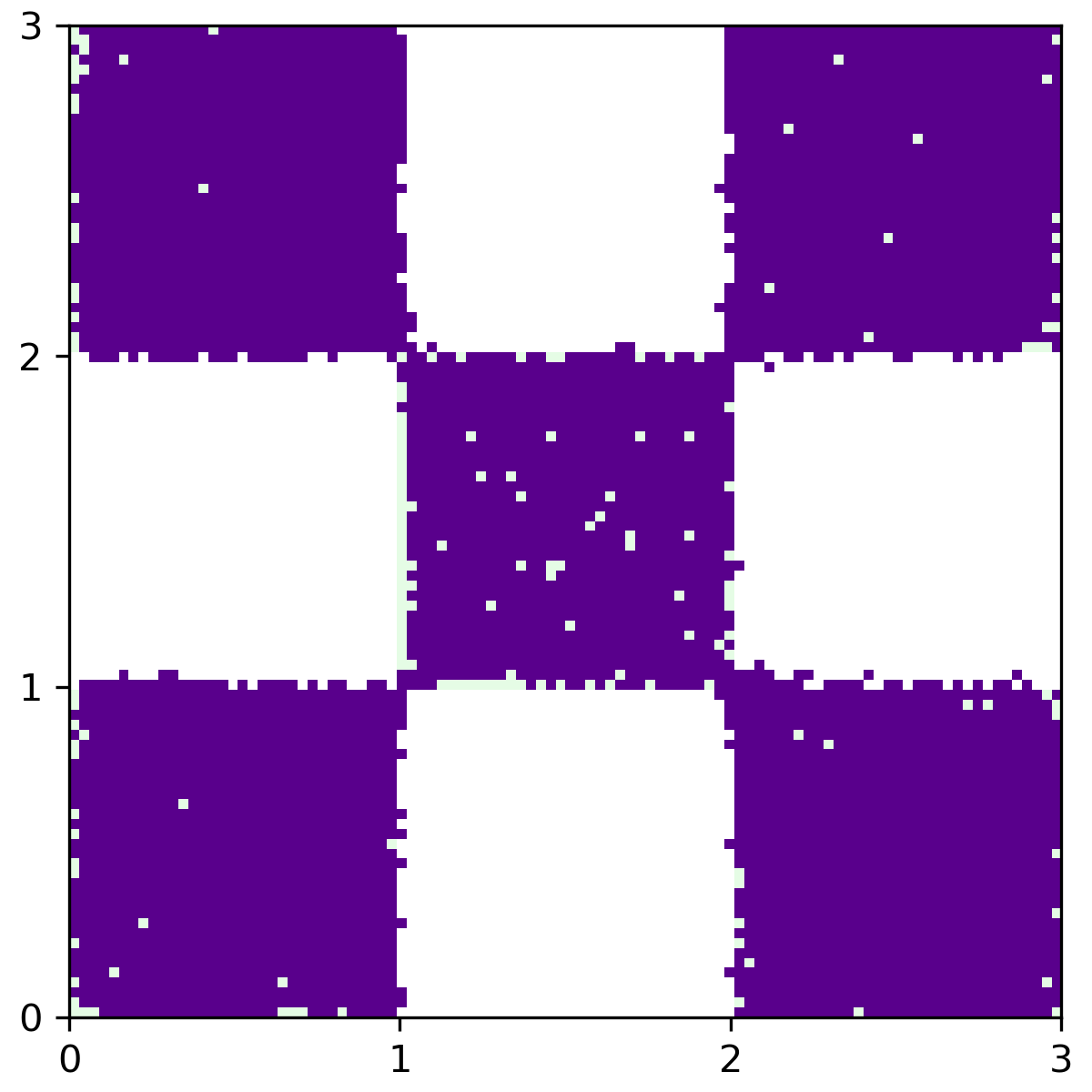}
      \caption{\AlgNameShort (ours)}
      \label{fig:toy_top_d}
    \end{subfigure}\hfill
    \\[0.4em]\vspace{-1mm}
    % row 2 (repeat)
    \setlength{\imgw}{0.33\textwidth}
    \begin{subfigure}{\imgw}
      \centering
      \includegraphics[width=\textwidth]{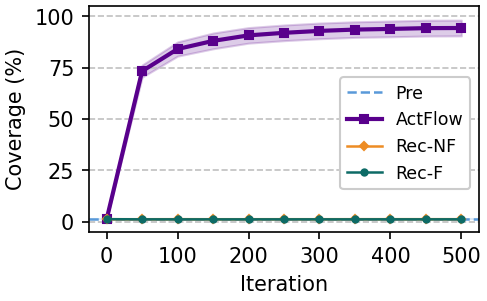}
      \caption{Coverage}
      \label{fig:toy_bottom_a}
    \end{subfigure}%
    \begin{subfigure}{\imgw}
      \centering
      \includegraphics[width=\textwidth]{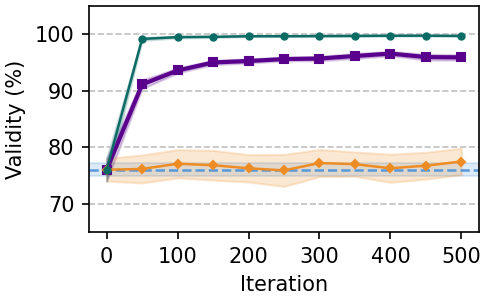}
      \caption{Validity}
      \label{fig:toy_bottom_b}
    \end{subfigure}%
    \begin{subfigure}{\imgw}
      \centering
      \includegraphics[width=\textwidth]{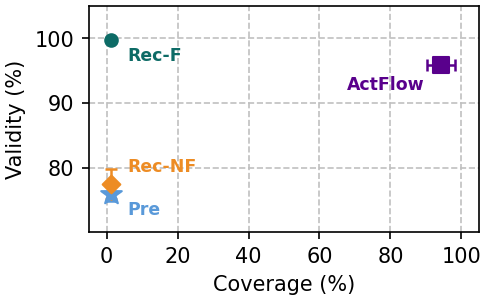}
      \caption{Coverage-Validity}
      \label{fig:toy_bottom_c}
    \end{subfigure} \vspace{-1mm}
    \caption{\looseness-1 (\ref{fig:toy_top_a}) The valid design space (green) and the estimated generable set ($\tau = 0.01)$ of a pre-trained model (blue). (\ref{fig:toy_top_a}-\ref{fig:toy_top_d}) \AlgRecNF and \AlgRecF fail to expand the pre-trained model's generable set, while \AlgNameShort is able to discover and expand over the entire valid design space. (\ref{fig:toy_bottom_a}-\ref{fig:toy_bottom_c}) \AlgNameShort increases both model coverage ($1.16\%$ to $94.27\%$) and validity ($76\%$ to $95.9\%$), \AlgRecNF fails at increasing both, while \AlgRecF increases validity while even decreasing coverage ($1.16\%$ to $1.1\%$).} 
    \label{fig:experiments_fig_2} \vspace{-0mm}
\end{figure*}
\endgroup

\setlength{\imgw}{0.25\textwidth}

\looseness -1 We evaluate \AlgNameShort for expanding the valid generable set of pre-trained generative models, and compare it against widely adopted self-generation baselines: continued pre-training on unfiltered model samples (\AlgRecNF)~\citep{shumailov2024ai, alemohammad2023self} and on verifier-filtered samples (\AlgRecF)~\citep{dong2023raft, gulcehre2023reinforced}. We propose two types of experiments: $(i)$ an illustrative visually interpretable setting, and $(ii)$ high-dimensional biochemical design tasks over molecules, therapeutic peptides, and protein sequences. Since standard generative modeling evaluation metrics are misaligned with OOD generative modeling, we employ the OOD modeling criteria presented in Apx. \ref{sec:evaluation_metrics_apx}, namely \emph{coverage} via number of valid clusters, \emph{diversity} via Vendi~\citep{friedman2022vendi}, and overall model validity. Additional experimental details are provided in Apx.~\ref{sec:experimental_details_appendix}.

\vspace{-0mm}
\paragraph{Illustrative Visually Interpretable Setting.}
\looseness -1 We first evaluate \AlgNameShort in a two-dimensional illustrative setting with valid design space shown in green in Fig.~\ref{fig:toy_top_a}--\ref{fig:toy_top_d}, and visualize the models' generable sets via discretization. All methods run for $T=500$ iterations and generate $B=64$ samples per iteration. \AlgNameShort uses $\alpha_t=0.005$, $\beta=1/13$, and flow representation at timestep $s=0.9$. Fig.~\ref{fig:toy_top_a}--\ref{fig:toy_top_d} show that \AlgNameShort (violet) expands the pre-trained generable set ($\tau = 0.01$, Fig.~\ref{fig:toy_top_a}) to near-optimally cover  the valid design space (Fig.~\ref{fig:toy_top_c}), whereas both baselines fail to expand it.  Fig.~\ref{fig:toy_bottom_a}--\ref{fig:toy_bottom_c} show that \AlgNameShort increases both coverage, from $1.16\%$ to $94.27\%$, and validity, from $76.00\%$ to $95.89\%$. By contrast, \AlgRecNF leaves coverage unchanged and barely improves validity, while \AlgRecF attains higher validity but decreases coverage. Thus, even in this simple two-dimensional setting, \AlgNameShort is the only method that achieves OOD generable-set expansion, substantially improving both coverage and validity, expanding through sparse directions (corners) to new valid regions.

\vspace{-0mm}
\paragraph{Molecular design on QM9.}
\looseness -1 We evaluate \AlgNameShort on FlowMol Gaussian~\citep{dunn2024mixed}, pre-trained on QM9~\citep{ramakrishnan2014quantum}. All methods run for $1000$ iterations after $66$ initial iterations without fine-tuning. \AlgNameShort uses the flow representation at timestep $s=0.9$ and $\beta=1/10$. Fig.~\ref{fig:qm9_a}--\ref{fig:qm9_d} show that, relative to the pre-trained model, \AlgNameShort substantially expands valid molecular coverage, reaching $88.40$ valid clusters, compared to $45.40$ for \AlgRecF and $42.80$ for \AlgRecNF. \AlgNameShort also achieves the highest diversity, with Vendi $306.08$, while preserving high validity ($95.90\%$). In contrast, \AlgRecF preserves validity ($95.24\%$) but expands significantly less, whereas \AlgRecNF severely degrades validity ($12.26\%$). Thus, on QM9, \AlgNameShort significantly outperforms both recursive sampling baselines, achieving a significantly stronger combination of valid coverage, diversity, and validity.

\vspace{-0mm}
\paragraph{Molecular design on GEOM-Drugs.}
\looseness -1 We further evaluate \AlgNameShort on FlowMol Gaussian~\citep{dunn2024mixed}, pre-trained on GEOM-Drugs~\citep{axelrod2022geom}, a substantially larger, and more chemically relevant dataset of drug-like molecules. All methods use $2000$ fine-tuning steps after a warm-up period, where $4096$ samples are acquired. \AlgNameShort uses flow representation timestep $s=0.8$, $\beta=1/7$, $\alpha_t = 0$. Fig.~\ref{fig:drugs_a}--\ref{fig:drugs_d} show that \AlgNameShort expands valid molecular coverage substantially beyond the baselines, reaching $144.3$ valid clusters, compared to $89.33$ for \AlgRecF and $44.67$ for \AlgRecNF. \AlgNameShort also achieves the highest diversity, with Vendi $303.1$, while maintaining comparable validity to \AlgRecF ($56.6\%$ versus $58.83\%$). In contrast, \AlgRecF expands substantially less, whereas \AlgRecNF collapses in validity, as reported in Fig. \ref{fig:drugs_d}. Thus, on GEOM-Drugs, \AlgNameShort provides a stronger expansion profile than both baselines, ultimately increasing the pre-trained model coverage by $144.3\%$ and its validity by $56.6\%$.

\begingroup
  \captionsetup[subfigure]{aboveskip=1.7pt, belowskip=0pt}
\setlength{\imgw}{0.25\textwidth}
\begin{figure*}[t]
    \centering
    \begin{subfigure}{\imgw}
      \centering
      \includegraphics[width=\textwidth]{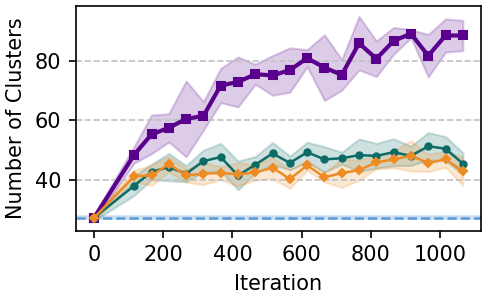}
      \caption{QM9 Coverage}
      \label{fig:qm9_a}
    \end{subfigure}%
    \begin{subfigure}{\imgw}
      \centering
      \includegraphics[width=\textwidth]{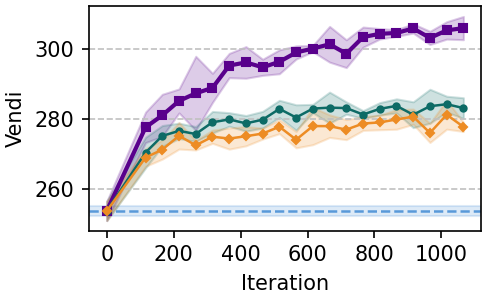}
      \caption{QM9 Diversity}
      \label{fig:qm9_b}
    \end{subfigure}%
    \begin{subfigure}{\imgw}
      \centering
      \includegraphics[width=\textwidth]{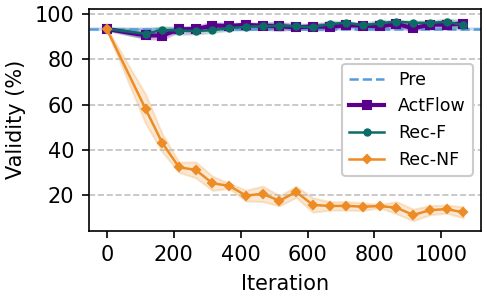}
      \caption{QM9 Validity}
      \label{fig:qm9_c}
    \end{subfigure}%
    \begin{subfigure}{\imgw}
      \centering
      \includegraphics[width=\textwidth]{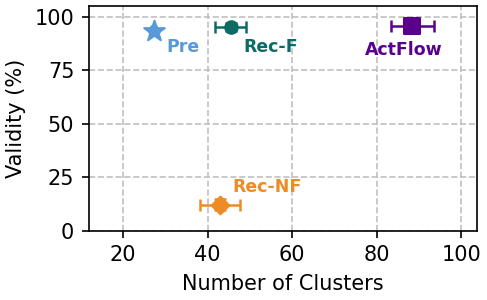}
      \caption{Coverage-Validity}
      \label{fig:qm9_d}
    \end{subfigure} 
    \\[0.4em] \vspace{-1mm}
    \begin{subfigure}{\imgw}
      \centering
      \includegraphics[width=\textwidth]{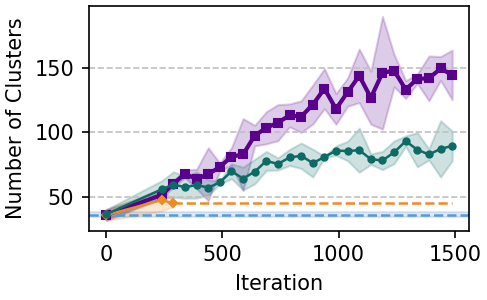}
      \caption{GEOM Coverage}
      \label{fig:drugs_a}
    \end{subfigure}%
    \begin{subfigure}{\imgw}
      \centering
      \includegraphics[width=\textwidth]{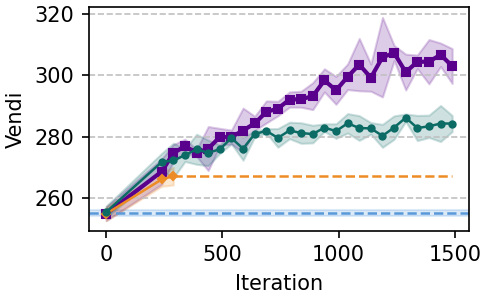}
      \caption{GEOM Diversity}
      \label{fig:drugs_b}
    \end{subfigure}%
    \begin{subfigure}{\imgw}
      \centering
      \includegraphics[width=\textwidth]{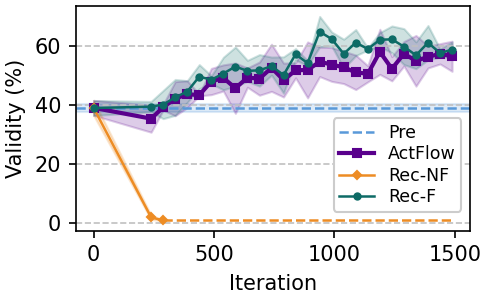}
      \caption{GEOM Validity}
      \label{fig:drugs_c}
    \end{subfigure}%
    \begin{subfigure}{\imgw}
      \centering
      \includegraphics[width=\textwidth]{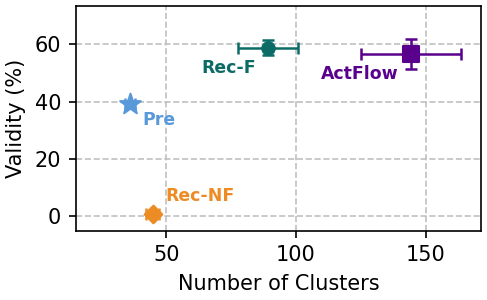}
      \caption{Coverage-Validity}
      \label{fig:drugs_d}
    \end{subfigure} \vspace{-2.5mm}
    \caption{\looseness-1
    (\ref{fig:qm9_a}-\ref{fig:qm9_d}) Molecular design on QM9~\citep{ramakrishnan2014quantum}. \AlgNameShort expands valid coverage substantially more than \AlgRecNF and \AlgRecF, reaching $88.40$ valid clusters versus $42.80$ and $45.40$ respectively (Fig. \ref{fig:qm9_a}). \AlgNameShort also achieves the highest diversity (\ref{fig:qm9_b}) while preserving high validity (\ref{fig:qm9_d}). (\ref{fig:drugs_a}-\ref{fig:drugs_d}) Design of drug-like molecules on GEOM-Drugs~\citep{axelrod2022geom}. \AlgNameShort increases the pre-trained model coverage, diversity, and validity, vastly outperforming both \AlgRecF and \AlgRecNF baselines. 
    }
    \label{fig:experiments_block_2} \vspace{-0mm}
\end{figure*}
\endgroup

\vspace{-0mm}
\paragraph{Therapeutic Peptide Design} 

\begin{wrapfigure}{r}{0.27\textwidth}
  \centering \vspace{-0mm}
\includegraphics[width=0.27\textwidth]{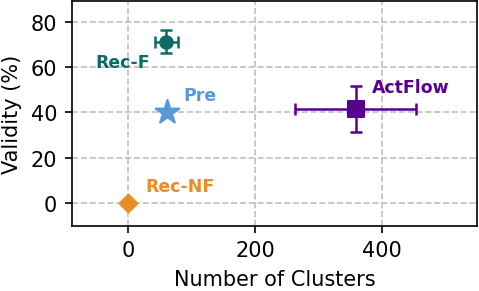}
  \caption{Peptides coverage-validity tradeoff.
}
  \label{fig:peptides_pareto} \vspace{-1mm}
\end{wrapfigure}

\looseness -1 We assess \AlgNameShort on biological sequence generation via discrete diffusion models, as detailed in Apx. \ref{sec:discrete_diffusion}. We consider the task of therapeutic peptide design \citep{wang2022therapeutic} and employ the pre-trained SMILES discrete diffusion model from PepTune \citep{tang2025peptune}. As shown in Table \ref{tab:peptides_proteins_table} and Fig \ref{fig:peptides_pareto}, \AlgNameShort significantly increases the number of clusters (\ie coverage) computed from PeptideCLM embeddings \citep{Feller2024.08.09.607221} from $61$ to $~358$, while maintaining validity, determined via the \texttt{SMILES2PEPTIDE} verifier~\citep{tang2025peptune}. Similarly, \AlgNameShort achieves the highest diversity of $58.87$ Vendi score over the same embeddings. On the contrary, \AlgRecF maintains the similar coverage and diversity as the pre-trained model, while increasing validity, and \AlgRecNF collapses entirely to $0.0\%$ validity, as a single misplaced token can result in an invalid peptide -- thus inducing no valid clusters. These results show that \AlgNameShort's strong empirical performance extends beyond continuous flows to discrete diffusion models (see Apx. \ref{sec:discrete_diffusion}), significantly increasing both model coverage and diversity, outperforming all baselines, while maintaining stable validity.

\vspace{-0mm}
\paragraph{Protein Sequence Design.}
\begin{wrapfigure}{r}{0.27\textwidth}
  \centering \vspace{-3.5mm}
\includegraphics[width=0.27\textwidth]{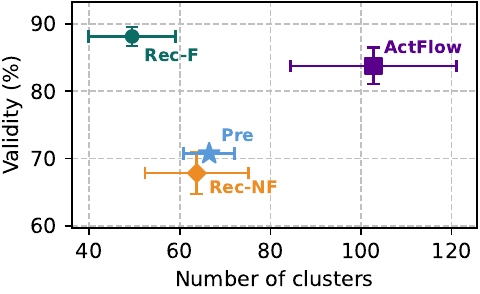}
  \caption{\looseness -1 Proteins coverage-validity tradeoff.
}
  \label{fig:proteins_pareto} \vspace{-0mm}
\end{wrapfigure}
\looseness -1  We finally evaluate \AlgNameShort on protein sequence design, via a continuous ESM diffusion model from SGPO~\citep{yang2025steeringgenerativemodelsexperimental}, pre-trained on the CreiLOV fluorescence dataset~\citep{chen2023deep}. We use $512$ iterations with $1000$ fine-tuning steps each.
\AlgNameShort employs flow representation timestep at $t=0.8$. \AlgNameShort substantially expands valid coverage, computed over token-level ESM embeddings, increasing the number of clusters from $66.50$ to $102.75$, compared to $63.75$ for \AlgRecNF and $49.50$ for \AlgRecF, which lead to significant decrease in both coverage and validity metrics, see Table \ref{tab:peptides_proteins_table} and Fig. \ref{fig:proteins_pareto}. \AlgNameShort also achieves the highest diversity, with Vendi $42.14$ vs $12.87$ for the pre-trained model and $11.85$/$11.67$ for the baselines, while improving validity from $70.81\%$ to $83.74\%$. In contrast, \AlgRecF reaches higher validity ($88.12\%$) but collapses in coverage and diversity; \AlgRecNF reduces coverage and diversity relative to the pre-trained model. \AlgNameShort achieves the strongest coverage--diversity--validity profile (see Fig. \ref{fig:proteins_pareto}), significantly expanding the pre-trained    
protein sequence model to new valid regions of the design space.

\begin{figure*}[t!]
    \centering
    \includegraphics[
        width=\textwidth,
        keepaspectratio,
        trim=0 00mm 0 00mm,
        clip
    ]{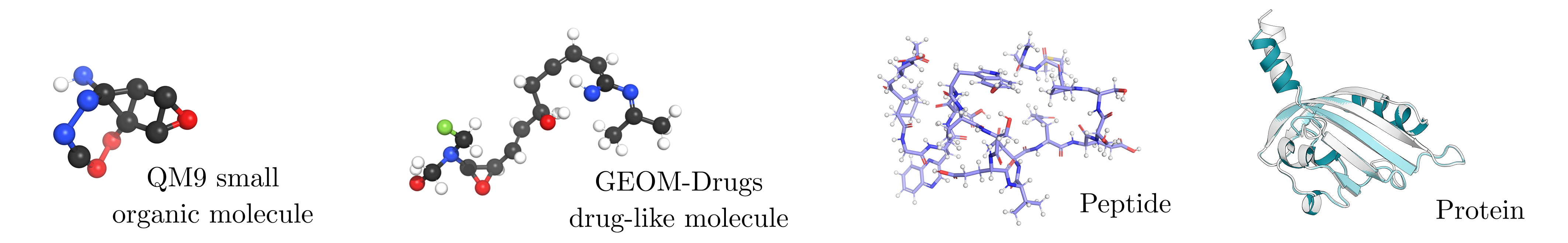}
    \vspace{-5.5mm}
    \caption{\looseness -1 3D structures directly generated or string-converted by models expanded via \AlgNameShort.}
    \label{fig:3d_structures_fig} \vspace{-1.5mm}
\end{figure*}
\definecolor{darkred}{RGB}{200,0,0}
\begin{table*}[t!]
\centering
\vspace{-1mm}
\small
\setlength{\tabcolsep}{4pt}
\renewcommand{\arraystretch}{1.12}
\resizebox{\linewidth}{!}{
\begin{tabular}{
l
S[table-format=3.2\uncert{00.00}]
S[table-format=1.2\uncert{0.00}]
S[table-format=1.2\uncert{0.00}]
S[table-format=2.2\uncert{0.00}]
S[table-format=3.2\uncert{00.00}]
S[table-format=2.2\uncert{00.00}]
S[table-format=1.2\uncert{0.00}]
S[table-format=2.2\uncert{0.00}]
}
\toprule
& \multicolumn{4}{c}{\textbf{Therapeutic Peptide Design}} 
& \multicolumn{4}{c}{\textbf{Protein Sequence Design}} \\
\cmidrule(lr){2-5} \cmidrule(lr){6-9}
\textbf{Method} 
& {\textbf{Coverage $\uparrow$}} 
& {\textbf{Diversity $\uparrow$}}
& {\textbf{FID}}
& {\textbf{Validity} (\%) $\uparrow$}
& {\textbf{Coverage $\uparrow$}}
& {\textbf{Diversity $\uparrow$}}
& {\textbf{FID}}
& {\textbf{Validity} (\%) $\uparrow$} \\
\midrule
Pre 
& 44.33\uncert{3.46} & 13.45\uncert{0.14} & 0.00\uncert{0.00} & 39.31\uncert{1.02} 
& 66.50\uncert{5.63} & 12.87\uncert{0.63} & 0.05\uncert{0.01} & 70.81\uncert{1.12} \\
\AlgRecNF
& {\color{darkred}} 0.00\uncert{0.00} & {\color{darkred}} 0.00\uncert{0.00} & 0.00\uncert{0.00} & {\color{darkred}} 0.00\uncert{0.00}
& {\color{darkred}} 63.75\uncert{11.45} & {\color{darkred}} 11.85\uncert{0.34} & 0.27\uncert{0.06} & {\color{darkred}} 67.81\uncert{3.11} \\
\AlgRecF 
& 59.67\uncert{17.97} & 13.62\uncert{4.34} & 3.18\uncert{1.77} & 71.16\uncert{5.22} 
& {\color{darkred}} 49.50\uncert{9.60} & {\color{darkred}} 11.67\uncert{0.40} & 0.25\uncert{0.04} & 88.12\uncert{1.42} \\
\rowcolor{softbluebg}
\AlgNameShort 
& 358.33\uncert{95.45} & 58.87\uncert{25.98} & 59.15\uncert{44.31} & 41.59\uncert{10.06} 
& 102.75\uncert{18.36} & 42.14\uncert{10.85} & 5.45\uncert{3.34} & 83.74\uncert{2.70} \\
\bottomrule
\end{tabular}
}
\caption{\looseness -1 \AlgNameShort significantly expands valid model coverage (\ie number of valid clusters) and diversity (\ie Vendi~\citep{friedman2022vendi}), while also increasing validity across therapeutic peptide and protein sequence design tasks. We report in \textcolor{darkred}{red} models where coverage or diversity decreased from initial model. These results show that \AlgNameShort vastly outperforms widely adopted synthetic pre-training baselines.}
\label{tab:peptides_proteins_table}
\vspace{-0mm}
\end{table*}

\vspace{-0mm}
\section{Related Work}\vspace{-0mm}
\paragraph{Diffusion and flow model reward adaptation for generative optimization.}
\looseness -1 Several works adapt pre-trained diffusion and flow models for reward maximization, either via fine-tuning~\citep[\eg][]{uehara2024understanding, santi2025flow, liu2025flow} or inference-time sampling~\citep[\eg][]{uehara2025inference, jensen2026value, uehara2025inference}. Without injecting further validity information, these methods remain constrained by the pre-trained model coverage, leading to \emph{over-optimization}, i.e., invalid samples, when steering the model excessively beyond the training distribution~\citep[][]{gutjahr2025constrained}. In this work, we formalize this limitation through the notion of \emph{generable set}, and introduce \AlgNameShort, a task-agnostic continued pre-training method that expands coverage to new valid regions of the design space.

\vspace{-0mm}
\paragraph{Synthetic data generation for diffusion model self-adaptation}
\looseness -1 Generative priors can provide useful synthetic data for downstream learning~\citep[\eg][]{azizi2023synthetic,nguyen2025we}, but recursively training on generated data can induce model collapse~\citep[][]{shumailov2024ai,alemohammad2023self}. For generative model adaptation, prior work uses synthetic data for negative guidance~\citep{alemohammad2024self} or conservative verifier-free self-play~\citep[][]{yuan2024self}, with the primary goal of improving in-distribution sample quality~\citep[\eg][]{alemohammad2024self}. In contrast, we study flow adaptation for out-of-distribution modeling: rather than further refining high-density modes, the goal is to reallocate mass toward newly verified valid regions beyond the training distribution. To our knowledge, \AlgNameShort is the first theory-backed synthetic continued-pretraining method for out-of-distribution generative modeling.

\vspace{-0mm}
\paragraph{Diffusion and flow based design space exploration}
\looseness -1 Recent works introduce scalable methods for diffusion- and flow-based design space exploration via entropy maximization~\citep[\eg][]{de2025provable, santi2025flow} or approximations~\citep[\eg][]{celik2025dime}. A more recent line assumes access to a known differentiable verifier and uses it to rebalance model density toward valid regions via verifier-constrained entropy maximization~\citep[][]{de2026verifier}. In contrast, we assume access only to black-box, unknown verifiers, covering highly-relevant settings in science, where feedback is experimental and/or non-differentiable. Moreover, \AlgNameShort expands valid coverage through directed synthetic data generation rather than verifier-constrained optimization.

\vspace{-0mm}
\paragraph{Safe (active) exploration theory}
\looseness -1 Safe exploration studies learning under unknown safety constraints. Early work uses Gaussian-process confidence bounds to restrict exploration to points certified above a safety threshold~\citep{sui2015safe, sui2018stagewise}. Subsequent safe-RL methods extend this principle to MDPs and control~\citep[\eg][]{turchetta2016safe,berkenkamp2017safe,koller2018learning}. We bridge \emph{safety} in safe exploration with \emph{validity} in generative modeling: expanding a generable set amounts to discovering new valid (i.e., safe) regions. In scientific discovery, invalid samples lead to rejected verifier queries, not unsafe actions in safety-critical systems; thus, \AlgNameShort does not perform safe-set constrained planning. Rather, we employ this viewpoint to provide a first-of-its-kind reachability-based theory of out-of-distribution flow modeling.

\vspace{-0mm}
\section{Conclusion}\vspace{-0mm}
\looseness -1 We depart from data distribution matching and formulate \emph{out-of-distribution flow modeling} via \emph{generable set expansion}, \ie  expanding the valid region of design space a model samples with non-negligible probability. We introduced \AlgNameShort, a continued-pretraining scheme that uses verifier feedback on self-generated data to actively expand in the learned flow representation -- leading to first-of-their-kind reachability-based guarantees for out-of-distribution flow modeling. Across small organic molecules, drug-like molecules, therapeutic peptides, and protein sequences, \AlgNameShort consistently improves coverage, diversity, and validity over widely adopted recursive self-generation baselines. As for limitations, while our framework allows for task-agnostic expansions toward valid, previously inaccessible regions where new-to-nature discoveries may reside -- future work will need to assess whether this form of exploration yields concrete gains (\ie discoveries) in specific real-world applications.

\section*{Acknowledgments} 
This publication was made possible by the ETH AI Center doctoral fellowship to Riccardo De Santi.
The project has received funding from the Swiss
National Science Foundation under NCCR Catalysis (grant number 180544 and 225147) and NCCR Automation (grant agreement 51NF40 180545), National Centres of Competence in Research funded by the Swiss National Science Foundation.
This work was supported by an ETH Zurich Research Grant.
This research was also supported by a grant from the High-throughput Institute for Discovery (HIT-ID) at the University of Pennsylvania to the laboratory of Pranam Chatterjee. 

\bibliography{biblio}

@article{heusel2017gans,
  title={Gans trained by a two time-scale update rule converge to a local nash equilibrium},
  author={Heusel, Martin and Ramsauer, Hubert and Unterthiner, Thomas and Nessler, Bernhard and Hochreiter, Sepp},
  journal={Advances in neural information processing systems},
  volume={30},
  year={2017}
}

@article{song2019generative,
  title={Generative modeling by estimating gradients of the data distribution},
  author={Song, Yang and Ermon, Stefano},
  journal={Advances in neural information processing systems},
  volume={32},
  year={2019}
}

@article{ho2020denoising,
  title={Denoising diffusion probabilistic models},
  author={Ho, Jonathan and Jain, Ajay and Abbeel, Pieter},
  journal={Advances in neural information processing systems},
  volume={33},
  pages={6840--6851},
  year={2020}
}

@inproceedings{corso2022diffdock,
    title={DiffDock: Diffusion Steps, Twists, and Turns for Molecular Docking}, 
    author = {Corso, Gabriele and Stärk, Hannes and Jing, Bowen and Barzilay, Regina and Jaakkola, Tommi},
    booktitle={International Conference on Learning Representations (ICLR)},
    year={2023}
}

@article{chi2025diffusion,
  title={Diffusion policy: Visuomotor policy learning via action diffusion},
  author={Chi, Cheng and Xu, Zhenjia and Feng, Siyuan and Cousineau, Eric and Du, Yilun and Burchfiel, Benjamin and Tedrake, Russ and Song, Shuran},
  journal={The International Journal of Robotics Research},
  volume={44},
  number={10-11},
  pages={1684--1704},
  year={2025},
  publisher={Sage Publications Sage UK: London, England}
}

@inproceedings{hoogeboom2022equivariant,
  title={Equivariant diffusion for molecule generation in 3d},
  author={Hoogeboom, Emiel and Satorras, V{\i}ctor Garcia and Vignac, Cl{\'e}ment and Welling, Max},
  booktitle={International conference on machine learning},
  pages={8867--8887},
  year={2022},
  organization={PMLR}
}

@article{song2020score,
  title={Score-based generative modeling through stochastic differential equations},
  author={Song, Yang and Sohl-Dickstein, Jascha and Kingma, Diederik P and Kumar, Abhishek and Ermon, Stefano and Poole, Ben},
  journal={International Conference on Learning Representations (ICLR)},
  year={2021}
}

@inproceedings{sohl2015deep,
  title={Deep unsupervised learning using nonequilibrium thermodynamics},
  author={Sohl-Dickstein, Jascha and Weiss, Eric and Maheswaranathan, Niru and Ganguli, Surya},
  booktitle={International conference on machine learning},
  pages={2256--2265},
  year={2015},
  organization={PMLR}
}

@article{lipman2024flow,
  title={Flow matching guide and code},
  author={Lipman, Yaron and Havasi, Marton and Holderrieth, Peter and Shaul, Neta and Le, Matt and Karrer, Brian and Chen, Ricky TQ and Lopez-Paz, David and Ben-Hamu, Heli and Gat, Itai},
  journal={arXiv preprint arXiv:2412.06264},
  year={2024}
}

@article{lipman2022flow,
  title={Flow matching for generative modeling},
  author={Lipman, Yaron and Chen, Ricky TQ and Ben-Hamu, Heli and Nickel, Maximilian and Le, Matt},
  journal={International Conference on Learning Representations (ICLR)},
  year={2023}
}

@inproceedings{de2025provable,
  title={Provable Maximum Entropy Manifold Exploration via Diffusion Models},
  author={De Santi, Riccardo and Vlastelica, Marin and Hsieh, Ya-Ping and Shen, Zebang and He, Niao and Krause, Andreas},
  booktitle={International Conference on Machine Learning},
  year={2025}

}

@article{dunn2024mixed,
  title={Mixed continuous and categorical flow matching for 3d de novo molecule generation},
  author={Dunn, Ian and Koes, David Ryan},
  journal={ArXiv},
  pages={arXiv--2404},
  year={2024}
}

@article{ramakrishnan2014quantum,
  title={Quantum chemistry structures and properties of 134 kilo molecules},
  author={Ramakrishnan, Raghunathan and Dral, Pavlo O and Rupp, Matthias and Von Lilienfeld, O Anatole},
  journal={Scientific data},
  volume={1},
  number={1},
  pages={1--7},
  year={2014},
  publisher={Nature Publishing Group}
}

@article{friedman2022vendi,
  title={The Vendi Score: A Diversity Evaluation Metric for Machine Learning},
  author={Dan Friedman and Adji Bousso Dieng},
  journal={Trans. Mach. Learn. Res.},
  year={2022},
  volume={2023},
  url={https://api.semanticscholar.org/CorpusID:252715476}
}

@article{uehara2024understanding,
  title={Understanding reinforcement learning-based fine-tuning of diffusion models: A tutorial and review},
  author={Uehara, Masatoshi and Zhao, Yulai and Biancalani, Tommaso and Levine, Sergey},
  journal={arXiv preprint arXiv:2407.13734},
  year={2024}
}

@inproceedings{santi2025flow,
 	author = {De Santi, Riccardo and Vlastelica, Marin and Hsieh, Ya-Ping and Shen, Zebang and He, Niao and Krause, Andreas},
 	booktitle = {Advances in Neural Information Processing Systems (NeurIPS)},
 	title = {Flow Density Control: Generative Optimization Beyond Entropy-Regularized Fine-Tuning},
 	year = {2025}}

@inproceedings{celik2025dime,
    title={{DIME}: Diffusion-Based Maximum Entropy Reinforcement Learning},
    author={Onur Celik and Zechu Li and Denis Blessing and Ge Li and Daniel Palenicek and Jan Peters and Georgia Chalvatzaki and Gerhard Neumann},
    booktitle={Forty-second International Conference on Machine Learning},
    year={2025},
    url={https://openreview.net/forum?id=Aw6dBR7Vxj}
}

@article{jumper2021highly,
  title={Highly accurate protein structure prediction with AlphaFold},
  author={Jumper, John and Evans, Richard and Pritzel, Alexander and Green, Tim and Figurnov, Michael and Ronneberger, Olaf and Tunyasuvunakool, Kathryn and Bates, Russ and {\v{Z}}{\'\i}dek, Augustin and Potapenko, Anna and others},
  journal={nature},
  volume={596},
  number={7873},
  pages={583--589},
  year={2021},
  publisher={Nature Publishing Group UK London}
}

@article{Weininger1988,
  author = {Weininger, David},
  journal = {J. Chem. Inf. Comput. Sci.},
  number = 1,
  pages = {31-36},
  publisher = {American Chemical Society},
  title = {SMILES, a chemical language and information system. 1. introduction to methodology and encoding rules},
  volume = 28,
  year = 1988
}

@inproceedings{gutjahr2025constrained,
	title={Constrained Molecular Generation via Sequential Flow Model Fine-Tuning},
	author={Gutjahr, Sven and De Santi, Riccardo and Schaufelberger, Luca and Jorner, Kjell and Krause, Andreas},
	booktitle={International Conference on Machine Learning (ICML)},
	year={2026},
}

@article{axelrod2022geom,
  title={GEOM, energy-annotated molecular conformations for property prediction and molecular generation},
  author={Axelrod, Simon and Gomez-Bombarelli, Rafael},
  journal={Scientific Data},
  volume={9},
  number={1},
  pages={185},
  year={2022},
  publisher={Nature Publishing Group UK London}
}

@inproceedings{yang2025steeringgenerativemodelsexperimental,
      title={Steering Generative Models with Experimental Data for Protein Fitness Optimization}, 
      author={Jason Yang and Wenda Chu and Daniel Khalil and Raul Astudillo and Bruce J. Wittmann and Frances H. Arnold and Yisong Yue},
      booktitle={Advances in Neural Information Processing Systems (NeurIPS)},
      year={2025},
      url={https://arxiv.org/abs/2505.15093}, 
}

@inproceedings{scholkopf2001generalized,
  title={A generalized representer theorem},
  author={Sch{\"o}lkopf, Bernhard and Herbrich, Ralf and Smola, Alex J},
  booktitle={International conference on computational learning theory},
  pages={416--426},
  year={2001},
  organization={Springer}
}

@article{pasztor2024bandits,
  title={Bandits with preference feedback: A stackelberg game perspective},
  author={P{\'a}sztor, Barna and Kassraie, Parnian and Krause, Andreas},
  journal={Advances in Neural Information Processing Systems},
  volume={37},
  pages={11997--12034},
  year={2024}
}

@inproceedings{sui2015safe,
  title={Safe exploration for optimization with Gaussian processes},
  author={Sui, Yanan and Gotovos, Alkis and Burdick, Joel and Krause, Andreas},
  booktitle={International conference on machine learning},
  pages={997--1005},
  year={2015},
  organization={PMLR}
}

@inproceedings{sui2018stagewise,
  title={Stagewise safe Bayesian optimization with Gaussian processes},
  author={Sui, Yanan and Zhuang, Vincent and Burdick, Joel and Yue, Yisong},
  booktitle={International conference on machine learning},
  pages={4781--4789},
  year={2018},
  organization={PMLR}
}

@inproceedings{de2026verifier,
	title={Verifier-Constrained Flow Expansion for Discovery Beyond the Data},
	author={De Santi, Riccardo and Protopapas, Kimon and Hsieh, Ya-Ping and Krause, Andreas},
	booktitle={International Conference on Learning Representations (ICLR)},
	year={2026},
	month={April},
}

@article{guo2024takes,
  title={TANGO: direct optimization of constrained synthesizability for generative molecular design},
  author={Guo, Jeff and Schwaller, Philippe},
  journal={Nature Computational Science},
  year={2026}
}

@article{watson2023novo,
  title={De novo design of protein structure and function with RFdiffusion},
  author={Watson, Joseph L and Juergens, David and Bennett, Nathaniel R and Trippe, Brian L and Yim, Jason and Eisenach, Helen E and Ahern, Woody and Borst, Andrew J and Ragotte, Robert J and Milles, Lukas F and others},
  journal={Nature},
  volume={620},
  number={7976},
  pages={1089--1100},
  year={2023},
  publisher={Nature Publishing Group UK London}
}

@inproceedings{alberti2025data,
 author = {Alberti, Silas and Hasanaliyev, Kenan and Shah, Manav and Ermon, Stefano},
 booktitle = {International Conference on Learning Representations},
 editor = {Y. Yue and A. Garg and N. Peng and F. Sha and R. Yu},
 pages = {3084--3100},
 title = {Data Unlearning in Diffusion Models},
 url = {https://proceedings.iclr.cc/paper_files/paper/2025/file/08a362bd4ae1934e099ce025f06039fe-Paper-Conference.pdf},
 volume = {2025},
 year = {2025}
}

@article{liu2025flow,
  title={Flow-grpo: Training flow matching models via online rl},
  author={Liu, Jie and Liu, Gongye and Liang, Jiajun and Li, Yangguang and Liu, Jiaheng and Wang, Xintao and Wan, Pengfei and Zhang, Di and Ouyang, Wanli},
  journal={Advances in Neural Information Processing Systems (NeurIPS)},
  year={2025}
}

@article{uehara2025inference,
  title={Inference-time alignment in diffusion models with reward-guided generation: Tutorial and review},
  author={Uehara, Masatoshi and Zhao, Yulai and Wang, Chenyu and Li, Xiner and Regev, Aviv and Levine, Sergey and Biancalani, Tommaso},
  journal={arXiv preprint arXiv:2501.09685},
  year={2025}
}

@inproceedings{jensen2026value,
title={Value Matching: Scalable and Gradient-Free Reward-Guided Flow Adaptation},
author={Cristian {Perez Jensen} and Luca Schaufelberger and Riccardo De Santi and Kjell Jorner and Andreas Krause},
booktitle={The Fourteenth International Conference on Learning Representations},
year={2026},
url={https://openreview.net/forum?id=7iXt44Actj}
}

@article{lou2023discrete,
  title={Discrete diffusion modeling by estimating the ratios of the data distribution},
  author={Lou, Aaron and Meng, Chenlin and Ermon, Stefano},
  journal={International Conference on Learning Representations (ICLR)},
  year={2024}
}

@article{wang2022therapeutic,
  title={Therapeutic peptides: current applications and future directions},
  author={Wang, Lei and Wang, Nanxi and Zhang, Wenping and Cheng, Xurui and Yan, Zhibin and Shao, Gang and Wang, Xi and Wang, Rui and Fu, Caiyun},
  journal={Signal transduction and targeted therapy},
  volume={7},
  number={1},
  pages={48},
  year={2022},
  publisher={Nature Publishing Group UK London}
}

@article{tang2025peptune,
  title={Peptune: De novo generation of therapeutic peptides with multi-objective-guided discrete diffusion},
  author={Tang, Sophia and Zhang, Yinuo and Chatterjee, Pranam},
  journal={42nd International Conference on Machine Learning},
  year={2025}
}

@article{pearce2023imitating,
  title={Imitating human behaviour with diffusion models},
  author={Pearce, Tim and Rashid, Tabish and Kanervisto, Anssi and Bignell, Dave and Sun, Mingfei and Georgescu, Raluca and Macua, Sergio Valcarcel and Tan, Shan Zheng and Momennejad, Ida and Hofmann, Katja and others},
  journal={International Conference on Learning Representations (ICLR)},
  year={2023}
}

@article{rector2026general,
  title={General Multimodal Protein Design Enables DNA-Encoding of Chemistry},
  author={Rector-Brooks, Jarrid and Lambert, Th{\'e}ophile and Skreta, Marta and Roth, Daniel and Long, Yueming and Li, Zi-Qi and Zhang, Xi and Cretu, Miruna and Li, Francesca-Zhoufan and Ganapathy, Tanvi and others},
  journal={arXiv preprint arXiv:2604.05181},
  year={2026}
}

@article{chen2023deep,
  title={Deep mutational scanning of an oxygen-independent fluorescent protein CreiLOV for comprehensive profiling of mutational and epistatic effects},
  author={Chen, Yongcan and Hu, Ruyun and Li, Keyi and Zhang, Yating and Fu, Lihao and Zhang, Jianzhi and Si, Tong},
  journal={ACS Synthetic Biology},
  volume={12},
  number={5},
  pages={1461--1473},
  year={2023},
  publisher={ACS Publications}
}

@article{shi2024simplified,
  title={Simplified and generalized masked diffusion for discrete data},
  author={Shi, Jiaxin and Han, Kehang and Wang, Zhe and Doucet, Arnaud and Titsias, Michalis},
  journal={Advances in neural information processing systems},
  volume={37},
  pages={103131--103167},
  year={2024}
}

@article{sahoo2024simple,
  title={Simple and effective masked diffusion language models},
  author={Sahoo, Subham and Arriola, Marianne and Schiff, Yair and Gokaslan, Aaron and Marroquin, Edgar and Chiu, Justin and Rush, Alexander and Kuleshov, Volodymyr},
  journal={Advances in Neural Information Processing Systems},
  volume={37},
  pages={130136--130184},
  year={2024}
}

@article{ou2024your,
  title={Your absorbing discrete diffusion secretly models the conditional distributions of clean data},
  author={Ou, Jingyang and Nie, Shen and Xue, Kaiwen and Zhu, Fengqi and Sun, Jiacheng and Li, Zhenguo and Li, Chongxuan},
  journal={International Conference on Learning Representations (ICLR)},
  year={2025}
}

@article{zheng2024masked,
  title={Masked diffusion models are secretly time-agnostic masked models and exploit inaccurate categorical sampling},
  author={Zheng, Kaiwen and Chen, Yongxin and Mao, Hanzi and Liu, Ming-Yu and Zhu, Jun and Zhang, Qinsheng},
  journal={International Conference on Learning Representations (ICLR)},
  year={2025}
}

@article{shumailov2024ai,
  title={AI models collapse when trained on recursively generated data},
  author={Shumailov, Ilia and Shumaylov, Zakhar and Zhao, Yiren and Papernot, Nicolas and Anderson, Ross and Gal, Yarin},
  journal={Nature},
  volume={631},
  number={8022},
  pages={755--759},
  year={2024},
  publisher={Nature Publishing Group UK London}
}

@article{azizi2023synthetic,
  title={Synthetic data from diffusion models improves imagenet classification},
  author={Azizi, Shekoofeh and Kornblith, Simon and Saharia, Chitwan and Norouzi, Mohammad and Fleet, David J},
  journal={arXiv preprint arXiv:2304.08466},
  year={2023}
}

@article{alemohammad2024self,
  title={Self-improving diffusion models with synthetic data},
  author={Alemohammad, Sina and Humayun, Ahmed Imtiaz and Agarwal, Shruti and Collomosse, John and Baraniuk, Richard},
  journal={arXiv preprint arXiv:2408.16333},
  year={2024}
}

@article{yuan2024self,
  title={Self-play fine-tuning of diffusion models for text-to-image generation},
  author={Yuan, Huizhuo and Chen, Zixiang and Ji, Kaixuan and Gu, Quanquan},
  journal={Advances in Neural Information Processing Systems},
  volume={37},
  pages={73366--73398},
  year={2024}
}

@inproceedings{nguyen2025we,
  title={Do We Need All the Synthetic Data? Targeted Image Augmentation via Diffusion Models},
  author={Nguyen, Dang and Li, Jiping and Zheng, Jinghao and Mirzasoleiman, Baharan},
  booktitle={The Fourteenth International Conference on Learning Representations},
  year={2026}
}

@inproceedings{alemohammad2023self,
  title={Self-consuming generative models go mad},
  author={Alemohammad, Sina and Casco-Rodriguez, Josue and Luzi, Lorenzo and Humayun, Ahmed Imtiaz and Babaei, Hossein and LeJeune, Daniel and Siahkoohi, Ali and Baraniuk, Richard},
  booktitle={The Twelfth International Conference on Learning Representations},
  year={2024}
}

@article{turchetta2016safe,
  title={Safe exploration in finite markov decision processes with gaussian processes},
  author={Turchetta, Matteo and Berkenkamp, Felix and Krause, Andreas},
  journal={Advances in neural information processing systems},
  volume={29},
  year={2016}
}

@article{berkenkamp2017safe,
  title={Safe model-based reinforcement learning with stability guarantees},
  author={Berkenkamp, Felix and Turchetta, Matteo and Schoellig, Angela and Krause, Andreas},
  journal={Advances in neural information processing systems},
  volume={30},
  year={2017}
}

@inproceedings{koller2018learning,
  title={Learning-based model predictive control for safe exploration},
  author={Koller, Torsten and Berkenkamp, Felix and Turchetta, Matteo and Krause, Andreas},
  booktitle={2018 IEEE conference on decision and control (CDC)},
  pages={6059--6066},
  year={2018},
  organization={IEEE}
}

@article{laurent2000adaptive,
  title={Adaptive estimation of a quadratic functional by model selection},
  author={Laurent, B{\'e}atrice and Massart, Pascal},
  journal={The Annals of Statistics},
  volume={28},
  number={5},
  pages={1302--1338},
  year={2000},
  publisher={Institute of Mathematical Statistics}
}

@book{boucheron2013concentration,
  title={Concentration Inequalities: A Nonasymptotic Theory of Independence},
  author={Boucheron, St{\'e}phane and Lugosi, G{\'a}bor and Massart, Pascal},
  year={2013},
  publisher={Oxford University Press}
}

@article{dong2023raft,
  title={{RAFT}: Reward rAnked FineTuning for Generative Foundation Model Alignment},
  author={Hanze Dong and Wei Xiong and Deepanshu Goyal and Yihan Zhang and Winnie Chow and Rui Pan and Shizhe Diao and Jipeng Zhang and KaShun SHUM and Tong Zhang},
  journal={Transactions on Machine Learning Research},
  issn={2835-8856},
  year={2023},
  url={https://openreview.net/forum?id=m7p5O7zblY},
}

@article{gulcehre2023reinforced,
  title={Reinforced self-training (rest) for language modeling},
  author={Gulcehre, Caglar and Paine, Tom Le and Srinivasan, Srivatsan and Konyushkova, Ksenia and Weerts, Lotte and Sharma, Abhishek and Siddhant, Aditya and Ahern, Alex and Wang, Miaosen and Gu, Chenjie and others},
  journal={arXiv preprint arXiv:2308.08998},
  year={2023}
}

@article{lin2022language,
  title={Language models of protein sequences at the scale of evolution enable accurate structure prediction},
  author={Lin, Zeming and Akin, Halil and Rao, Roshan and Hie, Brian and Zhu, Zhongkai and Lu, Wenting and Smetanin, Nikita and dos Santos Costa, Allan and Fazel-Zarandi, Maryam and Sercu, Tom and Candido, Sal and others},
  journal={bioRxiv},
  year={2022},
  publisher={Cold Spring Harbor Laboratory}
}

@article{halgren1996merck,
  title={Merck molecular force field. I. Basis, form, scope, parameterization, and performance of MMFF94},
  author={Halgren, Thomas A},
  journal={Journal of computational chemistry},
  volume={17},
  number={5-6},
  pages={490--519},
  year={1996},
  publisher={Wiley Online Library}
}

@article{li2021smiles,
  title={SMILES pair encoding: a data-driven substructure tokenization algorithm for deep learning},
  author={Li, Xinhao and Fourches, Denis},
  journal={Journal of chemical information and modeling},
  volume={61},
  number={4},
  pages={1560--1569},
  year={2021},
  publisher={ACS Publications}
}

@article {Feller2024.08.09.607221,
	author = {Feller, Aaron L. and Wilke, Claus O.},
	title = {Peptide-specific chemical language model successfully predicts membrane diffusion of cyclic peptides},
	elocation-id = {2024.08.09.607221},
	year = {2024},
	doi = {10.1101/2024.08.09.607221},
	publisher = {Cold Spring Harbor Laboratory},
	URL = {https://www.biorxiv.org/content/early/2024/08/09/2024.08.09.607221},
	eprint = {https://www.biorxiv.org/content/early/2024/08/09/2024.08.09.607221.full.pdf},
	journal = {bioRxiv}
}
\bibliographystyle{plain}

\newpage
\appendix
\section{Appendix}
\tableofcontents
\newpage
\addtocontents{toc}{\protect\setcounter{tocdepth}{2}}

\section{Central Limitation of Standard Pre-training with an Imperfect Model}
\label{sec:imperfect_model_discussion_apx}
In the main text, Eq.~\eqref{eq:representational_limit} is stated for clarity under the idealization that the pre-trained model generates only valid designs. In practice, this need not hold: the generable set may contain invalid designs, so that
\[
\Omega^\tau_{\theta^{\mathrm{pre}}} \cap (\X \setminus \Omega^\star) \neq \emptyset .
\]
This does not change the central limitation. One simply applies the coverage statement to the \emph{valid part} of the generable set,
\[
\Omega^{\tau,\mathrm{val}}_{\theta^{\mathrm{pre}}}
:=
\Omega^\tau_{\theta^{\mathrm{pre}}}\cap \Omega^\star
\subseteq \Omega^\star .
\]
The realistic limitation is therefore
\[
\Omega^{\tau,\mathrm{val}}_{\theta^{\mathrm{pre}}} \subset \Omega^\star,
\qquad
\mathrm{Vol}\!\left(\Omega^{\tau,\mathrm{val}}_{\theta^{\mathrm{pre}}}\right)
\ll
\mathrm{Vol}(\Omega^\star),
\]
i.e., the pre-trained model covers only a small fraction of the valid design space with non-negligible probability, even if it also assigns mass to invalid regions.

This is precisely the setting addressed by \AlgNameShort. The algorithm does not require the initial generable set to be valid; it only requires verifier feedback to identify which generated samples are valid and to adapt the model toward newly verified valid regions. Thus, generable set expansion should be understood as expansion of the valid, verifier-certified portion of the model's coverage. Moreover, when the initial model has partial validity, reallocating mass toward newly verified valid regions can also increase the model's overall validity, as we observe in our experiments, where the pre-trained models are imperfect and \AlgNameShort systematically improves not only coverage and diversity, but typically increases substantially validity as well.

\paragraph{Volume over the valid design space.}
Throughout, \(\Vol(\cdot)\) denotes volume with respect to a domain-appropriate reference measure on the valid design space \(\Omega^\star\). Formally, let \(\nu^\star\) be a reference measure supported on \(\Omega^\star\): Lebesgue measure when \(\Omega^\star\) is full-dimensional, intrinsic Hausdorff measure when \(\Omega^\star\) is a lower-dimensional manifold, and counting measure in discrete design spaces. For any \(A\subseteq\X\), we write
\[
\Vol(A) := \nu^\star(A\cap\Omega^\star).
\]
Thus, the central limitation in Eq.~\eqref{eq:representational_limit} measures how much of the valid design space is covered by the model's \(\tau\)-generable set, rather than ambient Lebesgue volume in \(\X\). This avoids degeneracies in settings where valid designs lie on lower-dimensional manifolds or discrete spaces.

\section{A Warm-Up Analysis of Failure Modes of Expansion via Synthetic Data}
\subsection{Compact Analysis}
\label{sec:warm_up_gaussian_analysis}

\paragraph{Warm Up Gaussian Setting.}

\looseness -1 A widely adopted approach to synthetic data generation is recursive sampling with closed-loop verifier filtering. The purpose of this section is \emph{not} to establish a general impossibility result for this paradigm, but to use a minimal analytical model to isolate possible failure modes relevant to generable-set expansion. To this end, we introduce a simple probabilistic notion of the \emph{generable valid frontier}: valid samples that remain reliably generable under the current model, yet lie away from its dominant modes. Such points form a natural abstraction of expansion-enabling samples. The analysis then shows that standard recursive self-generation schemes may be unlikely to produce them even in a simple low-dimensional setting, thereby shedding light on possible failure modes of data generation schemes and motivating algorithmic desiderata for expansion. We report derivations in Apx. \ref{sec:gaussian_analysis_appendix}. 

\paragraph{A Gaussian abstraction of pre-training.}
We consider a design space \(\X=\mathbb{R}^d\).
Let \(U \subset \mathbb{R}^d\) be a \(k\)-dimensional linear subspace, let \(m:=d-k\), and write
\(\mathbb{R}^d = U \oplus U^\perp\).
We interpret \(U\) as a low-dimensional region well captured by pre-training, and \(U^\perp\) as
orthogonal directions along which expansion beyond the dominant pre-trained modes could occur.
For \(x \in \mathbb{R}^d\), write \smash{\(x=x_U+x_\perp\)} with \smash{\(x_U=\Pi_U x\)} and \smash{\(x_\perp=\Pi_{U^\perp}x\)}. We model the pre-trained generator \smash{$p^{\theta_0}$} by the anisotropic Gaussian
\[
X \sim p^{\theta_0} =: p_0 = \mathcal{N} \bigl(0,\, I_U \oplus \sigma^2 I_{U^\perp}\bigr),
\qquad 0<\sigma \ll 1.
\]
Thus, the model is spread along \(U\), while it is sharply concentrated around \(U\) in the
orthogonal directions. For a density threshold \(\tau>0\), the corresponding generable set is
\[
\Omega_{\theta_0}^{\tau}
=
\left\{
x \in \mathbb{R}^d :
\|x_U\|_2^2 + \sigma^{-2}\|x_\perp\|_2^2 \le r_\tau^2
\right\},
\qquad
r_\tau^2
:=
2\log \left(\frac{(2\pi)^{-d/2}\sigma^{-m}}{\tau}\right),
\]
namely an ellipsoid elongated along \(U\) and very thin along \(U^\perp\).

\paragraph{A probabilistic notion of generable valid frontier.}
To reason about samples that may support expansion, we isolate a shell inside
\(\Omega_{\theta_0}^{\tau}\) that is close to its orthogonal boundary. Fix \(R_U>0\) and radii \(0<\rho_-<\rho_+\) such that $R_U < r_\tau$ and $\rho_+ < \sigma \sqrt{r_\tau^2 - R_U^2}$. We denote the \emph{generable frontier} by
\[
\mathcal{F}_{\mathrm{frontier}}
:=
\left\{
x \in \mathbb{R}^d :
\|x_U\|_2 \le R_U,\ 
\rho_- \le \|x_\perp\|_2 \le \rho_+
\right\}.
\]
By construction, \(\mathcal{F}_{\mathrm{frontier}} \subseteq \Omega_{\theta_0}^{\tau}\). Thus, points in \(\mathcal{F}_{\mathrm{frontier}}\) remain reliably generable under the current model, since they lie inside its generable set, while already exhibiting a substantial orthogonal deviation from the model dominant modes. Next, fix a unit vector \(u_\star \in U^\perp\) and an opening angle
\(\phi \in (0,\pi/2)\). We define the cone of \emph{valid directions of expansion}
\[
C_\phi(u_\star)
:=
\left\{
y \in U^\perp \setminus \{0\} :
\left\langle \frac{y}{\|y\|_2}, u_\star \right\rangle \ge \cos\phi
\right\}.
\]
We call \emph{generable valid frontier} the set
\[
\mathcal{V}_{\mathrm{frontier}}
:=
\left\{
x \in \mathcal{F}_{\mathrm{frontier}} :
x_\perp \in C_\phi(u_\star)
\right\}.
\]

\looseness -1 Thus \(\mathcal{V}_{\mathrm{frontier}}\) consists of samples that lie in the model generable frontier, and are aligned with a valid direction of expansion.  While \(\mathcal{V}_{\mathrm{frontier}}\) does not provide a general first-principles characterization of all useful samples, it formally captures in a minimal way the qualitative tension we want to study: useful self-generated samples should plausibly be ($i$) \emph{generable} by the current model, yet ($ii$) \emph{not} lying within its dominant modes, and ($iii$) aligned with valid directions of expansion. The next result shows that standard self-generation scheme might sample data within this natural class of frontier-valid samples with extremely low probability even in the introduced illustrative Gaussian generator setting.
\begin{restatable}[Standard self-generation is unlikely to find frontier-valid samples]{proposition}{passiveNegativeMain}
\label{thm:passive_negative_main}
Let \(X \sim p_0\), and let \(\mathcal{V}_{\mathrm{frontier}}\) be defined above. For any \(\eta \in (0,1)\), define
$
\rho_\eta
:=
\sigma
\sqrt{
m + 2\sqrt{m\log(1/\eta)} + 2\log(1/\eta)
}$. Assume this choice is feasible for the frontier shell, i.e. \(\rho_\eta<\rho_+\). If \(\rho_-=\rho_\eta\), then
\[
\Pr \bigl(X \in \mathcal{V}_{\mathrm{frontier}}\bigr)
\le
\eta \exp \left(-\frac{m-1}{2}\cos^2\phi\right).
\]
Consequently, for \(N\) i.i.d.\ passive samples \(X_1,\dots,X_N \sim p_0\),
\[
\Pr \left(\exists i \in [N] : X_i \in \mathcal{V}_{\mathrm{frontier}}\right)
\le
N \eta \exp \left(-\frac{m-1}{2}\cos^2\phi\right).
\]
\end{restatable}

\paragraph{What the proposition reveals.}
The proposition isolates two sources of difficulty in recursive self-generation. The \(\eta\)-term reflects the rarity of reaching the generable frontier at all, namely of producing a sample with substantial orthogonal deviation from the dominant modes. Beyond this, an additional exponential penalty governs the chance of landing in the correct valid cone. Hence, when valid expansion directions are sparse, the sample complexity of passively obtaining a frontier-valid point becomes exponential in the orthogonal dimension \(m\).

The algorithmic implication is clear: closed-loop verifier filtering alone is not enough for reliable expansion. Successful expansion requires actively steering generation toward the generable valid frontier.

\subsection{Extensive Analysis with Proofs}
\paragraph{Gaussian warm-up: self-generation requires rare valid frontier events}
\label{app:gaussian_negative_result}

We present a simple Gaussian model that isolates the two difficulties behind standard self-generation for efficient out-of-distribution discovery: sampling a \emph{large deviation} away
from the low-dimensional region captured by pre-training, and doing so along a \emph{valid direction}
of expansion.

We work directly on the design space \(\X=\mathbb{R}^d\), so in this warm-up there is no separate
representation map and \(\X=\Z\).
Let \(U \subset \mathbb{R}^d\) be a \(k\)-dimensional linear subspace, and write
\[
m := d-k, \qquad \mathbb{R}^d = U \oplus U^\perp .
\]
For every \(x \in \mathbb{R}^d\), denote by
\[
x_U := \Pi_U x, \qquad x_\perp := \Pi_{U^\perp} x,
\]
so that \(x=x_U+x_\perp\).
Throughout this subsection we assume \(m \ge 2\), since the directional effect of interest only
appears when the orthogonal complement has dimension at least two.

\paragraph{Pre-trained model.}
We model a pre-trained generator $p^{\theta_0}$ by the anisotropic Gaussian
\[
X \sim p^{\theta_0} =: p_0 = \mathcal{N}\!\bigl(0,\, I_U \oplus \sigma^2 I_{U^\perp}\bigr),
\qquad 0<\sigma \ll 1.
\]
Equivalently,
\[
X_U \sim \mathcal{N}(0,I_k), \qquad X_\perp \sim \mathcal{N}(0,\sigma^2 I_m),
\]
independently.
Its density is
\[
p_0(x)
=
(2\pi)^{-d/2}\sigma^{-m}
\exp\!\left(
-\frac12\Bigl(\|x_U\|_2^2 + \sigma^{-2}\|x_\perp\|_2^2\Bigr)
\right).
\]

\paragraph{Generable set.}
Fix a level \(\epsilon \in (0,(2\pi)^{-d/2}\sigma^{-m})\).
In the sense of Definition \ref{definition:generable_set}, the \(\epsilon\)-level generable set of the pre-trained model is
\[
\Omega^{\epsilon}_{\theta_0}
:=
\{x \in \mathbb{R}^d : p_0(x)\ge \epsilon\}.
\]
For the Gaussian above this is the ellipsoid
\[
\Omega^{\epsilon}_{\theta_0}
=
\left\{
x \in \mathbb{R}^d :
\|x_U\|_2^2 + \sigma^{-2}\|x_\perp\|_2^2 \le r_\epsilon^2
\right\},
\]
where
\[
r_\epsilon^2
:=
2\log\!\left(\frac{(2\pi)^{-d/2}\sigma^{-m}}{\epsilon}\right).
\]

\paragraph{Valid direction and generable valid frontier.}
Fix a unit vector \(u_\star \in U^\perp\) and an opening angle \(\phi \in (0,\pi/2)\).
We define the one-sided cone of valid orthogonal directions by
\[
C_\phi(u_\star)
:=
\left\{
y \in U^\perp \setminus \{0\} :
\left\langle \frac{y}{\|y\|_2}, u_\star \right\rangle \ge \cos\phi
\right\}.
\]
We also fix a radius \(R_U>0\) and shell radii \(0<\rho_-<\rho_+\) satisfying
\[
R_U < r_\epsilon,
\qquad
\rho_+ < \sigma \sqrt{r_\epsilon^2 - R_U^2}.
\]
The corresponding \emph{generable valid frontier} is
\[
\mathcal{V}_{\mathrm{frontier}}
:=
\left\{
x \in \mathbb{R}^d :
\|x_U\|_2 \le R_U,\ 
\rho_- \le \|x_\perp\|_2 \le \rho_+,\ 
x_\perp \in C_\phi(u_\star)
\right\}.
\]
These are rare points that are still inside the current generable set, but already lie on a specific
valid orthogonal frontier along which future continued pre-training may expand.

\begin{lemma}[The generable valid frontier lies inside the current generable set]
\label{lem:gaussian_frontier_inside_generable}
Under the above choice of \(R_U,\rho_+,\epsilon\), one has
\[
\mathcal{V}_{\mathrm{frontier}} \subseteq \Omega^{\epsilon}_{\theta_0}.
\]
\end{lemma}
\begin{proof}
Take any \(x \in \mathcal{V}_{\mathrm{frontier}}\).
Then \(\|x_U\|_2 \le R_U\) and \(\|x_\perp\|_2 \le \rho_+\), hence
\[
\|x_U\|_2^2 + \sigma^{-2}\|x_\perp\|_2^2
\le
R_U^2 + \sigma^{-2}\rho_+^2
<
r_\epsilon^2.
\]
By the explicit description of \(\Omega^\epsilon_{\theta_0}\), this implies \(x \in \Omega^\epsilon_{\theta_0}\).
\end{proof}

\paragraph{A larger valid design space.}
We consider a valid design space that extends beyond
the current generable set in the same direction \(u_\star\).
Fix any \(R>\sigma r_\epsilon\), and define
\[
\mathcal{V}_{\mathrm{out}}
:=
\left\{
x \in \mathbb{R}^d :
\|x_U\|_2 \le R_U,\ 
\sigma\sqrt{r_\epsilon^2-\|x_U\|_2^2} < \|x_\perp\|_2 \le R,\ 
x_\perp \in C_\phi(u_\star)
\right\}.
\]

We then set
\[
S_0
:=
\left\{
x \in \mathbb{R}^d :
\|x_U\|_2 \le R_U,\ 
\|x_\perp\|_2 \le \rho_-
\right\},
\qquad
\Omega^\star := S_0 \cup \mathcal{V}_{\mathrm{frontier}} \cup \mathcal{V}_{\mathrm{out}}.
\]

Thus \(\Omega^\star\) contains an already-valid core \(S_0\), a rare but still generable frontier
\(\mathcal{V}_{\mathrm{frontier}}\), and a genuinely out-of-distribution valid region
\(\mathcal{V}_{\mathrm{out}} \subset (\Omega^\epsilon_{\theta_0})^c\).
Indeed, for any \(x\in \mathcal{V}_{\mathrm{out}}\),
\[
\|x_U\|_2^2+\sigma^{-2}\|x_\perp\|_2^2
>
\|x_U\|_2^2+\sigma^{-2}\sigma^2
\bigl(r_\epsilon^2-\|x_U\|_2^2\bigr)
=
r_\epsilon^2,
\]
so \(x\notin \Omega^\epsilon_{\theta_0}\).

The role of the negative result below is to show that, even before trying to reach
\(\mathcal{V}_{\mathrm{out}}\), passive self-generation is already unlikely to find
the expansion-enabling frontier samples \(\mathcal{V}_{\mathrm{frontier}}\).

\paragraph{A convenient large-deviation scale.}
Since \(X_\perp \sim \mathcal{N}(0,\sigma^2 I_m)\), one has
\[
\frac{\|X_\perp\|_2^2}{\sigma^2} \sim \chi_m^2.
\]
Therefore, for any \(\eta \in (0,1)\), the choice
\[
\rho_\eta
:=
\sigma
\sqrt{
m + 2\sqrt{m\log(1/\eta)} + 2\log(1/\eta)
}
\]
satisfies
\[
\Pr \bigl(\|X_\perp\|_2 \ge \rho_\eta\bigr) \le \eta
\]
by the Laurent-Massart inequality~\citep{laurent2000adaptive, boucheron2013concentration}.

\passiveNegativeMain*

\begin{proof}
By definition of \(\mathcal{V}_{\mathrm{frontier}}\),
\[
\{X \in \mathcal{V}_{\mathrm{frontier}}\}
\subseteq
\left\{
\|X_\perp\|_2 \ge \rho_-
\right\}
\cap
\left\{
\left\langle \frac{X_\perp}{\|X_\perp\|_2}, u_\star \right\rangle \ge \cos\phi
\right\}.
\]
Therefore
\[
\Pr \bigl(X \in \mathcal{V}_{\mathrm{frontier}}\bigr)
\le
\Pr\!\left(
\|X_\perp\|_2 \ge \rho_-,
\left\langle \frac{X_\perp}{\|X_\perp\|_2}, u_\star \right\rangle \ge \cos\phi
\right).
\]

Now write
\[
X_\perp = \sigma G,
\qquad
G \sim \mathcal{N}(0,I_m).
\]
Let
\[
R := \|G\|_2,
\qquad
S := \frac{G}{\|G\|_2} \in \mathbb{S}^{m-1}.
\]
For an isotropic Gaussian, \(R\) and \(S\) are independent, and \(S\) is uniform on the unit sphere
\(\mathbb{S}^{m-1}\).
Since \(X_\perp=\sigma G\), the previous probability equals
\[
\Pr\!\left(
R \ge \frac{\rho_-}{\sigma},
\langle S,u_\star\rangle \ge \cos\phi
\right)
=
\Pr\!\left(R \ge \frac{\rho_-}{\sigma}\right) \cdot
\Pr \left(\langle S,u_\star\rangle \ge \cos\phi\right),
\]
which proves
\[
\Pr \bigl(X \in \mathcal{V}_{\mathrm{frontier}}\bigr)
\le
\Pr \bigl(\|X_\perp\|_2 \ge \rho_-\bigr)\, \cdot
\Pr \left(\left\langle \frac{X_\perp}{\|X_\perp\|_2}, u_\star \right\rangle \ge \cos\phi\right).
\]

For the directional term, a standard spherical-cap bound for \(S \sim \mathrm{Unif}(\mathbb{S}^{m-1})\)
gives, for every \(t \in [0,1]\),
\[
\Pr \bigl(\langle S,u_\star\rangle \ge t\bigr)
\le
\exp\!\left(-\frac{m-1}{2}t^2\right).
\]
Applying this with \(t=\cos\phi\) yields
\[
\Pr\!\left(
\left\langle \frac{X_\perp}{\|X_\perp\|_2}, u_\star \right\rangle \ge \cos\phi
\right)
=
\Pr \bigl(\langle S,u_\star\rangle \ge \cos\phi\bigr)
\le
\exp\!\left(-\frac{m-1}{2}\cos^2\phi\right).
\]
Combining the two bounds proves
\[
\Pr \bigl(X \in \mathcal{V}_{\mathrm{frontier}}\bigr)
\le
\Pr \bigl(\|X_\perp\|_2 \ge \rho_-\bigr)
\exp\!\left(-\frac{m-1}{2}\cos^2\phi\right).
\]

If \(\rho_-=\rho_\eta\), then by the above \(\chi^2\)-tail bound,
\[
\Pr \bigl(\|X_\perp\|_2 \ge \rho_\eta\bigr) \le \eta,
\]
hence
\[
\Pr \bigl(X \in \mathcal{V}_{\mathrm{frontier}}\bigr)
\le
\eta \exp\!\left(-\frac{m-1}{2}\cos^2\phi\right).
\]

Finally, for \(N\) i.i.d.\ passive samples \(X_1,\dots,X_N\), the union bound gives
\[
\Pr\!\left(\exists i \in [N] : X_i \in \mathcal{V}_{\mathrm{frontier}}\right)
\le
\sum_{i=1}^N \Pr(X_i \in \mathcal{V}_{\mathrm{frontier}})
=
N\,\Pr(X \in \mathcal{V}_{\mathrm{frontier}})
\le
N \eta \exp\!\left(-\frac{m-1}{2}\cos^2\phi\right).
\]
\end{proof}

\label{sec:gaussian_analysis_appendix}
\newpage

\section{Gradient Descent and Ascent on flow matching loss}
\subsection{Signed replay-based continued pre-training and connection to data unlearning}

This appendix formalizes the signed continued-pretraining update used in Sec. \ref{sec:algorithm} and clarifies its connection to data unlearning objectives for diffusion models~\cite{alberti2025data}. The key point is that, after each active-query round, verifier feedback naturally partitions the queried synthetic data into an \emph{accepted} set and a \emph{rejected} set. This induces a retain/forget decomposition analogous to data unlearning, except that here the partition is generated online by the verifier and used to improve out-of-distribution validity coverage.

\paragraph{Per-round replay buffers.}
At round $t$, let
\[
\mathcal D_t^+ \subseteq \{x_i : y_i = 1,\ i \le t\},
\qquad
\mathcal D_t^- \subseteq \{x_i : y_i = 0,\ i \le t\}
\]
denote the accepted and rejected replay buffers accumulated up to round $t$. Thus the queried synthetic data available for continued pre-training are partitioned into
\[
X_t := \mathcal D_t^+ \uplus \mathcal D_t^-,
\qquad
A_t := \mathcal D_t^-,
\]
where $\uplus$ denotes disjoint union of indexed samples. By construction,
\[
A_t \subseteq X_t,
\qquad
X_t \setminus A_t = \mathcal D_t^+.
\]
This is exactly the retain/forget decomposition considered in data unlearning: $X_t$ is the current empirical dataset, $A_t$ is the subset to be unlearned, and $X_t\setminus A_t$ is the retained subset.

\paragraph{Underlying flow-matching loss.}
Let $\ell(\theta;x)$ denote the standard flow-matching training loss on a sample $x \in \X$, namely the same loss used in the original pre-training objective. For any finite empirical collection $S$, define
\begin{equation}
L_S(\theta)
:=
\frac{1}{|S|}
\sum_{x \in S} \ell(\theta;x).
\label{eq:empirical_set_loss}
\end{equation}
In particular,
\[
L_{X_t\setminus A_t}(\theta)=L_{\mathcal D_t^+}(\theta),
\qquad
L_{A_t}(\theta)=L_{\mathcal D_t^-}(\theta).
\]

\paragraph{Reduction to the deletion objective.}
A first natural objective is to train the model as if the rejected samples had been removed from the replay set, namely
\begin{equation}
L^{\mathrm{del}}_t(\theta)
:=
L_{X_t \setminus A_t}(\theta)
=
L_{\mathcal D_t^+}(\theta).
\label{eq:deletion_obj_actflow}
\end{equation}
Writing
\[
m_t := |\mathcal D_t^+|,
\qquad
k_t := |\mathcal D_t^-|,
\qquad
n_t := m_t + k_t,
\]
the same algebra used in~\cite{alberti2025data} gives the exact decomposition
\begin{equation}
L^{\mathrm{del}}_t(\theta)
=
\frac{n_t}{m_t} L_{X_t}(\theta)
-
\frac{k_t}{m_t} L_{A_t}(\theta)
=
\frac{n_t}{m_t} L_{X_t}(\theta)
-
\frac{k_t}{m_t} L_{\mathcal D_t^-}(\theta).
\label{eq:deletion_decomp_actflow}
\end{equation}
Thus, even the valid-only objective can be rewritten using \emph{both} accepted and rejected samples. This is precisely the same deletion identity underlying SISS in~\cite{alberti2025data}, specialized here to the verifier-induced partition of queried synthetic data. In the terminology of~\cite{alberti2025data}, our practical implementation uses the corresponding \emph{non-importance-sampled} variant, i.e., the analogue of SISS (No IS), rather than the one-pass importance-sampled estimator.

\paragraph{Why we use a signed objective.}
For out-of-distribution generable-set expansion, Eq.~\eqref{eq:deletion_obj_actflow} is often too conservative: it treats rejected samples only as points to be discarded. In contrast, rejected verifier queries contain useful information about directions in design space toward which the current model should allocate \emph{less} probability mass. This motivates the signed replay objective used in the main text:
\begin{equation}
L^{\mathrm{signed}}_t(\theta)
:=
L_{\mathcal D_t^+}(\theta)
-
\alpha_t\, L_{\mathcal D_t^-}(\theta).
\label{eq:signed_replay_objective}
\end{equation}
Its gradient is exactly the signed update:
\begin{equation}
\nabla_\theta L^{\mathrm{signed}}_t(\theta_t)
=
\nabla_\theta L_{\mathcal D_t^+}(\theta_t)
-
\alpha_t \nabla_\theta L_{\mathcal D_t^-}(\theta_t).
\label{eq:signed_replay_gradient_full}
\end{equation}
At the minibatch level, our implementation follows the \emph{non-importance-sampled} variant: we draw separate minibatches
\[
U_t^+ \subseteq \mathcal D_t^+,
\qquad
U_t^- \subseteq \mathcal D_t^-,
\]
and form the empirical losses
\[
\widehat L_t^+(\theta)=\frac{1}{|U_t^+|}\sum_{x\in U_t^+}\ell(\theta;x),
\qquad
\widehat L_t^-(\theta)=\frac{1}{|U_t^-|}\sum_{x\in U_t^-}\ell(\theta;x).
\]
The practical signed update is then
\[
g_t
=
\nabla \widehat L_t^+(\theta_t)
-
\alpha_t \nabla \widehat L_t^-(\theta_t).
\]
Hence accepted samples provide an attractive update, while rejected samples provide a repulsive update of controlled magnitude. This corresponds to the analogue of SISS (No IS) in~\cite{alberti2025data}: two separate replay-buffer minibatches are used, rather than a single importance-sampled mixture minibatch.

\paragraph{Relation to weighted data-unlearning objectives.}
The weighted SISS objective of~\cite{alberti2025data} can be written as
\begin{equation}
L^{\mathrm{wSISS}}(\theta)
=
L_{X\setminus A}(\theta)
-
s \frac{|A|}{|X\setminus A|} L_A(\theta).
\label{eq:wsiss_generic}
\end{equation}
Applying this to the verifier-induced partition
\[
X = X_t = \mathcal D_t^+ \uplus \mathcal D_t^-,
\qquad
A = A_t = \mathcal D_t^-,
\]
yields
\begin{equation}
L^{\mathrm{wSISS}}_t(\theta)
=
L_{\mathcal D_t^+}(\theta)
-
s_t \frac{k_t}{m_t} L_{\mathcal D_t^-}(\theta).
\label{eq:wsiss_actflow}
\end{equation}
Therefore our signed replay objective in Eq.~\eqref{eq:signed_replay_objective} is exactly the same class of weighted retain-minus-forget objective, with the identification
\begin{equation}
\alpha_t
=
s_t \frac{k_t}{m_t}.
\label{eq:alpha_superfactor_map}
\end{equation}
In this sense, the update used by \AlgNameShort is a direct reduction of the weighted unlearning objective of~\cite{alberti2025data} to the online verifier-guided expansion setting considered here.

\paragraph{Possible importance sampling (IS) extension.}
As mentioned, our implementation is the direct analogue of SISS (No IS) in~\cite{alberti2025data}, rather than the one-pass importance-sampled estimator introduced there. An importance-sampled extension could also be used in our setting by sampling from a suitable mixture over accepted and rejected replay samples and reweighting by the corresponding importance ratios, but we avoid this additional estimator complexity here for simplicity.

\paragraph{Interpretation in our setting.}
The reduction above is exact at the level of the empirical replay buffers held fixed at round $t$. Conditionally on the current buffers $(\mathcal D_t^+,\mathcal D_t^-)$, the signed objective
\[
L^{\mathrm{signed}}_t(\theta)
=
L_{\mathcal D_t^+}(\theta)
-
\alpha_t L_{\mathcal D_t^-}(\theta)
\]
should be read as follows:
\begin{enumerate}
    \item the first term increases likelihood of synthetic samples that have been certified as valid by the verifier and therefore support expansion into new valid regions;
    \item the second term decreases likelihood of queried samples that were explicitly rejected by the verifier and therefore encode evidence against those directions of expansion.
\end{enumerate}
Thus, unlike standard valid-only pre-training, the signed replay update uses \emph{both} outcomes of the verifier query and converts the online generation-and-verification loop into a principled retain/forget training signal.

\paragraph{Practical choice of $\alpha_t$.}
In the practical implementation, the loss-level weight \(\alpha_t\) is calibrated online as a gradient-norm fraction. Concretely, for a target ratio \(\alpha \in [0,1]\), we set
\[
\alpha_t
\;:=\;
\alpha\,
\frac{\|\nabla \widehat L_t^+(\theta_t)\|_2}
{\|\nabla \widehat L_t^-(\theta_t)\|_2},
\]
This yields
\[
\bigl\|\alpha_t \nabla \widehat L_t^-(\theta_t)\bigr\|_2
\approx
\alpha \,\bigl\|\nabla \widehat L_t^+(\theta_t)\bigr\|_2,
\]
so the rejected-sample contribution is scaled to have a prescribed norm fraction relative to the accepted-sample contribution. 
\label{sec:loss_appendix}
\newpage

\section{Intermezzo: Theory of Active Safe Logistic Regression}
Suppose that given any query point $x$, the verifier provides binary feedback $y \in \curly{0,1}$ generated according to
\begin{align}
    \label{eq: verifier}
    \Pr[y=1 \mid x] = s(f(x)),
\end{align}
where $s: \R \to [0,1]$ is the sigmoid function, and $f: \X \to \R$ is a validity function. We assume that $f$ lies in the RKHS $\calH_k$ associated with a kernel $k:\X \times \X \to \R$, with bounded norm $\norm{f}_k \leq B$. We also assume that $f$ is $L_f$-smooth with respect to a metric $d$, namely
\[
|f(x)-f(x')| \le L_f d(x,x') \qquad \forall x,x' \in \X.
\]

We adopt a \emph{pre-query} indexing convention. At the beginning of round $t$, the learner has access to the history
\[
H_{t-1} \triangleq \curly{(x_\tau,y_\tau)}_{\tau=1}^{t-1},
\]
fits a probabilistic model using $H_{t-1}$, constructs a safe set $S_t$, and then selects the next query $x_t$.
Given $H_{t-1}$, the learner estimates $f$ by minimizing the regularized negative log-likelihood
\begin{equation}
    \label{eq: logistic regression}
\begin{aligned}
    \mu_t
    &\triangleq \argmin_{g \in \calH_k,\ \norm{g}_k \le B} \calL(g,H_{t-1}), \\
    \calL(g,H_{t-1})
    &\triangleq \sum_{\tau=1}^{t-1}
    \Bigl[-y_{\tau}\log\bigl(s(g(x_\tau))\bigr)
    -(1-y_{\tau})\log\bigl(1-s(g(x_\tau))\bigr)\Bigr]
    + \frac{\lambda}{2}\norm{g}_k^2,
\end{aligned}
\end{equation}
where $\lambda>0$ is a regularization coefficient. By the Representer Theorem \citep{scholkopf2001generalized}, the solution lies in the span of the kernel sections at the previously queried points.

Given $\mu_t$, the learner predicts the verifier output via $s(\mu_t(x))$. Previous work gives anytime-valid confidence sets of the form
\[
[s(\mu_t(x)) \pm \beta_t(\delta)\sigma_t(x)]
\]
for kernelized logistic regression~\citep{pasztor2024bandits}. Under our indexing convention, the uncertainty scores are
\[
\sigma_t^2(x)
=
k(x,x)-k_{t-1}^\top(x)\,(K_{t-1}+\lambda\kappa I_{t-1})^{-1}k_{t-1}(x),
\]
where
\[
\kappa \triangleq \sup_{a \le B} \frac{1}{\dot s(a)},
\]
$k_{t-1}(x)$ is the kernel vector with entries $[k_{t-1}(x)]_i = k(x_i,x)$, and $K_{t-1}$ is the kernel matrix with entries $[K_{t-1}]_{i,j}=k(x_i,x_j)$.

We will use the following pre-query reindexing of the confidence-sequence guarantee.

\begin{theorem}[Kernelized Logistic Confidence Sequences]
\label{thm: confidence set}
Assume $f \in \mathcal{H}_k$ and $\|f\|_k \le B$. Assume that the data $x_1, \ldots, x_t$ used to fit the model $\mu_t$ lie in a compact subset $A \subseteq \X$. Assume further that $k(x,x') \leq 1$ for $x, x' \in A$. Let $0 < \delta < 1$ and define
\begin{equation}
\label{eq: beta_t}
\beta_t(\delta)
\triangleq
4L_sB
+
2L_s\sqrt{\frac{2\kappa}{\lambda}\Bigl(\gamma_{t-1}^A + \log(1/\delta)\Bigr)},
\end{equation}
where
\[
\gamma_t^A
\triangleq
\max_{x_1,\ldots,x_t \in A}\;
\frac{1}{2}\log\det\!\Bigl(I_t + (\lambda\kappa)^{-1}K_t\Bigr),
\qquad
L_s \triangleq \sup_{a \le B} \dot s(a).
\]
Then
\[
\Pr\!\left(
\forall t \ge 1,\ \forall x \in A:\ 
\bigl|s(\mu_t(x)) - s(f(x))\bigr|
\le
\beta_t(\delta)\,\sigma_t(x)
\right)
\ge 1-\delta.
\]
\end{theorem}

Using these confidence sets, the learner may perform active safe logistic regression, querying only points that are certified to satisfy $s(f(x)) \ge h$ for some threshold $h \in [0,1]$. This parallels safe Bayesian optimization in the regression setting~\citep[\eg][]{sui2015safe,sui2018stagewise}.

In this setting, safe designs are those for which the verifier returns label $1$ with sufficiently high probability. Suppose the learner is given a safe seed set $S_0 \subseteq \X$. At round $t$, the learner updates its certified safe set as
\begin{align}
\label{eq: safe set update}
    S_t
    =
    S_{t-1}
    \cup
    \curly{
    x \in \X \ \middle|\ 
    \exists x' \in S_{t-1}:
    s(\mu_t(x')) - \beta_t(\delta)\sigma_t(x') - L_sL_f d(x,x')
    \ge h
    }.
\end{align}

The next lemma shows that, conditioned on the confidence event of \Cref{thm: confidence set}, every point ever added to $S_t$ is indeed safe.

\begin{lemma}[Safety]
\label{lem: safety}
Condition on the event that the confidence intervals of \Cref{thm: confidence set} are valid. Then for every $t \ge 0$, every $x \in S_t$ satisfies $s(f(x)) \ge h$.
\end{lemma}

\begin{proof}
We argue by induction on $t$. The claim holds at $t=0$ by assumption on the seed set $S_0$. Now fix $t \ge 1$, and let $x \in S_t$. Either $x \in S_{t-1}$, in which case the claim follows by induction, or
\[
x \in
\curly{
x \in \X \ \middle|\
\exists x' \in S_{t-1}:
s(\mu_t(x')) - \beta_t(\delta)\sigma_t(x') - L_sL_f d(x,x')
\ge h
}.
\]
In the latter case, there exists $x' \in S_{t-1}$ such that
\[
s(\mu_t(x')) - \beta_t(\delta)\sigma_t(x') - L_sL_f d(x,x') \ge h.
\]
By the confidence event,
\[
s(f(x')) \ge s(\mu_t(x')) - \beta_t(\delta)\sigma_t(x').
\]
Moreover, since $s$ is $L_s$-Lipschitz and $f$ is $L_f$-Lipschitz,
\[
s(f(x)) \ge s(f(x')) - L_sL_f d(x,x').
\]
Combining the two inequalities gives
\[
s(f(x))
\ge
s(\mu_t(x')) - \beta_t(\delta)\sigma_t(x') - L_sL_f d(x,x')
\ge h.
\]
Thus every $x \in S_t$ is safe.
\end{proof}

Notice that by construction, the sequence $\curly{S_t}_{t\ge 0}$ is monotone:
\[
S_{t-1} \subseteq S_t \qquad \forall t \ge 1.
\]
Using this set of safe decisions, the learner may follow an active learning procedure which samples
the next point at which to query the verifier as the one in the set of safe decisions that has the highest
uncertainty score, often referred to as \emph{safe uncertainty sampling}:
\begin{align}
    \label{eq:safe_uncertainty_sampling}
    x_t = \argmax_{x \in S_t} \sigma_t(x).
\end{align}
After querying the verifier at $x_t$ and observing $y_t$, the history updates as
\[
H_t = H_{t-1} \cup \curly{(x_t,y_t)}.
\]
We ultimately want to show some notion of suitable expansion of the safe set by following the above
sampling procedure. As our safety set at any time is defined as an expansion of the safe set at the
previous time, we cannot guarantee convergence to the entire subset of $\X$ which is safe. Instead, we
must define some notion of \emph{reachable safe set} of decisions. To this end, we consider the tightened one
step reachability operator defined as
\[
R_\varepsilon(S)
\triangleq
\curly{
x \in \X
\ \middle|\
\exists x' \in S:
s(f(x')) - L_sL_f d(x,x') - \varepsilon \ge h
}.
\]
The $H$-fold recursive application of this operator is denoted by $R_\varepsilon^H(S)$, and its closure as $H \to \infty$ is denoted $\bar R_\varepsilon(S)$. Our goal is to show that safe uncertainty sampling expands the certified safe set toward $\bar R_0(S_0)$. We first prove an uncertainty-reduction bound over the current safe set. Let
\[
\Omega^\star \triangleq \curly{x \in \X : s(f(x)) \ge h}
\]
denote the subset of $\X$ corresponding to the \emph{valid design space}, and assume that $\Omega^\star$ is compact. Without loss of generality, by normalization of the kernel, assume that $k(x,x') \leq 1$ for $x, x' \in \Omega^\star$.

\begin{lemma}[Uncertainty Reduction]
    \label{lem: variance reduction}
    Fix $t \ge 0$ and $T \ge 1$. Under safe uncertainty sampling \eqref{eq:safe_uncertainty_sampling}, the epistemic uncertainty at time
$t+T$ decays as
    \[
        \max_{x \in S_t} \sigma_{t+T}(x)
        \le
        \sqrt{\frac{2}{\log(1+(\lambda\kappa)^{-1})}}
        \sqrt{\frac{\gamma_T^{\Omega^\star}}{T}}.
    \]
\end{lemma}

\begin{proof}
    By monotonicity of the uncertainty score, it holds that
    \begin{align*}
        \max_{x\in S_t} \sigma_{t+T}(x)
        &\leq \frac{1}{T}\sum_{n=0}^{T-1} \max_{x \in S_t} \sigma_{t+n}(x) &&\mbox{(uncertainty is monotone)} \\
        &\leq \frac{1}{T} \sum_{n=0}^{T-1} \max_{x \in S_{t+n}} \sigma_{t+n}(x) &&\mbox{($S_t$ are monotone)}\\
        &= \frac{1}{T} \sum_{n=0}^{T-1} \sigma_{t+n}(x_{t+n}) &&\mbox{(safe uncertainty sampling \eqref{eq:safe_uncertainty_sampling})}\\
        &\leq \frac{1}{\sqrt{T}} \sqrt{\sum_{n=0}^{T-1}\sigma^2_{t+n}(x_{t+n})} &&\mbox{(Cauchy--Schwarz)}.
    \end{align*}
    It holds that $x \leq \frac{\bar x}{\log(1 + \bar x)}\log(1+x)$ for $0\leq x \leq \bar x$. Let $K_{t:t+T-1}$ be the Gram matrix of the queried points $x_t,\dots,x_{t+T-1}$, let $K_{1:t-1}$ be the Gram matrix of the previously queried points $x_1,\dots,x_{t-1}$, let $K_{t:t+T-1,\,1:t-1}$ be the corresponding cross-Gram matrix, and define
    \[
    \Sigma_{t:t+T-1 \mid 1:t-1}
    \triangleq
    K_{t:t+T-1}
    -
    K_{t:t+T-1,\,1:t-1}
    (K_{1:t-1} + \lambda \kappa I_{t-1})^{-1}
    K_{1:t-1,\,t:t+T-1}.
    \]
    Then the above inequality along with the bound $k(x, x') \leq 1$ for $x,x' \in \Omega^\star$ implies that
    \begin{align*}
         \max_{x\in S_t} \sigma_{t+T}(x)
         &\leq \frac{1}{ \sqrt{T}} \sqrt{\frac{1}{\log( 1+ (\lambda \kappa)^{-1})}}  \sqrt{\sum_{n=0}^{T-1}\log\paren{1+ \frac{\sigma^2_{t+n}(x_{t+n})}{\lambda \kappa}}} \\
         &= \frac{1}{ \sqrt{T}} \sqrt{\frac{1}{\log( 1+ (\lambda \kappa)^{-1} )}}  \sqrt{\log \det(I_T + (\lambda \kappa)^{-1} \Sigma_{t:t+T-1 \mid 1:t-1})} \\
         &\leq  \frac{1}{ \sqrt{T}} \sqrt{\frac{1}{\log( 1+ (\lambda \kappa)^{-1})}} \sqrt{ \log \det(I_T + (\lambda \kappa)^{-1} K_{t:t+T-1})}\\
         &\leq\frac{1}{ \sqrt{T}} \sqrt{\frac{1}{\log( 1+ (\lambda \kappa)^{-1} )}}  \sqrt{\max_{x_1, \dots, x_T \in \Omega^\star}  \log \det(I_T + (\lambda \kappa)^{-1} K_T)}  \\
         &=\sqrt{2} \frac{\sqrt{\gamma_T^{\Omega^\star}}}{ \sqrt{T}} \sqrt{\frac{1}{\log( 1+ (\lambda \kappa)^{-1} )}}.
    \end{align*}
\end{proof}

The next lemma shows that once the uncertainty over the current safe set is sufficiently small, additional safe uncertainty sampling expands the safe set according to the one-step reachability operator. 

\begin{lemma}[One-step Reachable Expansion]
\label{lem: one step expansion}
Consider $\varepsilon > 0$, $t \ge 0$, and $T \ge 1$. Suppose that
\[
T \geq \frac{8 \beta_{t+T}^2 \gamma_T^{\Omega^\star} \log(1+(\lambda \kappa)^{-1})}{\varepsilon^2}.
\]
Then, conditioned on the event that the intervals of \Cref{thm: confidence set} are valid, it holds that
\[
R_{\varepsilon}(S_t) \subseteq S_{t+T}.
\]
\end{lemma}

\begin{proof}
By \Cref{lem: variance reduction} and the condition on $T$,
\begin{align}
    \label{eq: variance bound in set}
    \max_{x\in S_t}\sigma_{t+T}(x)
    \le
    \frac{\varepsilon}{2\beta_{t+T}}.
\end{align}
Consider any point $x \in R_\varepsilon(S_t)$. By definition, there exists $x' \in S_t$ such that
\[
h \le s(f(x')) - \varepsilon - L_sL_f d(x,x').
\]
By the confidence event at round $t+T$,
\[
s(f(x')) \le s(\mu_{t+T}(x')) + \beta_{t+T}\sigma_{t+T}(x').
\]
Hence
\begin{align*}
    h
    &\le
    s(f(x')) - \varepsilon - L_sL_f d(x,x') \\
    &\le
    s(\mu_{t+T}(x')) + \beta_{t+T}\sigma_{t+T}(x') - \varepsilon - L_sL_f d(x,x') \\
    &\le
    s(\mu_{t+T}(x')) - \beta_{t+T}\sigma_{t+T}(x') - L_sL_f d(x,x'),
\end{align*}
where the last step uses \eqref{eq: variance bound in set}, since $x' \in S_t$ implies
\[
\beta_{t+T}\sigma_{t+T}(x') \le \frac{\varepsilon}{2}.
\]
By monotonicity of the safe sets, $x' \in S_t \subseteq S_{t+T-1}$. Therefore there exists $x' \in S_{t+T-1}$ such that
\[
s(\mu_{t+T}(x')) - \beta_{t+T}\sigma_{t+T}(x') - L_sL_f d(x,x') \ge h.
\]
By the definition of $S_{t+T}$ in \eqref{eq: safe set update}, this implies $x \in S_{t+T}$.
\end{proof}

Applying the above result recursively leads to $H$-step expansion of $S_0$ at time $T^\star$. In particular, suppose $T^\star = HT$ for some $T$ satisfying
\[
T \geq \frac{8 \beta_{t+T}^2 \gamma_T^{\Omega^\star} \log(1+(\lambda \kappa)^{-1})}{\varepsilon^2}.
\]
This implies that the condition of the above lemma holds for each interval of length $T$, since $\beta_t$ is monotonically increasing. Consequently, it holds that
\[
R_{\varepsilon}^H(S_0) \subseteq S_{T^\star}.
\]
We can find $T^\star$ large enough to satisfy this condition as long as the complexity term $\beta_t^2 \gamma_t^{\Omega^\star}$ grows sublinearly with $t$.
\label{sec:safe_logistic_regression}
\newpage

\section{Theoretical Analysis of Diffusion Models Active Diffusion Expansion}
\label{app:ade analysis}
In the following, we first ($i$) report the general analysis for safe logistic regression that we introduced and presented in Apx. \ref{sec:safe_logistic_regression}, where we now re-interpret safety as validity, then we ($ii$) introduce a probabilistic modeling framework of the generative density $p_t$ via energy-based models, and ($iii$) derive coverage guarantees with sample complexity for the proposed active expansion algorithm. 

\subsection{Probabilistic Modeling of Binary Verifier over Generative Model Learned Representation}
\label{subsec:verifier_probabilistic_model}

We denote by $\phi: \X \to \Z$ the representation map learned by the pre-trained generative model. We now describe the active learning process of the verifier over the learned representation space $\Z$. Suppose that given any query point $x$ with latent representation $z \coloneqq \phi(x)$, the verifier provides binary feedback $y \in \curly{0,1}$ generated according to
\begin{align}
    \label{eq:verifier_apxC}
    \Pr[y=1 \mid z] = s(g(z)),
\end{align}
where $s: \R \to [0,1]$ is the sigmoid function, and $g: \Z \to \R$ is an unknown validity function. We assume that $g$ lies in the RKHS $\calH_k$ associated with a kernel $k:\Z \times \Z \to \R$, with bounded norm $\norm{g}_k \leq B$ and that $g$ is $L_g$-smooth with respect to a metric $d$, namely
\[
|g(z)-g(z')| \le L_g d(z,z') \qquad \forall z,z' \in \Z,
\]
for $B$ and $L_g$ known.
We additionally assume there exists a known constant $
\bar Z$ such that $\int_{\calZ} \exp(g(z)) dz \leq \bar Z$.  

On unbounded domains with respect to Lebesgue measure, the energy condition is incompatible with globally bounded kernels such as the squared-exponential kernel. Indeed, if $k$ is bounded on the diagonal and $\|g\|_k \le B$, then
\begin{align*}
    |g(z)| \le \|g\|_k \sqrt{k(z,z)} \mbox{ for all } z \in \mathcal{Z}
\end{align*}
implies that $g$ is uniformly bounded. Hence $\exp(g)$ is bounded below by a positive constant, and therefore cannot be integrable over an infinite-measure domain. Thus, in this setting, the energy condition should be understood as requiring a kernel class capable of representing functions with sufficiently negative tails. For example, sums of bounded kernels with even-degree polynomial kernels yield RKHS that contain coercive negative polynomials, such as $g(z)=c-z^\top A z$ with $A\succ 0$, which satisfy the energy condition.

We adopt a \emph{pre-query} indexing convention. At the beginning of round $t$, the learner has access to the history
\[
H_{t-1} \triangleq \curly{(z_\tau,y_\tau)}_{\tau=1}^{t-1}.
\]
Given $H_{t-1}$, the learner estimates $g$ by minimizing the regularized negative log-likelihood
\begin{equation}
    \label{eq:logistic_regression_apxC}
\begin{aligned}
    \mu_t
    &\triangleq \argmin_{u \in \calH_k,\ \norm{u}_k \le B} \calL(u,H_{t-1}), \\
    \calL(u,H_{t-1})
    &\triangleq \sum_{\tau=1}^{t-1}
    \Bigl[-y_{\tau}\log\bigl(s(u(z_\tau))\bigr)
    -(1-y_{\tau})\log\bigl(1-s(u(z_\tau))\bigr)\Bigr]
    + \frac{\lambda}{2}\norm{u}_k^2,
\end{aligned}
\end{equation}
where $\lambda>0$ is a regularization coefficient. By the Representer Theorem \citep{scholkopf2001generalized}, the solution lies in the span of the kernel sections at the previously queried points.

Given $\mu_t$, the learner predicts the verifier output via $s(\mu_t(z))$. Previous work gives anytime-valid confidence sets of the form
\[
[s(\mu_t(z)) \pm \beta_t\sigma_t(z)]
\]
for kernelized logistic regression~\citep{pasztor2024bandits}. Under our indexing convention, the uncertainty scores are
\[
\sigma_t^2(z)
=
k(z,z)-k_{t-1}^\top(z)\,(K_{t-1}+\lambda\kappa I_{t-1})^{-1}k_{t-1}(z),
\]
where
\[
\kappa \triangleq \sup_{a \le B} \frac{1}{\dot s(a)},
\]
$k_{t-1}(z)$ is the kernel vector with entries $[k_{t-1}(z)]_i = k(z_i,z)$, and $K_{t-1}$ is the kernel matrix with entries $[K_{t-1}]_{i,j}=k(z_i,z_j)$.

We will use the following pre-query reindexing of the confidence-sequence guarantee.

\begin{theorem}[Kernelized Logistic Confidence Sequences \citep{pasztor2024bandits}]
\label{thm:confidence_set_restated}
Assume $g \in \calH_k$ and $\|g\|_k \le B$. Assume that the queried points lie in a compact subset $A \subseteq \Z$. Without loss of generalization, suppose that $k(z,z') \leq 1\, \forall z,z' \in A$. Let $0 < \delta < 1$ and define
\begin{equation}
\label{eq: calibration coefficient}
\beta_t^A(\delta)
\triangleq
4L_sB
+
2L_s\sqrt{\frac{2\kappa}{\lambda}\Bigl(\gamma_{t-1}^A + \log(1/\delta)\Bigr)},
% \tag{5}
\end{equation}
where
\[
\gamma_t^A
\triangleq
\max_{z_1,\ldots,z_t \in A}\;
\frac{1}{2}\log\det \Bigl(I_t + (\lambda\kappa)^{-1}K_t\Bigr),
\qquad
L_s \triangleq \sup_{a \le B} \dot s(a).
\]
Then
\[
\Pr \left(
\forall t \ge 1,\ \forall z \in A:\ 
\bigl|s(\mu_t(z)) - s(g(z))\bigr|
\le
\beta_t^A(\delta)\,\sigma_t(z)
\right)
\ge 1-\delta.
\]
\end{theorem}

Using these confidence sets, the learner may perform guarded logistic regression, querying only points that are certified to satisfy $s(g(z)) \ge h$ for some threshold $h \in [0,1]$. This parallels safe Bayesian optimization in the regression setting~\citep[\eg][]{sui2015safe,sui2018stagewise}. 

In this setting, safe designs are those for which the probabilistic verifier returns label $1$ with sufficiently high probability. Suppose the learner is given a safe seed set $S_0 \subseteq \Z$. Given a monotonically increasing sequence of calibration coefficients at level $\delta$, $\curly{\hat \beta_t(\delta)}_{t\geq 0}$, the set of points that the learner can verify as safe may be expanded at round $t$ as
\begin{align}
\label{eq:recursion_S_t_apxD}
    S_t
    =
    S_{t-1}
    \cup
    \curly{
    z \in \Z \ \middle|\ 
    \exists z' \in S_{t-1}:
    s(\mu_t(z')) -  \hat \beta_t(\delta)\sigma_t(z') - L_sL_g d(z,z')
    \ge h
    }.
\end{align}
We differentiate between the sequence $\curly{\hat \beta_t(\delta)}$ and $\curly{ \beta_t^A(\delta)}$ as computing the latter requires knowledge of the subset $A$, whereas our downstream analysis shows that our sampling remains restricted to an unknown set. Despite this, the learner may reasonably have an upper bound on the information capacity of the kernel restricted to the unknown subset, and thus may reasonably have access to such upper bounds. 

The next lemma shows that, conditioned on the confidence event of \Cref{thm:confidence_set_restated}, every point ever added to $S_t$ satisfies the validity condition. 

\begin{lemma}[Safety]
\label{lem:safety_restated}
Condition on the event that the confidence intervals of \Cref{thm:confidence_set_restated} are valid at level $\delta\in (0,1)$ under a compact set $A$ with $k(z,z') \leq 1 \,\forall z,z'\in A$ and suppose that the sequence $\curly{\hat \beta_t(\delta)}_{t\geq 0}$ satisfies $\hat \beta_t(\delta) \geq \beta_t^A(\delta)$. Then for every $t \ge 0$, every $z \in S_t$ satisfies $s(g(z)) \ge h$.
\end{lemma}

\begin{proof}
We argue by induction on $t$. The claim holds at $t=0$ by assumption on the seed set $S_0$.

Now fix $t \ge 1$, and let $z \in S_t$. Either $z \in S_{t-1}$, in which case the claim follows by induction, or
\[
z \in
\curly{
z \in \Z \ \middle|\
\exists z' \in S_{t-1}:
s(\mu_t(z')) - \hat \beta_t(\delta)\sigma_t(z') - L_sL_g d(z,z')
\ge h
}.
\]
In the latter case, there exists $z' \in S_{t-1}$ such that
\[
s(\mu_t(z')) - \hat \beta_t(\delta)\sigma_t(z') - L_sL_g d(z,z') \ge h.
\]
By the confidence event,
\[
s(g(z')) \ge s(\mu_t(z')) - \hat \beta_t(\delta)\sigma_t(z').
\]
Moreover, since $s$ is $L_s$-Lipschitz and $g$ is $L_g$-Lipschitz,
\[
s(g(z)) \ge s(g(z')) - L_sL_g d(z,z').
\]
Combining the two inequalities gives
\[
s(g(z))
\ge
s(\mu_t(z')) -  \hat \beta_t(\delta)\sigma_t(z') - L_sL_g d(z,z')
\ge h.
\]
Thus every $z \in S_t$ exceeds the validity threshold.
\end{proof}

Notice that by construction, the sequence $\curly{S_t}_{t\ge 0}$ is monotone:
\[
S_{t-1} \subseteq S_t \qquad \forall t \ge 1.
\]

\subsection{Abstraction of the Generative Model as an EBM}

Our primary abstraction of the generative model is that it approximately tracks the sequence of certified safe sets constructed with a sequence of calibration parameters $\curly{\hat \beta_t(\delta)}$. We achieve this abstraction using an energy based model. Specifically, for a fixed $\gamma > 0$ and $0 < \ell \leq h$, we consider an energy function which satisfies 
\begin{equation}
    \label{eq: energy function}
\begin{aligned}
    \mu_t' \in \Bigg\{f: &\int_{\Z} e^{f(z)} dz \leq \bar Z, \\
    &f(z) \geq \operatorname{logit}(h) \mathbf{1}(S_t) \, \forall z\in \calZ, \\
    &  f(z) \leq \operatorname{logit}(\ell) \mbox{ for all $z$ such that } d(z, S_t) \geq \gamma \Bigg\}. 
\end{aligned}
\end{equation}
The lower bound on $f$ captures the behavior where the updates to our generative model maintain high density on regions that our verifier certifies as valid according to \eqref{eq: safe set update}. The upper bound on $f$ instead captures the behavior where the model maintains low density on points far from this validated set. The explicit bound on the partition function $\bar Z$ ensures that the normalization constant does not cause the density placed on the valid region to vanish. These capture the desiderata that our generative model updates maintain large density in regions which have been verified as valid, and maintain small density in regions far away from this verified region. 
The generative model updates using both positive and negative samples, defined in Apx.~\ref{sec:loss_appendix}, are designed to satisfy these desiderata. For our theoretical analysis, we assume that this abstraction is valid. 
\begin{assumption}
    \label{asmp: EBM}
    Consider $\gamma >0$ and let $\ell$ be such that 
\begin{align}
    \label{eq: ell bound}
    \frac{\exp(\operatorname{logit}(\ell))}{ \underline Z } \leq \frac{\exp(\operatorname{logit}(h))}{\bar Z}.
\end{align}
where $\underline Z \coloneqq \mathsf{vol}(S_0)\exp(\operatorname{logit}(h))$ and $\operatorname{logit}(u) \coloneqq \log \left(\frac{u}{1-u}\right)$.  We  assume that the generative model can be abstracted as an EBM by considering the density
\begin{align}
    p_t(z)
    =
    \frac{\exp(\mu_t'(z))}{Z_t},
    \qquad
    Z_t \coloneqq \int_{z \in \mathcal{Z}} \exp(\mu_t'(z)) \, \mathrm{d}z.
    \label{eq:energy_based_model_1_apx_replacement}
\end{align}
for the energy function $\mu_t'$ defined in \eqref{eq: energy function}.
\end{assumption}
We define the corresponding \emph{generable set} at level $\tau$ as the high-probability set
\begin{align}
    \Omega_t^\tau
    \coloneqq
    \left\{ x \in \mathcal{X} : p_t(z) \geq \tau \right\}.
    \label{eq:generable_set_high_prob_replacement}
\end{align}

The EBM-abstraction of the generative model allows us to draw a relationship between the generable set of the EBM and the certified safe sets in which the generable set at some level $\tau$ and round $t$ contains the set $S_t$.

\begin{restatable}[Generable and Verifier Set Inequality for the Recursive Safe Set]{lemma}{setInequalityRecursive}
\label{lemma:set_inequality_recursive}
Suppose that Assumption~\ref{asmp: EBM} holds with some $\gamma > 0$. 

Then, for any
\begin{align}
    \tau
    \leq \frac{\exp(\operatorname{logit}(h))}{\bar Z},
    \label{eq:epsilon_condition_light_tails_recursive}
\end{align}
$S_t$ is a subset of
the high-probability set $\Omega_t^\tau$ for the EBM model defined by \Cref{asmp: EBM}, namely $S_t \subseteq \Omega_t^\tau.$
\end{restatable}

\begin{proof}
By definition of the high-probability set,
\begin{align}
    \Omega_t^\tau
    &=
    \left\{
        x \in \mathcal{X} : p_t^\pi(x) \geq \tau
    \right\} =
    \left\{
        x \in \mathcal{X} :
        \frac{\exp(\mu_t'(x))}{Z_t} \geq \tau
    \right\} =
    \left\{
        x \in \mathcal{X} :
        \mu_t'(x) \geq \log \tau + \log Z_t
    \right\}.
    \label{eq:high_prob_set_as_mu_threshold_recursive}
\end{align}
Now let $z \in S_t$. By definition of $\mu_t'$,
    $\mu_t'(z)) \ge \operatorname{logit}(h)$.  Therefore, a sufficient condition for $x \in \Omega_t^\tau$ is $\operatorname{logit}(h) \ge \log \tau + \log Z_t,$ or equivalently,
    $\tau \le \exp(\operatorname{logit}(h))/Z_t.$ 

It remains to upper bound $Z_t$. This holds immediately by definition of $\mu_t'$ such that it satisfies 
\begin{align*} 
    Z_t = \int_\calZ \exp(\mu_t'(z)) dz \leq \bar Z. 
\end{align*}
Plugging this bound in above, we find that a sufficient condition is,
\begin{align}
    \tau
    \leq
    \frac{\exp(\operatorname{logit}(h))}
    {
        \bar Z
    }.
\end{align}
Under this condition, every $z \in S_t$ satisfies $\mu_t(z) \geq \log \tau + \log Z_t$, hence by~\eqref{eq:high_prob_set_as_mu_threshold_recursive} we have $z \in \Omega_t^\tau$. Therefore,
    $S_t \subseteq \Omega_t^\tau,$ 
concluding the proof.
\end{proof}

 At the same time, we can show that the generable set lies within an extension of the valid design space. To this end, we define the inflation of the valid design space. 
\begin{definition}
    We define the inflation of the valid set by amount $\zeta \in \mathbb{R}$ in metric $d$ as 
    \begin{align*}
        \Gamma(\zeta) = \curly{z \in \Z \setminus \Omega_*: d(z, \Omega_*) \leq \zeta}. 
    \end{align*}
\end{definition}

\begin{assumption}
    \label{asmp: compact hallucinations}
    Let $\gamma > 0$ be the same as Assumption~\ref{asmp: EBM}. We suppose that the set $\Omega_* \cup \Gamma(\gamma)$ is compact. Further assume without loss of generality (by normalization of the kernel) that $k(z,z') \leq 1$ for all $z, z' \in \Omega_* \cup \Gamma(\gamma)$.
\end{assumption}

\begin{lemma}
    \label{lem: bounded hallucinations}
    Fix some $\gamma > 0$. Suppose that the sequence of calibration coefficients $\curly{\hat \beta_t(\delta)}$ used to construct the sets $S_t$ satisfies $\hat \beta_t(\delta) \geq \beta_t^{\Omega_* \cup \Gamma(\gamma)}(\delta)$. Let Assumptions~\ref{asmp: EBM}-\ref{asmp: compact hallucinations} hold at level $\gamma$ and condition on the calibration event of \Cref{thm:confidence_set_restated} with $A \gets \Omega_* \cup \Gamma(\gamma)$. 
    For $\tau \geq \frac{\exp(\operatorname{logit}(\ell))}{\underline Z}$ it holds that $\Omega_t^{\tau} \subseteq \Omega_* \cup \Gamma(\gamma)$.
\end{lemma}
\begin{proof}
    By \eqref{eq:high_prob_set_as_mu_threshold_recursive} and the fact that $Z_t \geq \underline Z$ it holds that 
    \begin{align*}
        \Omega_t^\tau &\subseteq \curly{z \in \Z: \mu_t'(z) \geq \log(\tau) + \log(\underline Z)} \\
        &\subseteq \curly{z \in \Z: \mu_t'(z) \geq \operatorname{logit}(\ell)} \\
        &\subseteq S_t \cup \curly{z\in \Z\setminus S_t: d(z, S_t) \leq \gamma} \\
        &\subseteq \Omega_* \cup \curly{z\in \calZ \setminus \Omega_*: d(z, \Omega_*) \leq \gamma}\\
        &= \Omega_*\cup\Gamma(\gamma),
    \end{align*}
    where the final set inequality follows from \Cref{lem:safety_restated}.
\end{proof}

\subsection{Active Diffusion Expansion Core Analysis}
By building on the two previous subsections, we can finally analyze the active expansion process of the set $\Omega_t^\tau$. Intuitively, we proceed in two steps: first, we analyze the expansion of the verifier valid set $S_t$ under samples obtained via generative sampling; second, we use \Cref{lemma:set_inequality_recursive} to transfer this expansion guarantee to the generable set $\Omega_t^\tau$, which is the object we ultimately wish to study.

We adopt the same pre-query indexing convention as in \Cref{subsec:verifier_probabilistic_model}. Accordingly, at round $t$ the learner first constructs $\Omega_t^\tau$ and the verifier model $(\mu_t,\sigma_t)$, and then selects the next query point. The sampling scheme used by \Cref{alg:main_algorithm} is given by as follows. 

\begin{assumption}
    \label{asmp: sampling oracle}
    For some fixed $\alpha \geq 1$, we suppose that the sample points queried by the learner satisfy
    \begin{align}
    \label{eq:local_generative_sampling}
    z_t \in \Omega_t^\tau
    \qquad \text{s.t.} \qquad
    \sigma_t(z_t) \geq \frac{1}{\alpha}\max_{z \in \Omega_t^\tau} \sigma_t(z).
\end{align}
\end{assumption}
Crucially, the sampling oracle in \eqref{eq:local_generative_sampling} requires the generative model (\eg via inference-time techniques) to sample approximate maximizers of $\sigma_t(\cdot)$ over the current generable set $\Omega_t^\tau$.

We ultimately want to show a suitable notion of expansion of the valid set by following the above procedure. As our valid set at any time is defined as an expansion of the valid set at the previous time, we cannot guarantee convergence to the entire subset of $\Z$ which is valid. Instead, we must define some notion of \emph{reachable valid set}. To this end, we consider the tightened one-step reachability operator
\[
R_{\varepsilon}(S)
\triangleq
\curly{
z \in \Z
\ \middle|\
\exists z' \in S:
s(g(z')) - L_sL_g d(z,z') - \varepsilon \geq h
}.
\]
The $H$-fold recursive application of this operator is denoted $R_{\varepsilon}^H(S)$, and its closure as $H \to \infty$ is denoted $\bar R_\varepsilon(S)$. Ultimately we show that the valid set discovered by the learner converges to $\bar R_0(S_0)$, leading to the following theorem.
\begin{theorem}
\label{thm:formal_main_theorem}
    Fix some $\gamma > 0, \delta > 0$, $\varepsilon >0 $. Let $H$ be a positive integer. Suppose that the sequence of calibration coefficients $\curly{\hat \beta_t(\delta)}$ used to construct the sets $S_t$ satisfies $\hat \beta_t(\delta) \geq \beta_t^{\Omega_* \cup \Gamma(\gamma)}(\delta)$. Let Assumptions~\ref{asmp: EBM}-\ref{asmp: compact hallucinations} hold at level $\gamma$ and condition on the calibration event of \Cref{thm:confidence_set_restated} with $A \gets \Omega_* \cup \Gamma(\gamma)$. Let $\frac{\exp(\operatorname{logit}(\ell))}{\underline Z} \leq \tau \leq \frac{\exp(\operatorname{logit}(h)}{\bar Z}$.  Consider sampling with the oracle defined by \Cref{asmp: sampling oracle} for $T^\star = TH$ steps, with $T$ satisfying
    \[
    T \geq
    \frac{
    8\alpha^2 \hat \beta_{HT}(\delta)^2
    \gamma_{T}^{\Omega_* \cup \Gamma(\gamma)}
    }{
    \varepsilon^2 \log(1+(\lambda\kappa)^{-1})
    }.
    \]
    Then it holds that $R_{\varepsilon}^H(S_0) \subseteq \Omega_{T^\star}^\tau$. 
\end{theorem}

To show the above result, we first prove a few basic inequalities. Our first inequality bounds the uncertainty over the set of valid decisions after $T$ additional algorithm steps. 

\begin{lemma}[Uncertainty Reduction via Local Generative Sampling]
    \label{lem:variance_reduction_local}
    Consider the setting of \Cref{lem: bounded hallucinations}. 
    Fix $t \geq 0$ and $T \geq 1$. Under local generative sampling \eqref{eq:local_generative_sampling}, it holds that the epistemic uncertainty at time $t+T$ satisfies
    \[
        \max_{z \in S_t} \sigma_{t+T}(z)
        \leq
        \alpha
        \sqrt{\frac{2}{\log(1+(\lambda \kappa)^{-1})}}
        \sqrt{\frac{\gamma_T^{\Omega_* \cup \Gamma(\gamma)}}{T}}.
    \]
\end{lemma}

\begin{proof}
    By monotonicity of the uncertainty score, it holds that
    \begin{align*}
        \max_{z\in S_t} \sigma_{t+T}(z)
        &\leq \frac{1}{T}\sum_{n=0}^{T-1} \max_{z \in S_t} \sigma_{t+n}(z) &&\mbox{(uncertainty is monotone)} \\
        &\leq \frac{1}{T} \sum_{n=0}^{T-1} \max_{z \in S_{t+n}} \sigma_{t+n}(z) &&\mbox{($S_t$ are monotone)}\\
        &\leq \frac{1}{T} \sum_{n=0}^{T-1} \max_{z \in \Omega^\tau_{t+n}} \sigma_{t+n}(z) &&\mbox{(\Cref{lemma:set_inequality_recursive})}\\
        &\leq \frac{\alpha}{T} \sum_{n=0}^{T-1} \sigma_{t+n}(z_{t+n}) &&\mbox{(local generative sampling \eqref{eq:local_generative_sampling})}\\
        &\leq \frac{\alpha}{\sqrt{T}} \sqrt{\sum_{n=0}^{T-1}\sigma^2_{t+n}(z_{t+n})} &&\mbox{(Cauchy--Schwarz)}.
    \end{align*}
    It holds that $x \leq \frac{\bar x}{\log(1 + \bar x)}\log(1+x)$ for $0\leq x \leq \bar x$. Let $K_{t:t+T-1}$ be the Gram matrix of the queried points $z_t,\dots,z_{t+T-1}$, let $K_{1:t-1}$ be the Gram matrix of the previously queried points $z_1,\dots,z_{t-1}$, let $K_{t:t+T-1,\,1:t-1}$ be the corresponding cross-Gram matrix, and define
    \[
    \Sigma_{t:t+T-1 \mid 1:t-1}
    \triangleq
    K_{t:t+T-1}
    -
    K_{t:t+T-1,\,1:t-1}
    (K_{1:t-1} + \lambda \kappa I_{t-1})^{-1}
    K_{1:t-1,\,t:t+T-1}.
    \]
    Then the above inequality along with the fact that $k(z,z') \leq 1$ for all $z,z'\in\Omega_* \cup \Gamma(\gamma)$ implies that
    \begin{align*}
         \max_{z\in S_t} \sigma_{t+T}(z)
         &\leq \frac{\alpha}{ \sqrt{T}} \sqrt{\frac{1}{\log( 1+ (\lambda \kappa)^{-1})}}  \sqrt{\sum_{n=0}^{T-1}\log\paren{1+ \frac{\sigma^2_{t+n}(z_{t+n})}{\lambda \kappa}}} \\
         &= \frac{\alpha}{ \sqrt{T}} \sqrt{\frac{1}{\log(1+(\lambda \kappa)^{-1})}}  \sqrt{\log \det(I_T + (\lambda \kappa)^{-1} \Sigma_{t:t+T-1 \mid 1:t-1})} \\
         &\leq  \frac{\alpha}{ \sqrt{T}} \sqrt{\frac{1}{\log( 1+ (\lambda \kappa)^{-1})}} \sqrt{ \log \det(I_T + (\lambda \kappa)^{-1} K_{t:t+T-1})}\\
         &\leq\frac{\alpha}{ \sqrt{T}} \sqrt{\frac{1}{\log( 1+ (\lambda \kappa)^{-1})}}  \sqrt{\max_{z_1, \dots, z_T \in \Omega_* \cup \Gamma(\gamma)}  \log \det(I_T + (\lambda \kappa)^{-1} K_T)}  \\
         &=\alpha\sqrt{2} \frac{\sqrt{\gamma_T^{\Omega_*  \cup \Gamma(\gamma)}}}{ \sqrt{T}} \sqrt{\frac{1}{\log( 1+ (\lambda \kappa)^{-1})}}.
    \end{align*}
    Here the last inequality uses \Cref{lem: bounded hallucinations}, which implies that all queried points $z_t,\dots,z_{t+T-1}$ lie in $\Omega_* \cup \Gamma(\gamma)$.
\end{proof}

The next lemma shows that once the uncertainty over the current valid set is sufficiently small, additional local generative sampling expands the valid set according to the one-step reachability operator.

\begin{lemma}[One-step Reachable Expansion via Local Generative Sampling]
\label{lem:one_step_expansion_local}
Consider the setting of \Cref{lem:variance_reduction_local}.
Let $\varepsilon > 0$, $\delta >0$, $t \geq 0$, and $T \geq 1$. Suppose
\[
    T \geq
    \frac{
    8\alpha^2 \hat \beta_{t+T}(\delta)^2
    \gamma_T^{\Omega_* \cup \Gamma(\gamma)}
    }{
    \varepsilon^2 \log(1+(\lambda\kappa)^{-1})
    }.
\]
Then, 
it holds that
\[
R_{\varepsilon}(S_t) \subseteq S_{t+T}.
\]
\end{lemma}

\begin{proof}
    By \Cref{lem:variance_reduction_local},
    the condition on $T$ implies that
    \begin{align}
        \label{eq:variance_bound_local}
        \max_{z\in S_t} \sigma_{t+T}(z) \leq \frac{\varepsilon}{2 \hat \beta_{t+T}},
    \end{align}
    where we have adopted the shorthand $\hat \beta_t$ for $\hat \beta_t(\delta)$.
    Consider any point $z \in R_{\varepsilon}(S_t)$. By definition, there exists $z' \in S_t$ such that
    \[
        h \leq s(g(z')) - \varepsilon - L_s L_g d(z, z').
    \]
    By the confidence event at round $t+T$,
    \[
    s(g(z')) \leq s(\mu_{t+T}(z')) + \hat \beta_{t+T} \sigma_{t+T}(z').
    \]
    Hence
    \begin{align*}
        h
        &\leq s(g(z')) - \varepsilon - L_s L_g d(z, z') \\
        &\leq s(\mu_{t+T}(z')) + \hat\beta_{t+T} \sigma_{t+T}(z') - \varepsilon - L_s L_g d(z, z') \\
        &\leq s(\mu_{t+T}(z')) - \hat \beta_{t+T} \sigma_{t+T}(z') - L_s L_g d(z, z')
        && \mbox{($z' \in S_t$, substitute \eqref{eq:variance_bound_local})}.
    \end{align*}
    Since the sets $S_t$ are monotone, $z' \in S_t$ implies $z' \in S_{t+T-1}$. Therefore there exists $z' \in S_{t+T-1}$ such that
    \[
        h \leq s(\mu_{t+T}(z')) - \hat \beta_{t+T} \sigma_{t+T}(z') - L_s L_g d(z, z').
    \]
    By the definition of $S_{t+T}$ in \eqref{eq:recursion_S_t_apxD}, this implies $z \in S_{t+T}$.
\end{proof}

Applying the above result recursively leads to $H$-step expansion of $S_0$ at some timestep $T^\star$. In particular, suppose $T^\star = HT$ for some $T$ satisfying
\[
T \geq
\frac{
8\alpha^2 \hat \beta_{HT}(\delta)^2
\gamma_{T}^{\Omega_* \cup \Gamma(\gamma)}
}{
\varepsilon^2 \log(1+(\lambda\kappa)^{-1})
}.
\]
This implies that the condition of the above lemma holds for each interval of length $T$, since $\gamma_t$ and $\hat \beta_t$ are monotonically increasing. Consequently, it holds that
\[
R_{\varepsilon}^H(S_0) \subseteq S_{T^\star}.
\]
We can find $T^\star$ large enough to satisfy this condition as long as the complexity term $\hat\beta_t^2 \gamma_t^{\Omega_* \cup \Gamma(\tau)}$ grows sublinearly. Ultimately, by \Cref{lemma:set_inequality_recursive}, we have
\[
S_{T^\star} \subseteq \Omega_{T^\star}^{\tau},
\]
which implies that the generable set $\Omega_{T^\star}^{\tau}$ obtained after running the algorithm is a superset of the valid reachable set, namely
\begin{equation}
    R_{\varepsilon}^H(S_0) \subseteq S_{T^\star} \subseteq \Omega_{T^\star}^{\tau}.
    \label{eq:inequality_generable_reachable}
\end{equation}

While \eqref{eq:inequality_generable_reachable} states that the generable set of the extended diffusion model is a superset of the reachable valid set, it does not clarify whether enough model density has been placed within the valid reachable set $R_{\varepsilon}^H(S_0) \subseteq S_{T^\star}$ rather than within $\Omega_{T^\star}^{\tau} \setminus R_{\varepsilon}^H(S_0)$, which might contain invalid points. To shed light on this question, we next derive a pointwise lower bound on the final density over the generable set.

\subsection{Deriving a Lower Bound on Validity}

We now show that samples from the generative model trained after $T^\star$  are valid with high probability.  
\begin{corollary}
\label{corollary:validity_lower_bound}
    Under the setting of \Cref{thm:formal_main_theorem} and by the construction of the EBM model, it holds that 
    \begin{align*}
        \mathbb{P}_{z \sim p_{T^\star}}[z \in \Omega_*] \geq \mathsf{vol}(R_\varepsilon^H(S_0)) \frac{\exp(\operatorname{logit}(h))}{\bar Z}.
    \end{align*}
\end{corollary}
\begin{proof}
    It holds by definition of $p_{T^\star}$ that 
    \begin{align*}
        \mathbb{P}_{z \sim p_{T^\star}}[z \in \Omega_*] &= \int_{z \in\calZ} \mathbf{1}(\Omega_*)(z) p_{T^\star}(z) dz \\
        &= \int_{z \in\calZ} \mathbf{1}(\Omega_*)(z) \frac{\exp(\mu_{T^\star}'(z))}{\int_{\calZ} \exp(\mu_{T^\star}'(z)) dz} dz \\
        &\geq \int_{z \in R_\varepsilon^H(S_0)} \mathbf{1}(\Omega_*)(z) \frac{\exp(\mu_{T^\star}'(z))}{\bar Z} dz \\
        &=\int_{z \in R_\varepsilon^H(S_0)}  \frac{\exp(\mu_{T^\star}'(z))}{\bar Z} dz \\
        & \geq \int_{z \in R_\varepsilon^H(S_0)}  \frac{\exp(\operatorname{logit}(h))}{\bar Z} dz,
    \end{align*}
    where the last equality follows from the fact that $R_\varepsilon^H(S_0) \subseteq \Omega_*$, and the last inequality follows from the fact that $R_\varepsilon^H(S_0) \subseteq S_{T^\star}$ and $\mu_{T^\star}'(z) \geq \operatorname{logit}(h)$ for $z \in S_{T^\star}$. 
\end{proof}

\subsection{Proof of design-space coverage corollary}

\begin{assumption}[Fixed representation with nondegenerate Jacobian]
\label{ass:fixed_phi_jacobian}
The representation map $\phi:\X\to\Z$ is fixed throughout ADE, is a $C^1$-diffeomorphism onto $\Z=\phi(\X)$,
and satisfies $\abs{\det J\phi(x)}\ge j_{\min}>0$ for all $x\in\X$.
\end{assumption}

\corDesignSpaceReachability*

\begin{proof}
Condition on the event of Theorem~\ref{thm:main_reachability}, which holds with probability at least
$1-\delta$. On this event,
\begin{equation}
R_\epsilon^H(S_0)\subseteq \Omega^\tau_{T^\star}
=
\{z\in \Z:\; p'_{T^\star}(z)\ge \tau\}.
\label{eq:proof_main_theorem_event}
\end{equation}

We first show that for every $A\subseteq \X$,
\begin{equation}
\phi\bigl(R^{X,\phi}_\epsilon(A)\bigr)=R_\epsilon\bigl(\phi(A)\bigr).
\label{eq:operator_intertwining}
\end{equation}
Indeed, if $x\in R^{X,\phi}_\epsilon(A)$, then for some $x'\in A$,
\[
s\bigl(g(\phi(x'))\bigr)-L_sL_g\,d\bigl(\phi(x),\phi(x')\bigr)-\epsilon \ge h.
\]
Writing $z=\phi(x)$ and $z'=\phi(x')$, we get $z'\in \phi(A)$ and
\[
s(g(z'))-L_sL_g\,d(z,z')-\epsilon \ge h,
\]
hence $z\in R_\epsilon(\phi(A))$. This proves
$\phi(R^{X,\phi}_\epsilon(A))\subseteq R_\epsilon(\phi(A))$.

Conversely, if $z\in R_\epsilon(\phi(A))$, then for some $z'\in \phi(A)$,
\[
s(g(z'))-L_sL_g\,d(z,z')-\epsilon \ge h.
\]
Since $\phi$ is bijective, there exist unique $x=\phi^{-1}(z)$ and $x'=\phi^{-1}(z')$ with $x'\in A$.
Substituting $z=\phi(x)$ and $z'=\phi(x')$ gives
\[
s\bigl(g(\phi(x'))\bigr)-L_sL_g\,d\bigl(\phi(x),\phi(x')\bigr)-\epsilon \ge h,
\]
so $x\in R^{X,\phi}_\epsilon(A)$ and thus $z=\phi(x)\in \phi(R^{X,\phi}_\epsilon(A))$.
Therefore \eqref{eq:operator_intertwining} holds.

Iterating \eqref{eq:operator_intertwining} yields
\begin{equation}
\phi\Bigl(\bigl(R^{X,\phi}_\epsilon\bigr)^H(S_0^X)\Bigr)
=
R_\epsilon^H\bigl(\phi(S_0^X)\bigr)
=
R_\epsilon^H(S_0),
\label{eq:iterated_operator_intertwining}
\end{equation}
where we used $S_0^X=\phi^{-1}(S_0)$.

Now let $x\in \bigl(R^{X,\phi}_\epsilon\bigr)^H(S_0^X)$. By \eqref{eq:iterated_operator_intertwining},
its image $z:=\phi(x)$ belongs to $R_\epsilon^H(S_0)$, hence by \eqref{eq:proof_main_theorem_event},
\begin{equation}
p'_{T^\star}(z)\ge \tau.
\label{eq:representation_density_lower_bound}
\end{equation}

Since $\phi$ is a $C^1$-diffeomorphism, the change-of-variables formula implies that the design-space
and representation-space densities satisfy
\[
p^{\pi_{T^\star}}_1(x)
=
p'_{T^\star}(\phi(x))\,\abs{\det J\phi(x)}.
\]
Combining this with \eqref{eq:representation_density_lower_bound} and the Jacobian lower bound from
Assumption~\ref{ass:fixed_phi_jacobian} gives
\[
p^{\pi_{T^\star}}_1(x)
\ge
\tau\,j_{\min}
=
\tau_X.
\]
Therefore $x\in \Omega^{X,\tau_X}_{T^\star}$ by \eqref{eq:design_space_generable_set_cor}. Since $x$
was arbitrary, \eqref{eq:design_space_reachability_inclusion} follows.
\end{proof}

\subsection{Full design-space coverage under global reachability}

\begin{restatable}[Full coverage under finite-chain global reachability]{corollary}{fullCoverageChain}
\label{corollary:full_coverage_chain}
Assume the conditions of Theorem~\ref{thm:formal_main_theorem} and
Assumption~\ref{ass:fixed_phi_jacobian}. Let
\[
S_0^X := \phi^{-1}(S_0),
\qquad
\tau_X := j_{\min}\tau,
\]
as in Corollary~\ref{cor:design_space_reachability}. Suppose that for every design
\(x\in \Omega_*\) there exists a finite chain
\[
x_0,\dots,x_m \in \Omega_*,
\qquad
x_0\in S_0^X,\quad x_m=x,\quad m\le H,
\]
such that for every \(i=1,\dots,m\),
\[
L_sL_g\, d\bigl(\phi(x_i),\phi(x_{i-1})\bigr)+\varepsilon
\;\le\;
s\bigl(g(\phi(x_{i-1}))\bigr)-h .
\]
Then
\[
\Omega_* \subseteq \Omega_{T^\star}^{X,\tau_X}
\]
with probability at least \(1-\delta\), after the same number \(T^\star=HT\) of verified samples
as in Theorem~\ref{thm:formal_main_theorem}.
\end{restatable}

\begin{proof}
Condition on the calibration event used in Theorem~\ref{thm:formal_main_theorem}, which holds
with probability at least \(1-\delta\) by Theorem~\ref{thm:confidence_set_restated} under the
conditions of Theorem~\ref{thm:formal_main_theorem}. Fix \(x\in\Omega_*\), and let \(x_0,\dots,x_m\) be the chain given by the assumption. We first
prove by induction that
\[
x_i\in \bigl(R_{\varepsilon}^{X,\phi}\bigr)^i(S_0^X)
\qquad\text{for all } i\le m,
\]
with the convention \(\bigl(R_{\varepsilon}^{X,\phi}\bigr)^0(S_0^X)=S_0^X\).
The base case \(i=0\) is immediate since \(x_0\in S_0^X\). Assume
\(x_{i-1}\in \bigl(R_{\varepsilon}^{X,\phi}\bigr)^{i-1}(S_0^X)\). By the chain condition,
\[
s \bigl(g(\phi(x_{i-1}))\bigr)
-
L_sL_g\,d \bigl(\phi(x_i),\phi(x_{i-1})\bigr)
-\varepsilon
\ge h .
\]
Therefore, by the definition of \(R_{\varepsilon}^{X,\phi}\) in
\eqref{eq:design_space_reachability_operator},
\[
x_i \in
R_{\varepsilon}^{X,\phi}
\left(\bigl(R_{\varepsilon}^{X,\phi}\bigr)^{i-1}(S_0^X)\right)
=
\bigl(R_{\varepsilon}^{X,\phi}\bigr)^i(S_0^X).
\]
Thus
\[
x=x_m\in \bigl(R_{\varepsilon}^{X,\phi}\bigr)^m(S_0^X).
\]

Let \(z=\phi(x)\). Iterating the intertwining identity
\eqref{eq:operator_intertwining} \(m\) times gives
\[
\phi\left(\bigl(R_{\varepsilon}^{X,\phi}\bigr)^m(S_0^X)\right)
=
R_\varepsilon^m\bigl(\phi(S_0^X)\bigr)
=
R_\varepsilon^m(S_0),
\]
where the last equality uses \(S_0^X=\phi^{-1}(S_0)\). Hence
\[
z\in R_\varepsilon^m(S_0).
\]

We now show that every intermediate exact iterate is covered by the certified set at the corresponding
block time. Specifically, for every \(r\le H\),
\[
R_\varepsilon^r(S_0)\subseteq S_{rT}.
\]
The claim holds for \(r=0\). If it holds for some \(r<H\), then by monotonicity of the operator
\(R_\varepsilon\) and by Lemma~\ref{lem:one_step_expansion_local}, applied on the block starting at
time \(rT\),
\[
R_\varepsilon^{r+1}(S_0)
=
R_\varepsilon\bigl(R_\varepsilon^r(S_0)\bigr)
\subseteq
R_\varepsilon(S_{rT})
\subseteq
S_{(r+1)T}.
\]
Here the sample-complexity condition in Theorem~\ref{thm:formal_main_theorem} ensures that
Lemma~\ref{lem:one_step_expansion_local} applies on every block of length \(T\). Since \(m\le H\), the preceding inclusion and the monotonicity of the certified sets \(S_t\), which
follows from their recursive definition in \eqref{eq:recursion_S_t_apxD}, imply
\[
z\in R_\varepsilon^m(S_0)
\subseteq S_{mT}
\subseteq S_{HT}
=
S_{T^\star}.
\]
By Lemma~\ref{lemma:set_inequality_recursive},
\[
S_{T^\star}\subseteq \Omega_{T^\star}^{\tau}.
\]
Therefore \(p'_{T^\star}(\phi(x))\ge \tau\). Finally, by the change-of-variables formula and Assumption~\ref{ass:fixed_phi_jacobian},
\[
p^{\pi_{T^\star}}_1(x)
=
p'_{T^\star}(\phi(x))\,|\det J\phi(x)|
\ge
\tau j_{\min}
=
\tau_X.
\]
Hence \(x\in \Omega_{T^\star}^{X,\tau_X}\). Since \(x\in\Omega_*\) was arbitrary,
\[
\Omega_* \subseteq \Omega_{T^\star}^{X,\tau_X}.
\]
\end{proof}
\label{sec:apx_learning_theory}
\newpage

\section{ActFlow for Discrete Diffusion}

To adapt \AlgNameShort for \textit{discrete diffusion models}, we leverage the algorithm introduced in prior work for adapting a pre-trained discrete diffusion model to an intractable reward-tilted distribution~\citep{tang2025peptune}. Given a reward function defined over clean sequences, this method provably tilts a pre-trained discrete diffusion model to the reward-tilted distribution using off-policy reinforcement learning. The full algorithm is given in Alg \ref{alg:discrete-diffusion-training}.

\subsection{Discrete Diffusion as Continuous-Time Markov Chains}
In contrast to the continuous state space, where generative models are defined by stochastic differential equations (SDEs) or ordinary differential equations (ODEs), generative models in the discrete state space $\mathcal{V}:= \{1, \dots V\}$ are defined by \textbf{continuous-time Markov chains} (CTMCs). A CTMC is a stochastic process $X_{0:1}:=(X_s)_{s\in [0,1]}$ whose probability law is defined by a \textit{generator} $(\boldsymbol{Q}_s)_{s\in [0,1]}\in \mathbb{R}^{\mathcal{V}\times\mathcal{V}}$ defined as:
\begin{align}
    \boldsymbol{Q}_s(x,y)=\lim_{\Delta s\to 0}\frac{1}{\Delta s}(\text{Pr}(X_{s+\Delta s}=y|X_s=x)-\boldsymbol{1}_{x=y})
\end{align}
which defines the probability of transitioning from state $x\in \mathcal{V}$ to state $y\in \mathcal{V}$ at time $s$. 

Masked discrete diffusion models (MDMs) \citep{shi2024simplified, sahoo2024simple, ou2024your, zheng2024masked} are a class of discrete diffusion models that aim to generate sequences $x_1\sim p_{\text{data}}\in \mathcal{V}^L$ of length $L$ from a sequence of absorbing \textit{mask tokens} $\boldsymbol{M}$. The generative process is defined as a CTMC that evolves from a prior distribution $p_0$ defined as the Dirac delta of fully masked sequences to the data distribution $p_1\equiv p_{\text{data}}$. Since each position along the sequence $\ell\in \{1, \dots , L\}$ transitions from a mask token $X^\ell=\boldsymbol{M}$ to a clean token $X^\ell=x^\ell$, the generator can be parameterized as:
\begin{align}
    \boldsymbol{Q}_s(x,x^{\ell \gets v})=\gamma(s)p^{\theta_t}(\cdot|x)_{\ell, v}
\end{align}
where $x^{\ell\gets v}$ denotes the sequence where the $\ell$th token is replaced with state $v\in \mathcal{V}$ and $\gamma(t)$ is the forward noising schedule. To train $p^{\theta_t}(\cdot|x)$ to reconstruct sequences from the data distribution, we can optimize a \textbf{denoising cross-entropy (DCE)} loss \citep{shi2024simplified, sahoo2024simple, ou2024your} defined as:
\begin{align}
    \mathcal{L}_{\text{DCE}}(\theta; x_1):=\mathbb{E}_{s\sim \mathcal{U}(0,1)}\left[\frac{1}{s}\mathbb{E}_{p_s(\tilde{x}_s|x_1)}\sum_{\ell: \tilde{x}_s^\ell=\boldsymbol{M}}-\log p^{\theta_t}(\cdot|\tilde{x}_s^\ell)_\ell\right], \quad x_1\sim p_{\text{data}}
\end{align}
where $p_s(\tilde{x}_s|x_1)$ is the distribution of partially masked sequences obtained by masking each position with probability $s$ given a clean sequence $x_1\sim p_{\text{data}}$.

\subsection{Entropy-Regularized Uncertainty Optimization for Discrete Diffusion}
Given a function $\sigma(\cdot):\mathcal{V}^L\to\mathbb{R}$ that returns the epistemic uncertainty of a sequence $x_1\in \mathcal{V}^L$ and a pretrained discrete diffusion model that generates the path measure $\mathbb{P}^{\theta_0}$, we define the \textbf{uncertainty-tilted path measure} as:
\begin{align}
    \mathbb{P}^{\sigma}(X_{0:1}):=\frac{1}{Z}\mathbb{P}^{\theta_0}(X_{0:1})\exp\left(\frac{\sigma(X_1)}{\beta}\right), \quad p_1^\sigma(X_1)\propto p_{\text{data}}(X_1)\exp\left(\frac{\sigma(X_1)}{\beta}\right)\label{eq:uncertainty-tilted}
\end{align}
This uncertainty-tilted path measure coincides with the solution to the \textbf{entropy-regularized reward optimization} problem with reward defined as the uncertainty $r(\cdot):=\sigma(\cdot)$ given by:
\begin{align}
    \underset{\theta_{t+1}}{\arg\max}\mathbb{E}_{X_{0:1}\sim \mathbb{P}^{\theta_t}}[\sigma(X_1)]-\beta D_{\text{KL}}(\mathbb{P}^{\theta_t}\|\mathbb{P}^{\theta_0})
\end{align}
where $\mathbb{P}^{\theta_t}$ is the CTMC path measure generated from the adapted model with parameters $\theta$ and $\mathbb{P}^{\theta_0}$ is the frozen pre-trained model. $\beta$ is the weight of the KL regularization with the pre-trained model.

\subsection{Uncertainty-Aware Fine-Tuning of Discrete Diffusion}
To adapt a pre-trained discrete diffusion model to align with the \textit{uncertainty-tilted distribution} defined in (\ref{eq:uncertainty-tilted}), we leverage the off-policy reinforcement learning algorithm introduced in prior work ~\citep{tang2025peptune}, which is minimized when $\mathbb{P}^{u_\theta}=\mathbb{P}^{\sigma}$. The training objective is defined as the \textbf{weighted denoising cross-entropy} loss $\mathcal{L}_{\text{WDCE}}$ given by:
\begin{align}
    \mathcal{L}_{\text{WDCE}}:=\mathbb{E}_{X_{0:1}\sim \mathbb{P}^\sigma}\left[\mathcal{L}_{\text{DCE}}(\theta;x_1)\right]=\mathbb{E}_{X_{0:1}\sim \mathbb{P}^{\bar{u}}}\bigg[\underbrace{\frac{\mathrm{d}\mathbb{P}^\sigma}{\mathrm{d}\mathbb{P}^{\bar{\theta}_t}}}_{\exp(w^\sigma(X_1))}\mathcal{L}_{\text{DCE}}(\theta;x_1)\bigg]\label{eq:wdce}
\end{align}
where $\mathbb{P}^{\bar{\theta}_t}=\texttt{stopgrad}(\mathbb{P}^{\theta_t})$ is the CTMC path measure generated from the non-gradient-tracking model. We define the importance weight $w^\sigma(X_1):=\log \frac{\mathrm{d}\mathbb{P}^\sigma}{\mathrm{d}\mathbb{P}^{\bar{\theta}_t}}$ as the Radon-Nikodym derivative between the uncertainty-tilted path measure $\mathbb{P}^\sigma$ and the non-gradient-tracking adapted model $\mathbb{P}^{\bar{\theta}_t}$:
\begin{align}
    \log\frac{\mathrm{d}\mathbb{P}^\sigma}{\mathrm{d}\mathbb{P}^{\bar{\theta}_t}}(X_{0:1})=\underbrace{\frac{\sigma(X_1)}{\beta}+\sum_{s :X_{s+\Delta s}\neq X_s}\sum_{\ell: X_{s+\Delta s}^\ell\neq X_s^\ell}\log \frac{p^{\theta_0}(X_{s+\Delta s}^\ell|X_s)}{p^{\bar{u}}(X_{s+\Delta s}^\ell|X_s)}}_{w^\sigma(X_{0:1})}-\log Z
\end{align}
which reweights trajectories $X_{0:1}$ from the current frozen model by their likelihood under the uncertainty-tilted path measure. To optimize this objective, we iterate through the following steps:
\begin{enumerate}
    \item [(i)] Sample $B$ trajectories $X^i_{0:1}$ from the adapted model $\mathbb{P}^{\bar{\theta}_t}$ without gradient tracking while tracking the log-likelihoods of each step $\log p^{\theta_0}(X_{s+\Delta s}|X_s)-\log p^{\bar{\theta}_t}(X_{s+\Delta s}|X_s)$. 
    \item [(ii)] Compute the importance weights $w^\sigma(X^i_{0:1})$ using the log-likelihoods and the uncertainties evaluated on the clean sequences $\sigma(X^i_1)$.
    \item [(iii)] Store the clean sequences and their weights in a replay buffer $\mathcal{B}\gets \{x^i_1, w^\sigma(x_1^i)\}_{i=1}^B$.
    \item [(iv)] Resample $R$ partially masked versions of each clean sequence in the buffer $x_1^i$ at different time steps $s\in [0,1]$ to obtain $\{\tilde{x}^i_s,w^\sigma(x_1^i)\}_{i=1}^{B\times R}$.
    \item [(v)] Compute $\mathcal{L}_{\text{WDCE}}(\theta)$ from Eq (\ref{eq:wdce}) using $\{\tilde{x}^i_s,w^\sigma(x_1^i)\}_{i=1}^{B\times R}$ and update $\theta$ with $\nabla_\theta\mathcal{L}_{\text{WDCE}}$ for $N_{\text{iter}}$ iterations.
\end{enumerate}
After $T$ iterations, the output is the adapted model with path measure $\mathbb{P}^{\theta_T}$ that generates the distribution $p^{\theta_T}_1$ approximating the uncertainty-tilted path measure $\mathbb{P}^\sigma$ and distribution $p^\sigma_1$.

\begin{algorithm}[t]
\caption{\AlgNameShort for Discrete Diffusion}\label{alg:discrete-diffusion-training}
    \begin{algorithmic}[1]
        \State \textbf{Input:} pre-trained model $p^{\theta_0}(\cdot|X_s)$, uncertainty function $\sigma:\mathcal{V}^L\to\mathbb{R}$, number of iterations $T$, number of WDCE repeats $R$
        \For{$t=0,1, \dots, T-1$}
        \State $\{x^i_1,w^\sigma\}_{i=1}^B\gets \texttt{Generate}(p^{\theta_0}, p^{\theta_t})$
        \State $\mathcal{B}\gets\{x^i_1,w^\sigma\}_{i=1}^B$\Comment{replay buffer}
        \For{step in $1, \dots, N_{\text{step}}$}
        \State $\{\tilde{x}^i_s,w^\sigma\}_{i=1}^{B\times R}\gets \texttt{ResampleWithMask}(\mathcal{B};R)$
        \State Compute $\mathcal{L}_{\text{WDCE}}$ from (\ref{eq:wdce}) with $\{\tilde{x}^i_s,w^\sigma\}_{i=1}^{B\times R}$
        \State Update $\theta_{t+1}$ with $\nabla_{\theta}\mathcal{L}_{\text{WDCE}}$
        \EndFor
        \EndFor
        \State \textbf{return} adapted model $p^{\theta_T}$
    \end{algorithmic}
\end{algorithm}

\label{sec:discrete_diffusion}
\newpage

\section{Experimental Details}
\subsection{Domain-agnostic evaluation metrics for real-world OOD generative modeling.}
\label{sec:evaluation_metrics_apx}
\looseness -1 Standard evaluation criteria for deep generative models, such as Frechet Inception Distance (FID)~\citep{heusel2017gans}, aim to assess whether the generative model well-approximates the data distribution $p_{\mathrm{data}}$. This is misaligned with OOD generative modeling: successful expansion should increase valid coverage beyond the initial generable region, and may therefore {\em increase} distributional distance from the pre-trained model. In low-dimensional illustrative experiments, valid coverage can be estimated directly by discretizing the design space, together with validity. In molecular and protein spaces, however, such direct coverage computation is infeasible. We therefore employ several domain-agnostic evaluation metrics for OOD generative modeling, namely: $(i)$ number of fixed-threshold clusters covered by valid samples, to measure model coverage; $(ii)$ Vendi diversity~\citep{friedman2022vendi}, to measure distributional diversity; $(iii)$ FID to quantify divergence from the pre-trained model distribution, which we wish to increase; and validity percentage, to assess whether model expansion is happening over valid regions. In particular, the number of clusters corresponds to the number of hyper-spheres with a fixed data-specific radius, forming a packing over a finite draw of samples of fixed size across methods. This allows to approximately describe the volume covered by the union of hyper-spheres centered at the generated data points, and is an efficient-to-compute proxy metric to assess coverage of the model generable set.

\subsection{Illustrative 2D Experiments}
\paragraph{Overview.}
We evaluate \AlgNameShort in a two-dimensional illustrative design space. The base model is a continuous flow model over $\mathbb{R}^2$. The initial model is deliberately misspecified by centering the pretraining data at $(-1.1,0)$, using $512$ Gaussian samples with standard deviation $0.1$. The valid region is a $3\times 3$ checkerboard over $[-3.5,3.5]^2$: a point is valid if it lies in-bounds and falls in a checkerboard cell with even parity. The base model is pretrained for $2500$ steps with Adam, learning rate $10^{-3}$, and batch size $256$. Results are acquired over $20$ seeds and $95\%$ CIs are shown.

\paragraph{Algorithm configuration.}
We run \AlgNameShort for $500$ iterations and $20$ random seeds. At each iteration, \AlgNameShort self-generates $64$ samples. Fine-tuning uses batch size $256$ and $250$ gradient steps per iteration. Evaluation is performed every $50$ iterations using $3000$ samples for evaluation curves. We use $\beta=1/13$ and GP with RBF kernel with lengthscale $0.08$, and negative-gradient scale $\alpha_t=0.005$.

\paragraph{Uncertainty estimation.}
We use a Gaussian process with RBF kernel lengthscale set to $0.08$ and employ flow representation timestep $s=0.9$

\paragraph{Validity estimation (\ie verifier).}
The verifier is the deterministic checkerboard validity function. The domain $[-3.5,3.5]^2$ is partitioned into $3\times 3$ equal cells. A sample is valid if it lies inside the domain and its cell index $(i,j)$ satisfies $(i+j)\bmod 2=0$.

\paragraph{Coverage metric.}
Coverage is measured as generable coverage over the valid region. We draw $3{,}000$ samples from the current model and estimate its generable set on a $100\times 100$ histogram over $[-3.5,3.5]^2$. A bin is considered generable if its estimated density is at least $\tau=0.01$.

\paragraph{Ablations: no negative gradient in flow matching loss.} We report in Fig. \ref{fig:experiments_fig_toy_apx} visual results for the illustrative experiments, run with same parameters as in \ref{fig:experiments_fig_2}, except for $\alpha_t$, which is now set to $0.0$. As one can notice from Fig. \ref{fig:experiments_fig_toy_apx}, the change in this parameters leads to seemingly minimal changes in the results according to the metrics assessed.

\begingroup
  \captionsetup[subfigure]{aboveskip=1.7pt, belowskip=0pt}
% \subsection{Experiments on illustrative settings}
\setlength{\imgw}{0.21\textwidth}
\begin{figure*}[t]
    \centering
    % row 1
    \begin{subfigure}{\imgw}
      \centering
      \includegraphics[width=\textwidth]{images/visual_experiments/toy_gen_set_pre.png}
      \caption{Pre-trained samples}
      \label{fig:toy_top_a_apx}
    \end{subfigure}\hfill
    \begin{subfigure}{\imgw}
      \centering
      \includegraphics[width=\textwidth]{images/visual_experiments/toy_gen_set_recursive.png}
      \caption{\AlgRecNF samples}
      \label{fig:toy_top_b_apx}
    \end{subfigure}\hfill
    \begin{subfigure}{\imgw}
      \centering
      \includegraphics[width=\textwidth]{images/visual_experiments/toy_gen_set_recursive_filter.png}
      \caption{\AlgRecF samples}
      \label{fig:toy_top_c_apx}
    \end{subfigure}\hfill
    \begin{subfigure}{\imgw}
      \centering
      \includegraphics[width=\textwidth]{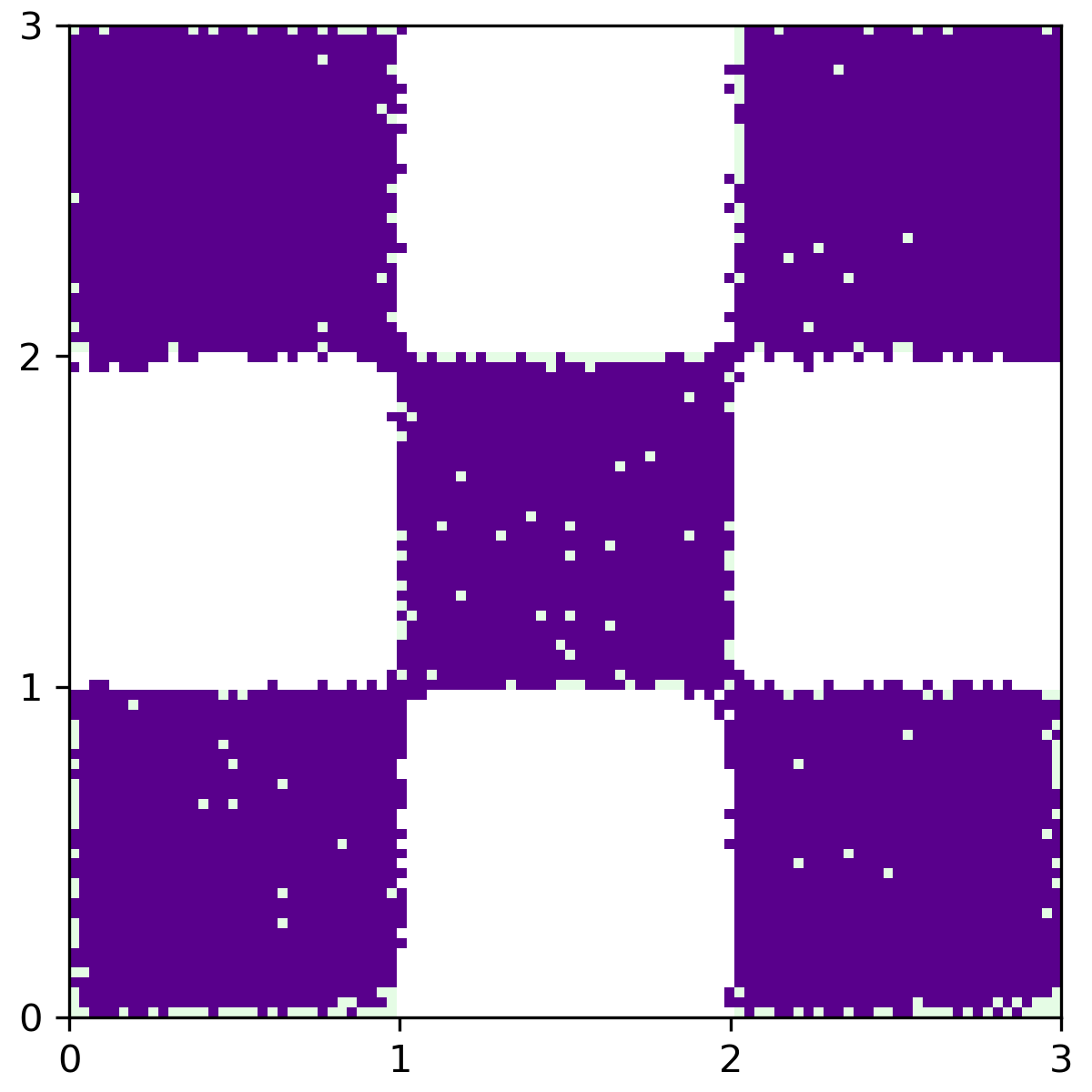}
      \caption{\AlgNameShort (ours)}
      \label{fig:toy_top_d_apx}
    \end{subfigure}\hfill
    \\[0.4em]\vspace{-1mm}
    % row 2 (repeat)
    \setlength{\imgw}{0.33\textwidth}
    \begin{subfigure}{\imgw}
      \centering
      \includegraphics[width=\textwidth]{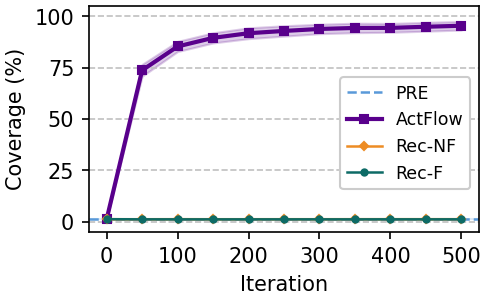}
      \caption{Coverage}
      \label{fig:toy_bottom_a_apx}
    \end{subfigure}%
    \begin{subfigure}{\imgw}
      \centering
      \includegraphics[width=\textwidth]{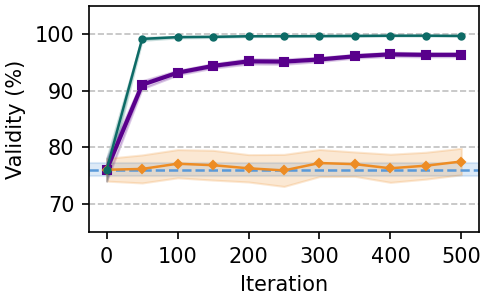}
      \caption{Validity}
      \label{fig:toy_bottom_b_apx}
    \end{subfigure}%
    \begin{subfigure}{\imgw}
      \centering
      \includegraphics[width=\textwidth]{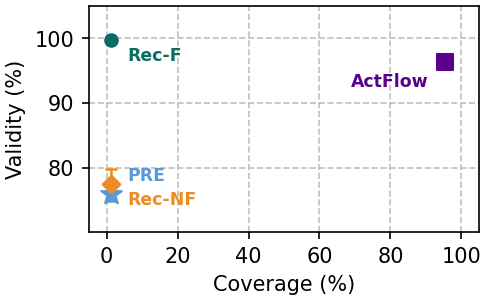}
      \caption{Coverage-Validity}
      \label{fig:toy_bottom_c_apx}
    \end{subfigure} \vspace{-1mm}
    \caption{\looseness-1 Results for illustrative experiments with same parameters as in \ref{fig:experiments_fig_2}, except for $\alpha_t=0$.} 
    \label{fig:experiments_fig_toy_apx} \vspace{-4mm}
\end{figure*}
\endgroup

\paragraph{Baselines.}
Both baselines use the same number of iterations, samples per iteration, fine-tuning steps, fine-tuning batch size, evaluation schedule, initial invalid model setting, and seeds as \AlgNameShort.

\paragraph{Hardware and Compute.}
Each 2D experiment job requests one RTX 2080 GPU, $16$GB memory per CPU, and a four-hour wall-clock limit.

\subsection{Molecular Design: QM9 Experiments}
\paragraph{Overview.} We apply our method to small-molecule generation using
FlowMol Gaussian~\citep{dunn2024mixed} pre-trained
on QM9~\citep{ramakrishnan2014quantum}. Results are acquired over $5$ seeds and $95\%$ CIs are shown.

\paragraph{Algorithm configuration.} We run \AlgNameShort for 
$1{,}066$ iterations, with $66$ initial warm-up iterations during which the model is not fine-tuned. Each iteration consists of $64$ samples, followed by $500$ fine-tuning gradient steps. Fine-tuning uses AdamW with learning rate $10^{-4}$, and batch
size $64$. Fine-tuning is deferred until $4096$ valid samples have
been collected, after which the accumulated
buffer is used jointly with new guided samples at each iteration. We employ $\beta = 1/10 = 0.1$, and $s = 0.9$.

\paragraph{Uncertainty estimation.} We use a deep bootstrapped ensemble of $5$ MLPs, each with two
hidden layers of $100$ units, ReLU activations, and $10\%$ dropout. Each ensemble member
is trained independently on a $90\%$ bootstrap subsample of the accumulated
feature label pairs, using Adam with learning rate $10^{-3}$, for up to $1000$ steps. The ensemble standard deviation across members is
used as the uncertainty signal.

\paragraph{Validity estimation (\ie verifier).} A generated molecule is deemed valid if
its RDKit-sanitised representation passes valence and bond-order checks  and consists of a single connected fragment. Sanitisation is preceded by an MMFF geometry relaxation~\citep{halgren1996merck}.

\paragraph{Coverage metric.} Coverage is measured as the number of distinct molecular
clusters obtained, computed via greedy sphere exclusion on Morgan fingerprints (radius $2$, $2048$ bits) using
Tanimoto similarity with threshold $\tau = 0.85$: a candidate is added as a new cluster
centre if its Tanimoto similarity to all previously selected centres is below $\tau$.
This is applied independently to the $500$ valid molecules in each evaluation batch.

\paragraph{Diversity metric.} Vendi score~\citep{friedman2022vendi} is computed on the
same $500$ valid molecules per evaluation. The kernel matrix $K$ is the pairwise Tanimoto
similarity over $2048$-bit Morgan fingerprints (radius $2$).

\paragraph{Baselines.} Both baselines use the same number of iterations, samples per iteration, fine-tuning steps, fine-tuning batch size, evaluation schedule, initial invalid model setting, and seeds as \AlgNameShort.

\paragraph{FID} We report in Fig. \ref{fig:experiments_FID_qm9geom_apx} FID results over iterates for QM9.

\paragraph{Hardware and Compute.} Each run used a single NVIDIA RTX 4090 GPU with
$4$ CPU cores and $32$\,GB of system memory, allocated for up to $120$\,h.

\subsection{Molecular Design: GEOM-Drugs Experiments}
\paragraph{Overview.} We apply our method to small-molecule generation using
FlowMol Gaussian~\citep{dunn2024mixed} pre-trained on
GEOM-Drugs~\citep{axelrod2022geom}, a dataset of ${\sim}300$\,k drug-like organic
molecules with energy-annotated conformers. Results are acquired over $3$ seeds
and $95\%$ CIs are shown.

\paragraph{Algorithm configuration.} We run \AlgNameShort for
$1{,}500$ iterations, with $190$ initial warm-up iterations during which the
model is not fine-tuned. Each iteration consists of $64$ samples, followed by
$2000$ fine-tuning gradient steps. Fine-tuning uses AdamW with learning rate
$10^{-4}$, and effective batch size $64$. Fine-tuning is deferred until $4096$ valid samples have been
collected. We employ $\beta = 1/7 \approx 0.14$, and $s = 0.8$.

\paragraph{Uncertainty estimation.} We use a deep bootstrapped ensemble of $5$
MLPs, each with two hidden layers of $100$ units, ReLU activations, and $10\%$
dropout. Each ensemble member is trained independently on a $90\%$ bootstrap
subsample of the accumulated feature--label pairs, using Adam with learning rate
$10^{-3}$, for up to $1000$ steps. The ensemble standard deviation across
members is used as the uncertainty signal.

\paragraph{Validity estimation (\ie verifier).} A generated molecule is deemed
valid if its RDKit-sanitised representation passes valence and bond-order checks
and consists of a single connected fragment. Sanitisation is preceded by an MMFF
geometry relaxation~\citep{halgren1996merck}.

\paragraph{Coverage metric.} Coverage is measured as the number of distinct
molecular clusters obtained, computed via greedy sphere exclusion on Morgan
fingerprints (radius $2$, $2048$ bits) using Tanimoto similarity with threshold
$\tau = 0.85$: a candidate is added as a new cluster centre if its Tanimoto
similarity to all previously selected centres is below $\tau$. This is applied
independently to the $500$ valid molecules in each evaluation batch.

\paragraph{Diversity metric.} Vendi score~\citep{friedman2022vendi} is computed
on the same $500$ valid molecules per evaluation. The kernel matrix $K$ is the
pairwise Tanimoto similarity over $2048$-bit Morgan fingerprints (radius $2$).

\paragraph{Baselines.} Both baselines use the same number of iterations, samples
per iteration, fine-tuning steps, fine-tuning batch size, evaluation schedule,
initial invalid model setting, and seeds as \AlgNameShort.

\begingroup
  \captionsetup[subfigure]{aboveskip=1.7pt, belowskip=0pt}
\setlength{\imgw}{0.25\textwidth}
\begin{figure*}[t]
    \centering
    \begin{subfigure}{\imgw}
      \centering
      \includegraphics[width=\textwidth]{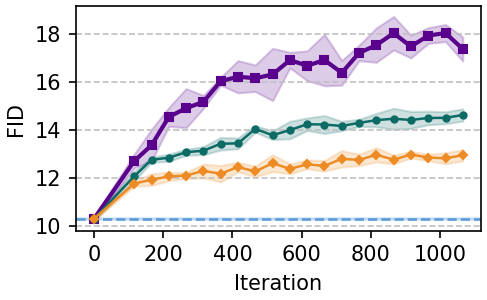}
      \caption{QM9 FID}
      \label{fig:qm9_apx_fid}
    \end{subfigure}%
    \begin{subfigure}{\imgw}
      \centering
      \includegraphics[width=\textwidth]{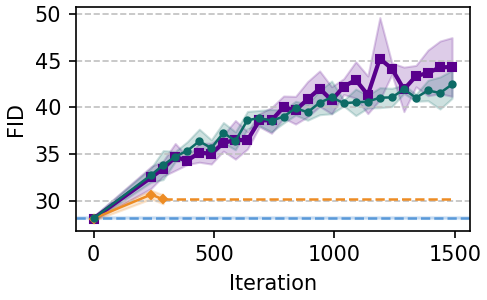}
      \caption{GEOM-Drugs FID}
      \label{fig:geom_apx_fid}
    \end{subfigure} 
    \caption{\looseness-1 FID results for QM9 and GEOM-Drugs}
    \label{fig:experiments_FID_qm9geom_apx} \vspace{-5mm}
\end{figure*}
\endgroup

\paragraph{FID} We report in Fig. \ref{fig:experiments_FID_qm9geom_apx} FID results over iterates for GEOM-Drugs.

\paragraph{Hardware and Compute.} Each run used a single NVIDIA RTX 4090 GPU
with $4$ CPU cores and $128$\,GB of system memory, allocated for up to
$120$\,h.

\subsection{Therapeutic Peptide Design Experiments}
\paragraph{Overview.}
We apply \AlgNameShort for therapeutic peptide design by adapting the pre-trained discrete diffusion peptide SMILES generator from PepTune \citep{tang2025peptune} trained on 11 million peptide SMILES. The Simplified Molecular-Input Line-Entry System (SMILES) \citep{Weininger1988} representation enables the generation of non-natural amino acids containing diverse chemical backbone and side-chain modifications, and cyclic modifications, significantly expanding the design space of therapeutic peptides over the standard 20 natural amino acids. The tokenization uses the SMILES Pair Encoding (SPE) tokenization scheme \citep{li2021smiles} from PeptideCLM \citep{Feller2024.08.09.607221} with vocabulary size $586$, including $5$ special tokens. Results are acquired over $5$ seeds and $95\%$ CIs are shown.

\paragraph{Algorithm configuration.}
We run \AlgNameShort for $5$ rounds following Alg \ref{alg:discrete-diffusion-training}, with an initial pool of $1000$ peptide sequences generated from the pretrained model and $100$ iterations per round. At each round, we draw $100$ sequences from the pool, score them with the verifier and Gaussian process uncertainty model, update the policy with the WDCE loss defined in Apx \ref{sec:discrete_diffusion} ($16$ replicates per sequence), and refresh the training pool by sampling from the updated policy. We train with Adam at a learning rate of $1 \!\times\!10^{-4}$ and batch size $100$. The reward is scaled by a parameter $\beta =0.005$.

\paragraph{Uncertainty estimation.}
Sequence-level uncertainty is obtained from a Gaussian process (GP) fit on top of the diffusion model's RoFormer encoder. For each peptide, we extract attention-pooled, $L_2$-normalised hidden states from the final transformer layer (768-dimensions) and treat the GP conditional posterior variance given the rest of the batch and the GP training set as the uncertainty signal. The GP uses an RBF kernel whose length scale is initialised from the mean pairwise embedding distance of $50$ samples from the pretrained model. The GP's posterior is refit at each round (100 iterations) on all sampled sequences. Only sequences that pass the verifier enter the GP buffer.

\paragraph{Validity estimation (\ie verifier).}
Generated SMILES are first parsed with RDKit. Those that yield a valid Mol object are then decoded with the \texttt{SMILES2PEPTIDE} verifier \citep{tang2025peptune}, which converts the SMILES into a sequence of natural and non-natural amino acids split on their peptide bonds.

\paragraph{Coverage metric.}
Coverage is measured as the number of distinct neighborhoods occupied by the $N{=}1000$ valid generated peptides in PeptideCLM \citep{Feller2024.08.09.607221} embedding space. Each peptide is embedded by mean-pooling the final-layer hidden states of the pretrained PeptideCLM RoFormer over non-padding tokens; the resulting vectors are $L_2$-normalised so that pairwise dissimilarity reduces to the cosine distance $d(x, x') = 1 - \langle x, x' \rangle$. We then apply sphere-exclusion clustering. Iterating through the sequences in arbitrary order, each peptide becomes either a new cluster center if its cosine distance to every existing center is at least $\tau$, or is otherwise assigned to the nearest center. We report the number of clusters at $\tau = 0.10$. Reported metrics are computed after 1200 iterations (12 rounds).

\paragraph{Diversity metric.}
We report the Vendi score \citep{friedman2022vendi} on the $L_2$-normalised PeptideCLM embeddings under an RBF kernel $K(x, x') = \exp\left(-\| x - x' \|^2 / (2 \sigma^2)\right)$ with length-scale $\sigma = 0.5$. Reported metrics are computed after 1200 iterations (12 rounds).

\paragraph{Fréchet Inception Distance.}
To quantify how far each fine-tuned generator drifts from the pretrained distribution of peptides, we report the Fréchet Inception Distance (FID) \citep{heusel2017gans} between two distributions: the $1000$ valid peptides generated by the model under evaluation and a matched reference set of $1000$ valid peptides sampled from the frozen pretrained checkpoint. The FID score is lower when closer to the pretrained distribution. We compute the score using $2048$-bit Morgan fingerprints (radius $2$) to evaluate drift in the chemical fingerprint space. Reported metrics are computed after 1200 iterations (12 rounds).

\paragraph{Baselines.}
We compare against three baselines: (1) the pretrained model \citep{tang2025peptune}, (2) \AlgRecNF where the policy is updated on its own samples with no uncertainty tilting (uniform weights on WDCE loss), (3) \AlgRecF where the policy is updated on its own samples filtered to retain only valid peptides classified by the verifier. We hold hyperparameters fixed across all baselines. 

\paragraph{Hardware and Compute.}
Each mode and seed run is trained on a single NVIDIA B200 GPU with $8$ CPU cores and $80$ GB of system RAM. A full $100$-iteration run fits within a $48$-hour wallclock budget per GPU.

\subsection{Protein Sequence Design Experiments}

\paragraph{Overview.} We apply our method to protein sequence design using a continuous ESM diffusion model from SGPO~\citep{yang2025steeringgenerativemodelsexperimental}. The base model is a continuous-space denoising network operating over ESM token-probability vectors of dimension equal to the vocabulary size ($31$ tokens, including the $20$ standard amino acids and special tokens), pre-trained on the CreiLOV fluorescence dataset~\citep{chen2023deep}. CreiLOV is a $119$-residue fluorescent protein; the dataset contains experimentally measured fluorescence fitness values for sequence variants. The diffusion process uses a cosine noise schedule, where $\alpha_t = \cos\!\left(\frac{(1-t)\pi}{2}\right)$ and $\beta_t = \sqrt{1 - \alpha_t^2}$, so that $t=1$ corresponds to data and $t=0$ to pure noise.

\paragraph{Algorithm configuration.} We run \AlgNameShort for $512$ iterations, each consisting of $64$ samples, followed by $1000$ fine-tuning gradient steps on the accumulated valid samples. Fine-tuning uses AdamW with learning rate $10^{-4}$, batch size $64$, and no weight decay. Fine-tuning is deferred until $4096$ valid samples have been collected (warm-up period), after which the accumulated buffer is used jointly with new guided samples at each iteration. We use $\beta = 1/50$. Features for the uncertainty estimator are extracted from the encoder of the ESM network at flow representation timestep $s = 0.8$, mean-pooled over the sequence dimension.

\paragraph{Uncertainty estimation.} We use a deep ensemble of $5$ MLPs, each with two hidden layers of $100$ units, ReLU activations, and $10\%$ dropout. Each ensemble member is trained independently on a $90\%$ bootstrap subsample of the accumulated feature-label pairs, using Adam with learning rate $10^{-3}$, for up to $1000$ steps. The ensemble standard deviation across members is used as the uncertainty signal.

\paragraph{Validity estimation (\ie verifier).} A generated sequence is deemed valid if its mean predicted local distance difference test (pLDDT), computed by ESMFold~\citep{lin2022language}, exceeds a threshold of $65$. ESMFold is run in batches of $32$ sequences.

\paragraph{Coverage metric.} Coverage is measured as the number of distinct sequence clusters accumulated across all $512$ iterations, computed via greedy sphere exclusion on sequence identity with threshold $\tau = 0.35$: a candidate is added as a new cluster center if its sequence identity to all previously selected center is below $\tau$.

\paragraph{Diversity metric.} Vendi score~\citep{friedman2022vendi} is computed on token-level ESM embeddings. For each generated sequence, we compute its embedding as the probability vector over the vocabulary projected through the ESM token embedding table ($L_2$-normalised and scaled by $\sqrt{d_\text{model}}$), then mean-pooled over the sequence length. Vendi score is computed using an RBF kernel with lengthscale $\ell = 2.0$ applied to these mean-pooled embeddings.

\paragraph{Ablation: feature timestep.} We ablated the fine-tuning flow representation time-steps $t \in \{0.5, 0.8, 0.9, 0.95\}$ in an alternative configuration with $1024$ valid samples are initially queried. Among the feature timestep variants, $t = 0.95$ performed worst in terms of coverage ($12$ clusters at $\lambda = 50$, Vendi $= 18.6$), showing that extracting features close to the data level might be significantly sub-optimal.

\paragraph{Baselines.} \AlgRecF runs the same fine-tuning loop without uncertainty-guided sampling, training only on valid samples (pLDDT $> 65$). \AlgRecNF additionally disables the validity filter during fine-tuning, training on all generated samples regardless of pLDDT. Both baselines use identical hyperparameters (fine-tuning steps, learning rate, batch size, warm-up threshold) to \AlgNameShort, differing only in whether uncertainty guidance and validity filtering are applied.

\paragraph{Hardware and Compute.} Each run used a single RTX 4090 GPU for 24 hours and 256 GB of memory.

\begingroup
  \captionsetup[subfigure]{aboveskip=1.7pt, belowskip=0pt}
\setlength{\imgw}{0.25\textwidth}
\begin{figure*}[ht]
    \centering
    \begin{subfigure}{\imgw}
      \centering
      \includegraphics[width=\textwidth]{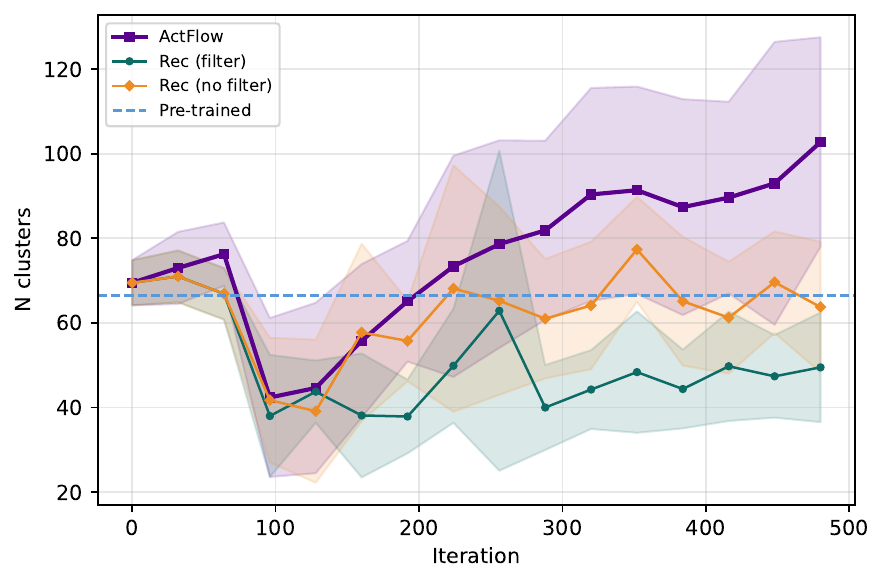}
      \caption{CreiLOV Coverage}
      \label{fig:proteins_a}
    \end{subfigure}%
    \begin{subfigure}{\imgw}
      \centering
      \includegraphics[width=\textwidth]{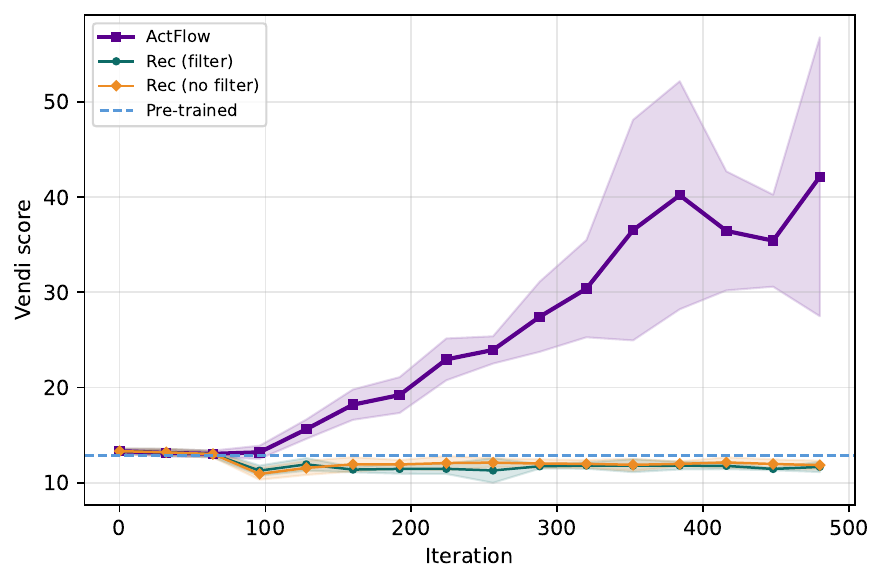}
      \caption{CreiLOV Diversity}
      \label{fig:proteins_b}
    \end{subfigure}%
    \begin{subfigure}{\imgw}
      \centering
      \includegraphics[width=\textwidth]{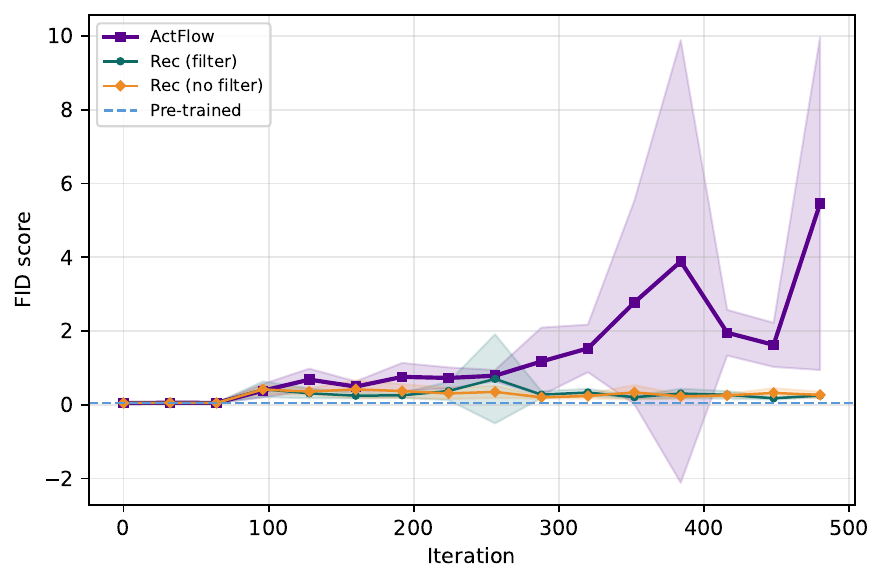}
      \caption{CreiLOV FID}
      \label{fig:proteins_c}
    \end{subfigure}%
    \begin{subfigure}{\imgw}
      \centering
      \includegraphics[width=\textwidth]{images/proteins/proteins_points.pdf}
      \caption{Coverage-Validity}
      \label{fig:proteins_d}
    \end{subfigure} \vspace{-2.5mm}
    \caption{\looseness-1
    Results of protein sequence design experiments over iterations (Figs \ref{fig:proteins_a} - \ref{fig:proteins_c}), and diversity-validity tradeoff at final iteration (Fig \ref{fig:proteins_d}). \AlgNameShort significantly outperforms \AlgRecNF and \AlgRecF in all diversity metrics (FID, Vendi, number of clusters), while maintaining a competitive validity.
    }
    \label{fig:experiments_block_proteins_apx} \vspace{-3mm}
\end{figure*}
\endgroup
\label{sec:experimental_details_appendix}
\newpage

\end{document}